\documentclass[11pt]{article}
\usepackage[letterpaper, left=1in, right=1in, top=1in,bottom=1in]{geometry}
\usepackage{amsfonts}       
\usepackage{nicefrac}       
\usepackage{bbm}
\usepackage[ruled,noend]{algorithm2e}
\usepackage{bm}
\usepackage{tikz}
\usetikzlibrary{positioning,chains,fit,shapes,calc,arrows}
\usetikzlibrary{arrows.meta}
\usepackage{graphicx}
\usepackage{pgf}
\usepackage{mathtools}
\usepackage{mathrsfs}
\usepackage{amsmath,amsthm,amssymb}
\usepackage{comment}
\usepackage{xfrac}
\usepackage{accents}

\usepackage{auxiliary}

\usepackage{float}

\usepackage{url}      

\usepackage[pagebackref=true]{hyperref} \PassOptionsToPackage{pagebackref=true}{hyperref}   

\usepackage{cleveref}

\title{The Transient Cost of Learning in Queueing Systems}

\author{
Daniel Freund\thanks{Massachusetts Institute of Technology, \texttt{dfreund@mit.edu}}
\and Thodoris Lykouris\thanks{Massachusetts Institute of Technology, \texttt{lykouris@mit.edu}} 
\and Wentao Weng\thanks{Massachusetts Institute of Technology, \texttt{wweng@mit.edu}}}

\date{First version: August 2023\\
Current version: April 2025\footnote{A condensed preliminary version of this work, titled \emph{Quantifying the Cost of Learning in Queueing Systems}, was accepted for presentation at the Conference on Neural Information Processing Systems (NeurIPS 2023).}}

\begin{document}

\maketitle

\begin{abstract}
Queueing systems are widely applicable stochastic models with use cases in communication networks, healthcare, service systems, etc. Although their optimal control has been extensively studied, most existing approaches assume perfect knowledge of the system parameters. This assumption rarely holds in practice where there is parameter uncertainty, thus motivating a recent line of work on bandit learning for queueing systems. This nascent stream of research focuses on the asymptotic performance of the proposed algorithms but does not provide insight on the transient performance in the early stages of the learning process.

In this paper, we propose the \emph{Transient Cost of Learning in Queueing ({$\colq$})}, a new metric that quantifies the maximum increase in time-averaged queue length caused by parameter uncertainty. We characterize the $\colq$ of a single-queue multi-server system, and then extend these results to multi-queue multi-server systems and networks of queues. In establishing our results, we propose a unified analysis framework for $\colq$ that bridges Lyapunov and bandit analysis, provides guarantees for a wide range of algorithms, and could be of independent interest.
\end{abstract}

\section{Introduction}\label{sec:intro}
Queueing systems are widely used stochastic models that capture congestion when services are limited. These models have two main components: jobs and servers. Jobs wait in queues and have different types. Servers differ in capabilities and speed. For example, in content moderation of online platforms~\cite{Makhijani21}, jobs are user posts with types defined by contents, languages and suspected violation types; servers are human reviewers who decide whether a post is benign or harmful. Moreover, job 
types can change over time upon receiving service. For instance, in a hospital, patients and doctors can be modeled as jobs and servers. A patient in the queue for emergent care can become a patient in the queue for surgery after seeing a doctor at the emergency department~\cite{armony2015patient}. That is, queues can form a network due to jobs transitioning in types. Queueing systems also find applications in other domains: call centers~\cite{GansKM03}, communication networks~\cite{Srikant_Ying_2014} and computer systems~\cite{harchol2013performance}.

The single-queue multi-server model is a simple example to illustrate the dynamics and decisions in queueing systems. In this model, there is one queue and $K$ servers operating in discrete periods. In each period, a new job arrives with probability $\lambda$. Servers have different service rates $\mu_1,\ldots,\mu_K$. The decision maker (DM) selects a server to serve the first job in the queue if there is any. If server~$j$ is selected, the first job in the queue leaves with probability $\mu_j$. The DM's goal is to maintain a low wait for the jobs, which is equivalent to keeping the queue length short. The optimal policy thus selects the server with the highest service rate $\mu^\star=\max_j \mu_j$; the usual regime of interest is one where the system is \emph{stabilizable}, i.e., $\mu^\star>\lambda$, which ensures that the queue length does not grow to infinity under an optimal policy. Of course, this policy requires perfect knowledge of the service rates. Under parameter uncertainty, the DM must balance the trade-off between exploring a server with an uncertain rate or exploiting a server with the highest observed service rate. 

A recent stream of work studies efficient learning algorithms for queueing systems. First proposed by~\cite{Walton14}, and later used by~\cite{KrishnasamySJS21, StahlbuhkSM21}, \emph{queueing regret} is a common metric to evaluate learning efficiency in queueing systems. In the single-queue multi-server model, let $Q(T,\pi)$ and $Q^{\star}(T)$ be the number of jobs in period $T$ under a policy $\pi$ and under the optimal policy respectively. Queueing regret is defined as {either the last-iterate difference in expected queue length, i.e., }$\expect{Q(T,\pi)-Q^\star(T)}$\cite{Walton14,KrishnasamySJS21}, {or the time-average version $\frac{1}{T}\expect{\sum_{t=1}^T Q(t,\pi)-Q^\star(t)}$}\cite{StahlbuhkSM21}. In the stabilizable case ($\lambda<\mu^\star$), the goal is usually to bound its scaling relative to $T$; scaling examples include $o(T)$ \cite{Walton14}, $\tilde{O}(1/T)$ \cite{KrishnasamySJS21}, and $O(1/T)$ \cite{StahlbuhkSM21}.

In this paper, we propose a transient alternative to queueing regret as an asymptotic  per-period queue length metric. Our metric captures the maximum increase in time-averaged wait time and thus focuses on the  initial periods. This contrasts with traditional queueing regret, which is mostly focused on the later periods (see Figure \ref{fig:col-motivation}).
As a result, depending on the specific application and context, our transient metric can be viewed as either a complement or a substitute for queueing regret. For example, when the time horizon is sufficiently long and there is a holding cost, the performance of algorithms should be evaluated based on queueing regret. Even then, our metric provides an important additional perspective to compare different algorithms. In contrast, when the horizon is short relative to the slack capacity (see \emph{traffic slackness} in \Cref{sec:model-single}), and large queue-lengths are disproportionately costlier, e.g., due to service-level agreements, reputational damages, or fairness concerns, then our transient metric may fully replace traditional queueing regret as the most relevant measurement.
Through the definition of our metric, and the accompanying near-tight bounds we develop, our work answers the two main questions that motivate this paper:

\emph{1. How can one measure the transient performance of learning algorithms in queueing systems?}

\emph{2. What learning algorithms have strong transient performance in general queueing systems?}

\begin{figure}[htbp]
  \centering  \includegraphics[width=.9\textwidth]{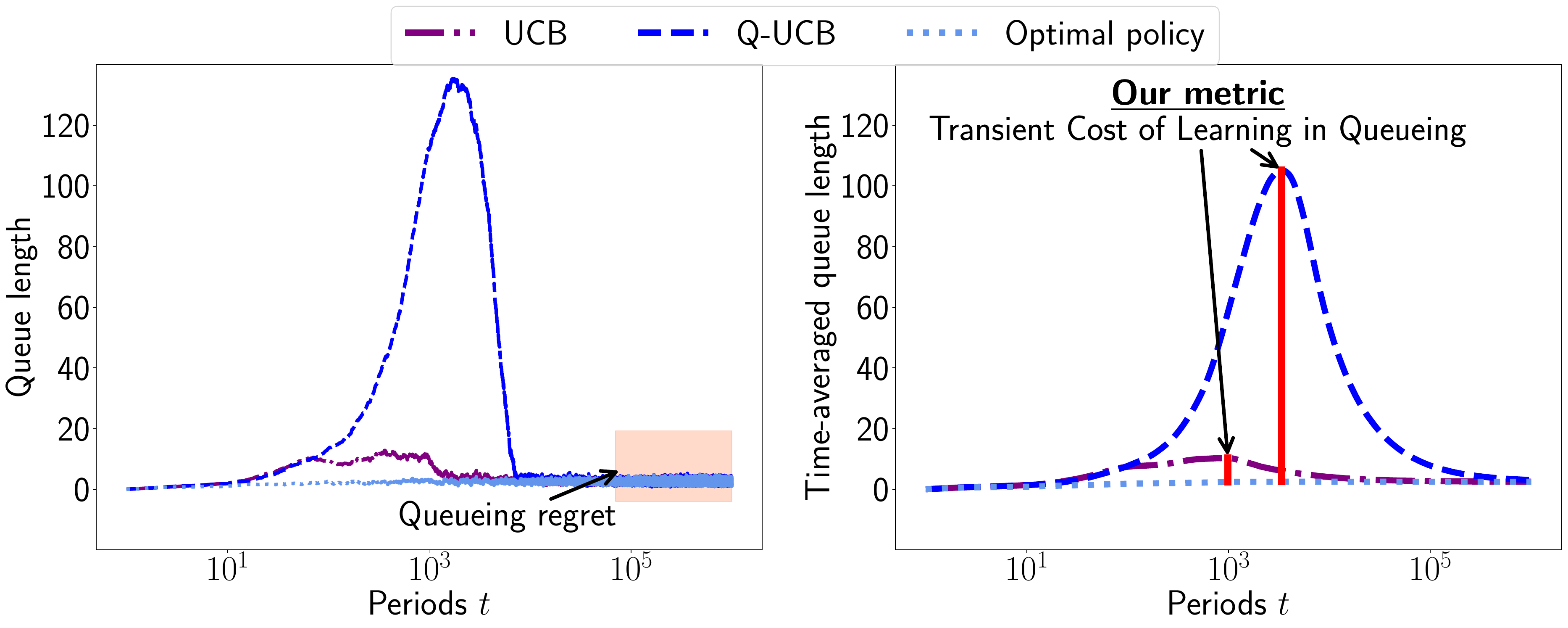}
  \\
  \caption{ Expected per-period and time-averaged queue lengths of UCB and Q-UCB \cite{KrishnasamySJS21} in a single-queue setting with $K=5,\lambda=0.45,\bolds{\mu} = (0.045,0.35,0.35,0.35,0.55)$; results are averaged over 30 runs. The difference between both algorithms' queue lengths is indistinguishable asymptotically (left figure) though they clearly differ in their learning efficiency for early periods as illustrated by \emph{Transient Cost of Learning in Queueing}, the metric that our work introduces (right figure).
  \label{fig:col-motivation} }
\end{figure}

\subsection{Our contribution}

\paragraph{Transient Cost of learning in queueing.} 
Tackling the first question, we propose the \emph{Transient Cost of Learning in Queueing} ($\colq$) as a way to measure the transient efficiency of a learning algorithm in queueing systems. The $\colq$
of a learning policy is defined as the maximum difference of its time-averaged queue length and that of any other policy (with knowledge of parameters) over the entire horizon (see Fig \ref{fig:col-motivation} on the right). In contrast to queueing regret, $\colq$ is 1) a finite-time metric that captures the learning efficiency in early periods and 2) focused on time-averaged queue length instead of per-period queue length. For any period $1,\ldots,T$, the time-averaged queue length is related to the job-average time in system by Little’s Law (see Appendix~\ref{app:clq-clw}); therefore, $\colq$ 
 can be equivalently viewed as a measurement of the maximum expected increase in wait time that any job experiences due to the DM's parameter uncertainty.
The formal definition of $\colq$ can be found in Section~\ref{sec:def-colq}.

\paragraph{Lower bound of $\colq$ (Theorem~\ref{thm:colq-lowerbound}).} To lower bound the transient cost of learning in queueing we rely on the simplest non-trivial stabilizable setting which involves one queue and $K$ servers. It is known that the expected queue length in steady-state scales as $O(1/\varepsilon)$ under the optimal policy, where $\varepsilon=\mu^\star-\lambda$ is the traffic slackness of the system. Fixing $\varepsilon$ and the number of servers $K$, we establish a worst-case lower bound $\Omega(\frac{K}{\varepsilon})$ of $\colq$. That is, for any $\varepsilon, K$ and a fixed policy, there exists a setting of arrival and service rates, such that the $\colq$ of this policy is at least $\Omega(\frac{K}{\varepsilon})$. Combined with the $O(1/\varepsilon)$ optimal queue length, this lower bound result shows that the effect of learning dominates the performance of queueing systems when $K$ is large (as it may increase the maximum time-averaged queue length by a factor of $K$). This is shown in Figure \ref{fig:col-motivation} (right) where the peak of time-averaged queue lengths of the optimal policy with knowledge of parameters is much lower than that of the other two policies: \textsc{Upper Confidence Bound} (\textsc{UCB}) \cite{AuerCF02} and and \textsc{Q-UCB} \cite{KrishnasamySJS21}).

\paragraph{An efficient algorithm for single-queue multi-server systems (Theorem \ref{thm:colq-ucb-single}).} Given the above lower bound, we show that the \textsc{UCB} algorithm attains an optimal $\colq$ up to a logarithmic factor in the single-queue multi-server setting, which was the main focus in \cite{KrishnasamySJS21} and \cite{StahlbuhkSM21}. Our analysis is motivated by Fig.~1 (right) where the time-averaged queue length initially increases and then decreases. Based on this pattern, we divide the horizon into an initial \emph{learning} stage and a later \emph{regenerate} stage. 

In the learning stage, the time-averaged queue length increases similar to the multi-armed bandit regret. We formalize this observation by coupling the queue under a policy $\pi$ with a nearly-optimal queue and show that their difference is captured by the \emph{satisficing regret} of policy $\pi$~\cite{russo2022satisficing}. Satisficing regret resembles the classical multi-armed bandit regret but disregards the loss of choosing a nearly optimal server; as a near-optimal server suffices  to obtain a queue length of $O(1 / \varepsilon)$ (see Eq.~\eqref{def:satis-regret}), a small satisficing regret would guarantee the desired queue length bound. Nevertheless, our learning stage bound is not sufficient as the satisficing regret eventually goes to infinity. 

In the regenerate stage, queue lengths decrease as the policy has learned the parameters sufficiently well; the queue then behaves similarly as under 
the optimal policy and stabilizes 
the system. To capture this observation, we use Lyapunov analysis and show that the time-averaged queue length for the initial $T$ periods scales as the optimal queue length, but with an additional term depending on the second moment of satisficing regret divided by $T$. Hence, as $T$ increases, the impact of learning gradually disappears. Combining the results in the learning and regenerate stages, we obtain a tight $\colq$ bound of UCB for the single-queue multi-server setting.

\paragraph{Efficient algorithms for multi-queue systems and queueing networks (Theorems \ref{thm:mw-ucb},\ref{thm:bp-ucb-col}).}
We next generalize the above result to multi-queue multi-server systems. In contrast to the single queue case, even with known rates, the optimal policy for a multi-queue multi-server system is non-trivial (and usually difficult) to find. A natural greedy policy, such as choosing a set of servers that maximizes instantaneous services as in the single-queue system, is known to have potentially unbounded queue length for a stabilizable system \cite{krishnasamy2018learning}. Designing an efficient learning policy for multi-queue systems is thus not straightforward. 

Instead, we build on the celebrated \textsc{MaxWeight} policy that stabilizes a multi-queue multi-server system with knowledge of system parameters \cite{tassiulas1992stability}. We design \textsc{MaxWeight-UCB} as a new algorithm to transform \textsc{MaxWeight} into a learning algorithm with appropriate estimates for system parameters and show that its $\colq$ scales near-optimally as $\tilde{\bigO}(1/\varepsilon)$ with respect to the traffic slackness $\varepsilon$ (Definition \ref{def:multi-slackness}). This result improves previous bounds for settings beyond a single-queue system. The best prior bound is $\tilde{\bigO}(1/\varepsilon^3)$ for the special case of bipartite queueing systems \cite{FreundLW22,YangSrikantYing} (or $\tilde{\bigO}(1/\varepsilon^2)$ with structural assumptions on service rates \cite{KrishnasamySJS21}). We extend our analysis for single-queue settings through a 
coupling approach that reduces the loss incurred by learning in a high-dimensional queue-length vector to a scalar-valued potential function and builds upon recent work of \cite{gupta2022greedy,wei2023constant}. 

Finally, we consider queueing networks that include multiple queues, multiple servers, and transitions of served jobs from servers to queues. For this setting, we  propose \textsc{BackPressure-UCB}, which incorporates online learning into the \textsc{BackPressure} algorithm \cite{tassiulas1992stability}. We prove that its $\colq$ also scales near-optimally as $\tilde{\bigO}(1/\varepsilon)$. To the best of our knowledge, this is the first efficient learning algorithm for general queueing networks (see related work for a discussion).

\subsection{Related work}

\begin{table}[]
\centering
\caption{Queueing regret, (implied) $\colq$ bounds, and settings of key related papers. The big-$O$ notation only includes scaling with respect to the time horizon $T$ for queueing regret and only includes scaling with respect the traffic slackness $\varepsilon$ for $\colq$.}
\label{tab:survey}
\begin{tabular}{|c|c|c|c|c|}
\hline
paper                     & queueing regret & TCLQ lower bound        & TCLQ upper bound             & Generality               \\ \hline
\cite{Walton14}           & $o(T)$          &                         &                              & Single queue             \\ \hline
\cite{KrishnasamySJS21}   & $O(\ln(T)/T)$   & $\Omega(\ln(1/\varepsilon))$             & $\tilde{O}(1/\varepsilon^2)$ & Single queue             \\ \hline
\cite{StahlbuhkSM21}      & $O(1/T)$        &                         & $\tilde{O}(1/\varepsilon^4)$ & Single queue             \\ \hline
\cite{YangSrikantYing}    &                 &                         & $\tilde{O}(1/\varepsilon^3)$ & Multi-queue \\ \hline
\cite{FreundLW22}         &                 &                         & $\tilde{O}(1/\varepsilon^3)$ & Multi-queue \\ \hline
\cite{nguyen2023learning} &                 &                         & $\tilde{O}(1/\varepsilon^3)$ & queueing network         \\ \hline
\textbf{This work}                 &                 & $\Omega(1/\varepsilon)$ & $\tilde{O}(1/\varepsilon)$   & queueing network         \\ \hline
\end{tabular}
\end{table}

A recent line of work studies online learning in queueing systems \cite{walton2021learning} and Table~\ref{tab:survey} summarizes the guarantees of papers that are most related to this work. To capture uncertainty in services, \cite{Walton14} studies a single-queue setting in which the DM selects a mode of service in each period and the job service time varies between modes (the dependence is a priori unknown and revealed to the DM after the service). The metric of interest is the \emph{queueing regret}, i.e., the difference of queue length between an algorithm and the optimal policy, for which the authors show a sublinear bound. \cite{KrishnasamySJS21} consider the same single-queue multi-server setting as ours and show that a forced exploration algorithm achieves a queueing regret scaling of $\tilde{O}(1/T)$ (under strong structural assumptions this result extends to multiple queues). \cite{StahlbuhkSM21} show that by probing servers when the queue is idle, it is possible to give an algorithm with queueing regret converging as $O(1/T)$. However, with respect to the traffic slackness $\varepsilon \to 0^+$, both bounds yield suboptimal $\colq$: \cite{KrishnasamySJS21} give at least $O(1/\varepsilon^2)$ and \cite{StahlbuhkSM21} give at least $O(1/\varepsilon^4)$ (see Appendices~\ref{app:qucb},~\ref{app:ssm}). In the analysis of \cite{KrishnasamySJS21}, forced exploration is used for low adaptive regret, i.e., regret over any interval \cite{HazanS09}; no such guarantee is known for adaptive exploration. But as noted by our Figure~\ref{fig:col-motivation} and~\cite[Figure 2]{KrishnasamySJS21}, an adaptive exploration algorithm like UCB has a better early-stage performance than Q-UCB. Using our framework in Section~\ref{sec:optimal-single}, we show that UCB indeed has a near-optimal $\colq$ that scales as $O(\frac{K}{\varepsilon}\ln\frac{K}{\varepsilon})$. Our framework also allows us to show that Q-UCB enjoys a $\colq$ scaling as $O(\frac{K}{\varepsilon}\ln\frac{K}{\varepsilon}+\ln^3\frac{K}{\varepsilon})$ (Appendix~\ref{app:qucb}). Though this improves the guarantee implied by \cite{KrishnasamySJS21} and shows that Q-UCB has both strong transient ($\colq$) and asymptotic (queueing regret) performance, the additional $O(\ln^3\frac{K}{\varepsilon})$ term  suggests an inefficiency in the transient regime that arises from forced exploration.

Focusing on the scaling of queueing regret, \cite{krishnasamy2018learning} and \cite{zhong2022learning} study scheduling in multi-queue settings (\cite{zhong2022learning} also consider job abandonment), \cite{choudhury2021job,fu2022optimal} study learning for a load balancing model, \cite{chen2023onlinelearning} study pricing and capacity sizing for a single-queue single-server model with unknown parameters. For more general settings, \cite{adler2023bayesian} design a Bayesian learning algorithm for Markov Decision Processes with countable state spaces (queueing systems are a special case) where parameters are sampled from a known prior over a restricted parameter space; in contrast, our paper does not assume any prior of the unknown parameters. The main difference between all of these works and ours is that we focus on how the maximum time-averaged queue lengths scale with respect to system parameters (traffic slackness and number of queues and servers), not on how the queue lengths scale as time grows. Apart from
the stochastic learning setting we focus on, there are also studies that tackle adversarial learning in queueing systems \cite{huang2023queue,liang2018minimizing}; these require different algorithms and analyses.

Beyond queueing regret, some papers focus on finite-time queue length guarantees. In a multi-queue multi-server setting, the MaxWeight algorithm has a polynomial queue length for stabilizable systems but requires knowledge of system parameters.  For a joint scheduling and utility maximization problem, \cite{NeelyRP12} combine MaxWeight with forced exploration to handle parameter uncertainty. By selecting a suitable window for sample collection, their guarantee corresponds to a $\colq$ bound of at least $O(K^4/\varepsilon^3)$ for our single-queue setting (see Appendix~\ref{app:nrp}). \cite{StahlbuhkSM19} study a multi-queue multi-server setting and propose a frame-based learning algorithm based on MaxWeight. They focus on a greedy approximation which has polynomial queue lengths when the system is stabilizable with twice of the arrival rates.  \cite{YangSrikantYing} consider a non-stationary setting and show that combining MaxWeight with discounted UCB estimation leads to stability and polynomial queue length that scales as $\tilde{O}(1/\varepsilon^3)$ (Appendix~\ref{app:ysy}). There is also a line of work studying decentralized learning in multi-queue multi-server settings. \cite{gaitonde2023price} assume queues are selfish and derive conditions under which a no-regret learning algorithm is stable; this is generalized to queueing networks in which queues and servers form a directed acyclic graph by \cite{FuHL22}. \cite{sentenac2021decentralized} allow collaborative agents and give an algorithm with maximum stability, although the queue length scales exponentially in the number of servers. \cite{FreundLW22} design a decentralized learning version of MaxWeight and show that the algorithm always stabilizes the system with polynomial queue lengths $\tilde{O}(1/\varepsilon^3)$ (Appendix~\ref{app:flw}). 
In contrast, our work shows for the centralized setting that MaxWeight with UCB achieves the near-optimal time-averaged queue length guarantee of $\tilde{O}(1/\varepsilon)$.

Our paper extends the ability of online learning to general single-class queueing networks \cite{BramsonDW21}. The literature considers different complications that arise in these settings, including jobs of different classes and servers that give service simultaneously to different jobs \cite{tassiulas1992stability,DBLP:journals/ior/DaiL05,BramsonDW21}. For the class of networks we consider, \textsc{BackPressure} can stabilize the system with knowledge of system parameters \cite{tassiulas1992stability}. As noted in \cite{BramsonDW21}, a potential drawback of \textsc{BackPressure} is its need of full knowledge of job transition probabilities. Our paper contributes to the literature by proposing the first \textsc{BackPressure}-based algorithm that stabilizes queueing networks without knowledge of system parameters.

Beyond our focus on uncertainties in services, an orthogonal line of work studies uncertainties in job types. \cite{alizamir2013diagnostic} consider a single server setting where an arriving job belongs to one of two types; but the true type is unknown and learned by services. They devise a policy that optimizes a linear function of the number of correctly identified jobs and the waiting time. \cite{bimpikis2019learning} study a similar setting with two types of servers where jobs can route from one server to the others and focuses on the impact on stability due to job type uncertainties. \cite{massoulie2018capacity, ShahGMV20} consider multiple job types and server types. Viewing Bayesian updates as job type transitions, they model the job learning process by queueing networks  and give stable algorithms based on \textsc{BackPressure}. \cite{johari2021matching,hsu2022integrated} consider online matchings between jobs with unknown payoffs and servers where the goal is to maximize the total payoffs subject to stability. As noted in \cite{massoulie2018capacity,ShahGMV20,johari2021matching}, a key assumption of this line of work is the perfect knowledge of server types (and job transition probability). Our result for queueing networks serves as a step to consider both server uncertainties and job uncertainties, at least in a context without payoffs.

Concurrent to our work, \cite{nguyen2023learning} propose a frame-based MaxWeight algorithm with sliding-window UCB for scheduling in a general multi-queue multi-server system with non-stationary service rates. With a suitable frame size (depending on the traffic slackness), they show stability of the algorithm and obtain a queue length bound of $\tilde{O}(1/\varepsilon^3)$ in the stationary setting (Appendix~\ref{app:nm}).

Our analytical framework relies on the distinction between learning and regenerate stages. This distinction was also observed in \cite{KrishnasamySJS21}, who call them ``early'' and ``late'' stages. In a heavy traffic regime with $\varepsilon \to 0^+$, they show that it takes at least $\Omega(K/\varepsilon)$ periods for the early stage to transition into the late stage and established a lower bound \cite[proposition 3]{KrishnasamySJS21} for the queueing regret in this period. Our analysis, especially our lower bound result, refines their understanding on the distinction between these two stages. First, the proof of our lower bound in Theorem~\ref{thm:colq-lowerbound} demonstrates that the early (learning) stage indeed lasts for at least $\Omega(K / \varepsilon^2)$ periods. Second, our lower bound on $\colq$ is stronger than what is implied by \cite[proposition 3]{KrishnasamySJS21}; see Remark~\ref{remark:compare}.

\section{Warm-up model: single-queue multi-Server systems}\label{sec:model-single}
We consider a sequential learning setting where a decision maker (DM) repeatedly schedules jobs to a set of servers of unknown quality over discrete time periods $t=1,2,\ldots$. For any $T$, we  refer to the initial $T$ periods as the time horizon $T$. To ease exposition, we first describe the simpler setting where there is only one job type (queue) and subsequently extend our approach to a general setting with multiple queues that interact through a network structure in Sections~\ref{sec:optimal-multi} and \ref{sec:optimal-network}.

\subsection{Model dynamics}
A single-queue multi-server system is specified by a tuple $(\set{K},\lambda,\bolds{\mu})$. There is one queue of jobs and a set of servers $\set{K}$ with $|\set{K}| = K$. The arrival rate of jobs is $\lambda$, that is, in each period there is a probability $\lambda$ that a new job arrives to the queue. The service rate of a server $k \in \set{K}$, that is, the probability it successfully serves the job it is scheduled to work on, is $\mu_k$. Let $Q(t)$ be the number of jobs at the start of period $t$. Initially there is no job and $Q(1) = 0$. 

Figure~\ref{fig:single-queue} summarizes the events that occur in each period $t$. If there is no job in the queue, i.e., $Q(t)=0$, then the DM selects no server; to ease notation, they select the null server $J(t)=\perp$. Otherwise, the DM selects a server $J(t)\in\set{K}$
\begin{figure}[h]
\centering
\hspace*{-1cm}
\scalebox{0.7}{
\begin{tikzpicture}[node distance=3cm, xshift=-2cm]

\tikzstyle{queue} = [rectangle, minimum width=3cm, minimum height=1cm, text centered, draw=black, fill=blue!30]
\tikzstyle{process} = [rectangle, minimum width=3cm, minimum height=1cm, text centered, draw=black, fill=orange!30]
\tikzstyle{startstop} = [rectangle, rounded corners, minimum width=3cm, minimum height=1cm, text centered, draw=black, fill=red!30]
\tikzstyle{arrow} = [thick,->,>=stealth]

\node (start) [startstop] {Start of Period $t$};
\node (queue) [queue, right of=start, xshift=1cm] {Queue: $Q(t)>0$};
\node (select) [process, right of=queue, xshift=2cm] {DM Selects Server $J(t)$};
\node (service) [process, right of=select, xshift=2cm] {Service Job from Queue};
\node (success) [process, below of=select, yshift=0.5cm, xshift=-1.5cm, text width=3cm] {Job Leaves System\\(Prob. $\mu_{J(t)}$)};

\node (arrival) [process, right of=success, xshift=2cm] {New Job Arrival (Prob. $\lambda$)};
\node (last) [startstop, right of=arrival, xshift=2cm] {Start of Period $t+1$};

\draw [arrow] (start) -- (queue);
\draw [arrow] (queue) -- node[anchor=south] {Yes} (select);
\draw [arrow] (queue) -- node[anchor=south] {No} (arrival);
\draw [arrow] (select) -- (service);
\draw [arrow] (service) -- node[anchor=west,xshift=0.3cm] {Successful} (success);
\draw [arrow] (success) -- (arrival);
\draw [arrow] (service) --node[anchor=south,xshift=1.3cm] {Unsuccessful} (arrival);
\draw [arrow] (arrival) -- (last);

\end{tikzpicture}
}
\caption{Flowchart for a single-queue multi-server system.}
\label{fig:single-queue}
\end{figure}
and requests service for the first job in the queue. The service request for period $t$ is successful with probability $\mu_{J(t)}$ and the
job then leaves the system; otherwise, it remains
in the queue. At the end of the period, a new job arrives with probability $\lambda$. We assume that arrival and service events are independent. Let $A(t)$ and $\{S_k(t)\}_{k \in \set{K}}$ be a set of independent Bernoulli random variables such that $\expect{A(t)}=\lambda$ and $\expect{S_k(t)} = \mu_k$ for $k \in \set{K}$; for the null server, $\mu_{\perp} = S_{\perp}(t) = 0$ for all $t$. The queue length dynamics are thus given by \begin{equation}\label{eq:dynamic-single}
    Q(t+1) = Q(t) - S_{J(t)}(t) + A(t).
\end{equation}

A non-anticipatory policy $\pi$ for the DM maps for every period $t$ the historical observations until $t$, i.e., $(A(\tau),S_{J(\tau)}(\tau))_{\tau < t}$, to a server $J(t) \in \set{K} \cup \{\perp\}$. We define $Q(t,\pi)$ as the queue length in period $t$ under policy $\pi$. The DM's goal is to select a non-anticipatory policy $\pi$ such that for any time horizon $T \geq 1$, the expected time-averaged queue length $\frac{1}{T}\sum_{t \leq T}\expect{Q(t,\pi)}$ is as small as possible. 

If service rates are known, the policy $\pi^{\star}$ selecting the server with the highest service rate $\mu^{\star}$ in every period (unless the queue is empty) minimizes the expected time-averaged queue length for any time horizon $T$~\cite{KrishnasamySJS21,StahlbuhkSM21}. If $\lambda \geq \mu^*$, even under $\pi^{\star}$, the expected time-averaged queue length goes to infinity as $T$ increases. We thus assume $\lambda < \mu^\star$, in which case, the system is \emph{stabilizable}, i.e., the expected time-averaged queue length under $\pi^\star$ is bounded by a constant for any time horizon $T$. We next define the \emph{traffic slackness} of this system:
\begin{definition}\label{def:sqms-slackness}
    The traffic slackness of a single-queue multi-server system is defined by $\varepsilon = \mu^\star - \lambda$.
\end{definition}
A larger traffic slackness implies that a system is easier 
to stabilize. It is known that the policy $\pi^\star$ obtains an expected time-averaged queue length of the order of $\frac{1}{\trafficslack}$~\cite[section 3.4.2]{Srikant_Ying_2014}.

\paragraph{Extensions to general queueing networks.} In Sections~\ref{sec:multi-model} we extend the single-queue multi-server system to a multi-queue multi-server system. In this system each queue has dedicated servers. Services of queues are coupled via the set of \emph{feasible schedules}: the DM can choose a subset of servers to work on jobs if and only if this subset is a feasible schedule. In Section~\ref{sec:network-model}, we further extend the multi-queue multi-server system to a queue network, where after a job gets service, it may transit back into another queue with certain probability. We discuss the generality of this model in Appendix~\ref{app:examples-network} by showing that it captures the model of \cite{tassiulas1992stability}.

\begin{remark}
Consistent with \cite{tassiulas1992stability} (see further discussion in Appendix~\ref{app:examples-network}), our model here (and similarly those in Sections~\ref{sec:multi-model} and \ref{sec:network-model}) requires the DM not to query any server if the queue is empty. This captures settings where it is infeasible or impractical to create \emph{fake} jobs. For example, in a call center, a job is like a customer. When there is no customer calling, the DM cannot query a server with a job as there is no customer. The same holds for a hospital setting: a doctor cannot help with a patient when there is no patient. However, for a telecommunication setting, like the ones in \cite{StahlbuhkSM21, KrishnasamySJS21}, it is possible to send a fake job through a channel to probe the service rate of that channel and thus they do not make the requirement like ours. Nevertheless, our results remain intact even when the DM can query a server with an empty queue. 
\end{remark}

\subsection{Objective: The Transient Cost of Learning in Queueing}\label{sec:def-colq}
We now define the \emph{Transient Cost of Learning in Queueing}, or $\colq$ as a shorthand, for our warm-up model. This metric captures the transient performance of a learning algorithm. In Sections~\ref{sec:optimal-multi} and \ref{sec:optimal-network}, we extend this definition to multi-queue multi-server systems and queueing networks.

Our starting point stems from the observation that an asymptotic metric (such as queueing regret), which measures performance in late periods, disregards the transient performance in initial periods (recall the left of Figure~\ref{fig:col-motivation}). Thus motivated, we define the \emph{Transient Cost of Learning in Queueing} (or $\colq$) as the maximum increase in expected time-averaged queue lengths under policy $\pi$ compared with the optimal policy. Specifically, we define the single-queue $\colq$ as:
\begin{equation}
\colq^{\single}(\lambda,\bolds{\mu},\pi) = \max_{T \geq 1} \frac{\sum_{t=1}^T \expect{Q(t,\pi) - Q(t,\pi^{\star})}}{T}.
\end{equation}
As shown in Figure~\ref{fig:col-motivation} (right), $\colq$ is a finite-time metric and explicitly takes into account how fast learning occurs in the initial periods.

Given that the traffic slackness measures the difficulty of stabilizing a system, we also consider the worst-case transient cost of learning in queueing over all pairs of $(\lambda,\bolds{\mu})$ with a fixed traffic slackness $\trafficslack$. In a slight abuse of notation, we overload $\colq^{\single}$ to also denote this worst-case value, i.e., 
\begin{equation}\label{def:tclq-single}
\colq^{\single}(K,\trafficslack,\pi) = \sup_{\lambda \in [0,1),\bolds{\mu} \in [0,1]^K\colon \lambda + \varepsilon \leq \max_k \mu_k} \colq^{\single}(\lambda,\bolds{\mu},\pi).
\end{equation}
Our goal is to design a policy $\pi$, without knowledge of the arrival rate, the service rates, and the traffic slackness, that achieves low worst-case transient cost of learning in a single-queue multi-server system. In Sections~\ref{sec:optimal-multi} and \ref{sec:optimal-network}, we extend this definition of transient cost of learning in queueing to multi-queue multi-server systems and queueing networks.

\begin{remark}
In addition to $\colq$, another potential metric of interest is the maximum increase in the expected \emph{job-averaged wait time}. 
In Appendix~\ref{app:clq-clw}, we show that as long as the arrival rate is not too small, bounding the two is indeed equivalent (up to constant factors).
\end{remark}
\begin{remark}\label{remark:slackness}
This work and prior work on learning in queues (e.g. \cite{KrishnasamySJS21, StahlbuhkSM21, YangSrikantYing}) fundamentally assumes that the underlying system is known to be stabilizable ($\varepsilon > 0$). However, it is possible that the DM does not have such prior knowledge or that the system is actually under-provisioned ($\varepsilon < 0$). We explore this question by understanding (i) how our theory on $\colq$ can serve as a tool for the DM to infer the traffic slackness of a system (Appendix~\ref{app:slackness-inference}) and (ii) the optimality of $\textsc{UCB}$ in the single-queue multi-server setting even when the traffic slackness is negative (Appendix~\ref{app:slackness-negative}).
\end{remark}

\section{Optimal transient cost of learning for our warm-up model}\label{sec:optimal-single}
In this section, we provide a tight characterization of the transient cost of learning in queueing for our warm-up model. We first state a lower bound on the $\colq$ of any policy (Theorem~\ref{thm:colq-lowerbound}) and then prove that the UCB policy matches that lower bound (Theorem~\ref{thm:colq-ucb-single}).

\noindent\textbf{Lower Bound.}  Our first result establishes a lower bound on $\colq^{\single}(K,\trafficslack,\pi)$. In particular, for any feasible policy $\pi$, we show a lower bound of $\Omega(\frac{K}{\trafficslack})$ for sufficiently large $K$. With known parameters, the optimal time-averaged queue length is of the order of $\nicefrac{1}{\trafficslack}$. Hence, our result shows that the transient cost of learning is non-negligible in queueing systems when there are many servers. For fixed $K$ and $\varepsilon$, our lower bound considers the transient cost of learning of the worst-case setting and is instance-independent.
\begin{theorem}\label{thm:colq-lowerbound}
For any $K \geq 2^{14},\trafficslack \in (0,0.25]$ and feasible policy $\pi$, $\colq^{\single}(K,\trafficslack,\pi) \geq \frac{K}{2^{14}\trafficslack}$.
\end{theorem}

Although our proof is based on the distribution-free lower bound $\Omega(\sqrt{KT})$ for classical multi-armed bandits \cite{AuerCFS02}, this result does not apply directly to our setting. 
In particular, suppose the queue in our system is never empty. Then  the accumulated loss in service of a policy is exactly the \emph{regret} in bandits and the lower bound implies that any feasible policy serves at least $\Omega(\sqrt{KT})$ jobs fewer than the optimal policy in the first $T$ periods. However, due to the traffic slackness, the queue does get empty under the optimal policy, and in periods when this occurs, the optimal policy also does not receive service. As a result, the difference in queue length  between a learning policy and the optimal policy could be lower than $\Theta(\sqrt{KT})$ despite the loss of service. 

We next discuss the intuition of our proof (formal proof in Appendix~\ref{app:proof-lowerbound}). Fixing $K$ and~$\trafficslack$, suppose the gap in service rates between the optimal server and others is  $2\trafficslack$. Then for any $t$ in a time horizon $T = O(\frac{K}{\trafficslack^2})$, the number of arrivals in the first $t$ periods  is around $\lambda t$ and the potential service of the optimal server is around $(\lambda+\trafficslack)t$. By the multi-armed bandit lower bound, the total service of a feasible policy is at most around $\lambda t + \trafficslack t - \sqrt{Kt} \leq \lambda t$ since $t \leq T = O(\frac{K}{\trafficslack^2})$. Therefore, the combined service rate of servers chosen in the first $T$ periods, i.e., $\sum_{t \leq T} \mu_{J(t)}$,  is strictly bounded from above by the total arrival rate $\lambda t$. A carefully constructed example shows that the number of unserved jobs is around $\trafficslack t$ for every $t \leq T$. As a result, the time-averaged queue length for the horizon $T$ is of the order of $\trafficslack T = O(\frac{K}{\trafficslack})$.

\begin{remark}\label{remark:compare}
In \cite[Proposition 3]{KrishnasamySJS21}, the authors established an instance-dependent lower bound on $\expect{Q(t,\pi) - Q^{\star}(t)}$. Their implied $\colq$ lower bound of $\Omega(K\ln(1/\varepsilon))$ is weaker than ours (see Appendix~\ref{app:lowerbound-compare}) and is constrained to $\alpha-$consistent policies whereas ours does not.
\end{remark}

\noindent\textbf{Upper Bound.} Motivated by the lower bound, we propose an efficient algorithm in the warm-up model with a focus on heavy-traffic optimality, i.e., ensuring $\colq=\tilde{O}(\nicefrac{1}{\varepsilon})$ as $\varepsilon \to 0^+$. Such a heavy-traffic regime has been the focus of a vast literature (see \cite[chapter 10]{Srikant_Ying_2014}) and is practically important as a well-provisioned service system operates with a small traffic slackness to reduce capacity cost.\cite{borst2004dimensioning} Moreover, it serves as a stress test because stabilizing the system with unknown parameters is more difficult when the traffic slackness is lower as an efficient algorithm must strive to learn parameters more accurately. Our algorithm is the classical \textsc{Upper Confidence Bound} policy (\textsc{UCB}, Algorithm~\ref{algo:single-ucb}) applied to the queueing setting. In each period $t$, when the queue is non-empty, \textsc{UCB} selects a server with the highest upper confidence bound estimation $\bar{\mu}_k(t) = \min\left(1,\hat{\mu}_k(t)+\sqrt{\frac{2\ln(t)}{C_k(t)}}\right)$ where $\hat{\mu}_k(t)$ is the sample mean of services and $C_k(t)$ is the number of times server $k$ is selected in the first $t-1$ periods. 

\begin{algorithm}[H]
\LinesNumbered
\DontPrintSemicolon
\caption{\textsc{UCB} for a single-queue multi-server system 
\label{algo:single-ucb}
}
\nl Sample mean $\hat{\mu}_{k}(1) \gets 0$, number of samples $C_k(1) \gets 0$ for $k \in \set{K}\cup \{\perp\}$, queue $Q(1)\gets 0$\;
\For{$t = 1\ldots$}{
\nl     $\bar{\mu}_{k}(t) = 
    \min\left(1, \hat{\mu}_{k}(t) + \sqrt{\frac{2\ln(t)}{C_k(t)}}\right), \forall k \in \set{K}$\label{line:ucb-esti}\;
\nl     \textbf{if } $Q(t) > 0$ \textbf{ then } $J(t) \gets \arg\max_k \bar{\mu}_k(t)$\;
\nl     \textbf{else }
    $J(t) \gets \perp$\;
   \tcc{Update queue length \& estimates based on  $S_{J(t)}(t)$, $A(t)$, and $J(t)$}
   $Q(t+1) \gets Q(t) - S_{J(t)}(t) + A(t)$ \\
    $C_{J(t)}(t+1) \gets C_{J(t)}(t)+1, \quad \hat{\mu}_{J(t)}(t+1) \gets \frac{C_{J(t)}(t)\hat{\mu}_{J(t)}(t)+S_{J(t)}(t)}{C_{J(t)}(t+1)}$\\
    \textbf{for } $k\neq J(t)$ \textbf { set } $C_k(t+1)\gets C_k(t), \quad \hat{\mu}_k(t+1)\gets \hat{\mu}_k(t)$
}
\end{algorithm}

We show that $\textsc{UCB}$ achieves near-optimal $\colq=\tilde{O}(\frac{K}{\varepsilon})$ for any $K$ and $\varepsilon$ in the single-queue multi-server setting with no prior information of any system parameters. 
\begin{theorem}\label{thm:colq-ucb-single}
For any $K \geq 1,\varepsilon \in (0,1]$, $\colq^{\single}(K,\varepsilon, \textsc{UCB}) \leq \frac{323K+64K(\ln K + 2\ln \nicefrac{1}{\varepsilon})}{\varepsilon}.$
\end{theorem}
To prove Theorem~\ref{thm:colq-ucb-single}, we establish an analytical framework to upper bound $\colq^{\single}$ for any policy $\pi$ by considering separately the initial \emph{learning} stage and the later \emph{regenerate} stage. The two stages are separated by a parameter $T_1$ that appears in our analysis: intuitively,  during the learning stage ($t<T_1$), the loss in total service of a policy compared with the optimal server's outweighs the slackness $\varepsilon$ of the system (Definition~\ref{def:sqms-slackness}), i.e., $\sum_{\tau=1}^t \mu^\star-\mu_{J(\tau)}>t\varepsilon$ and thus the queue length grows linearly with respect to the left-hand side. 
After the learning stage ($t>T_1$), when $\sum_{\tau=1}^t \mu^\star-\mu_{J(\tau)}<t\varepsilon$, the queue regenerates to a constant length independent of $t$. To prove the $\tilde{O}(\frac{K}{\varepsilon})$ bound on $\colq^{\single}$, we first couple the queue with an ``auxiliary'' queue where the DM always chooses a nearly optimal server in the learning stage. Then we utilize a Lyapunov analysis to bound the queue length during the regenerate stage. 

The framework establishes a connection between $\colq^{\single}(\lambda,\bolds{\mu},\pi)$ and the \emph{satisficing regret} defined as follows. For any horizon $T$, the satisficing regret $\sar^{\single}(\pi,T)$ is the total service rate gap between the optimal server and the server selected by $\pi$ except for the periods where the gap is less than $\frac{\varepsilon}{2}$ or the queue length is zero. That is, the selected server is satisficing as long as its service rate is nearly optimal or the queue is empty. To formally define it, we denote $\max(x,0)$ by $x^+$ and define the satisficing regret of a policy $\pi$ over the first $T$ periods by
\begin{equation}\label{def:satis-regret}
\sar^{\single}(\pi,T) = \sum_{t=1}^T \left(\mu^\star - \mu_{J(t)} - \frac{\varepsilon}{2}\right)^+\indic{Q(t) \geq 1}
\end{equation}
We use the satisficing regret instead of the canonical regret definition $\sum_{t = 1}^T (\mu^\star - \mu_{J(t)})$ (which measures $\mu^\star - \mu_{J(t)}$ without zeroing it out when it is small) for two reasons. First, to have a queue length of $O(\frac{1}{\varepsilon})$, it is sufficient to identify a \emph{good enough} server whose service rate is $\frac{\varepsilon}{2}$ above the arrival rate. The DM does not need to learn the best server.  Second, optimal bounds on regret are either instance-dependent $\set{O}\left(\sum_{k\colon \mu_k < \mu^\star}\left(\frac{1}{\mu^\star - \mu_k}\ln T\right)\right)$ \cite{AuerCF02}  or instance-independent $O(\sqrt{KT})$ \cite{AuerCFS02}. Both are futile to establish a $\tilde{\set{O}}(\frac{K}{\varepsilon})$ bound for $\colq^{\single}(K,\varepsilon)$: The first bound depends on the minimum gap (which can be infinitesimal), whereas the second is insufficient as we explain in the discussion after Lemma~\ref{lem:single-queue-regen}. Such a satisficing regret notion was also employed by \cite{russo2022satisficing}, who consider an infinite horizon Bayesian multi-armed bandit setting. 

We connect the time-averaged queue length of the system with the satisficing regret of the policy via Lemma~\ref{lem:single-queue-learn} (for the learning stage) and Lemma~\ref{lem:single-queue-regen} (for the regenerate stage). Lemma~\ref{lem:single-queue-learn} explicitly bounds the expected queue length through the expected satisficing regret; this is useful during the learning stage but does not give a strong bound for the regenerate stage. Lemma~\ref{lem:single-queue-regen} gives a bound that depends on $\frac{\sar^{\single}(\pi,T)^2}{T}$, and is particularly useful during the latter regenerate stage. We then show that the satisficing regret of $\textsc{UCB}$ is $O(\frac{K\ln T}{\varepsilon})$ (Lemma~\ref{lem:sar-ucb}). Combining these results, we establish a tight bound for the $\colq$ of $\textsc{UCB}$.

Formally, Lemma~\ref{lem:single-queue-learn} shows that the expected queue length under $\pi$ in period $t$ is at most that under a nearly optimal policy plus the expected satisficing regret up to that time.  
\begin{lemma}\label{lem:single-queue-learn}
For any policy $\pi$ and horizon $T$, we have $\frac{\sum_{t=1}^T\expect{Q(t)}}{T} \leq \frac{3}{\varepsilon} + \expect{\sar^{\single}(\pi,T)}$.
\end{lemma}
Lemma~\ref{lem:single-queue-learn} is established by coupling the queue with an auxiliary queue that always selects a nearly optimal server. However, it cannot provide a useful bound on $\colq$. For large~$T$, it is known that  $\expect{\sar^{\single}(\pi,T)}$ must grow with a rate of at least $\bigO(\frac{\log T}{\varepsilon})$ \cite{lai1985asymptotically}. Hence, Lemma~\ref{lem:single-queue-learn} only meaningfully bounds the queue length for small~$T$ (learning stage). For large~$T$ (regenerate stage), we instead have the following bound (Lemma~\ref{lem:single-queue-regen}).
\begin{lemma}\label{lem:single-queue-regen}
For any policy $\pi$ and horizon $T$, we have $\frac{\sum_{t=1}^T\expect{Q(t)}}{T} \leq \frac{4}{\varepsilon} + \frac{8}{\varepsilon^2}\cdot\frac{\expect{\sar^{\single}(\pi,T)^2}}{T}$.
\end{lemma}
This lemma shows that the impact of learning, reflected by $\expect{\sar^{\single}(\pi,T)^2}$, decays at a rate of $\frac{1}{T}$. Therefore, as long as $\sar^{\single}(\pi,T)^2$ is of a smaller order than $T$, the impact of learning eventually disappears. This also explains why the instance-independent $O(\sqrt{KT})$ regret bound for multi-armed bandits is insufficient for our analysis: the second moment of the regret scales linearly with the horizon and does not allow us to show a decreasing impact of learning on queue lengths.

Lemma~\ref{lem:single-queue-regen} suffices to show stability ($\lim_{T \to \infty}\frac{\sum_{t=1}^T\expect{Q(t)}}{T} < \infty$), but gives a suboptimal bound for small $T$. Specifically, when $\sar^{\single}(\pi,T)^2 \gtrapprox T$, this bound is of the suboptimal order of $\Omega(\frac{1}{\varepsilon^2})$. We thus need both Lemma~\ref{lem:single-queue-learn} and Lemma~\ref{lem:single-queue-regen} to establish a tight bound on the $\colq$. 

The following result bounds the first and second moments of the satisficing regret of $\textsc{UCB}$.
\begin{lemma}\label{lem:sar-ucb}
For any horizon $T$, we have
\[(i)\quad \expect{\sar^{\single}(\textsc{UCB},T)} \leq \frac{16K(\ln T+2)}{\varepsilon},~ \quad (ii) \quad \expect{\sar^{\single}(\textsc{UCB},T)^2} \leq \frac{2^9K^2(\ln T+2)^2}{\varepsilon^2}.\]
\end{lemma}

\begin{proof}[Proof of Theorem~\ref{thm:colq-ucb-single}]
Fix $K,\varepsilon$ and any pair of $\lambda,\bolds{\mu} = (\mu_1,\ldots,\mu_K)$ such that $\max_{k \in \set{K}} \mu_k = \lambda+\varepsilon$. Let $T_1 = \left\lfloor \left(\frac{2^{12}K^2}{\varepsilon^4}\right)^2\right\rfloor$. We establish an upper bound on $\colq^{\single}(\lambda,\bolds{\mu},\textsc{UCB})$ by bounding the time-averaged queue length for $T \leq T_1$ (learning stage) and $T \geq T_1$ (regenerate stage) separately.

For $T \leq T_1$, we have 
\begin{align*}
\frac{1}{T}\sum_{t\leq T}\expect{Q(t)} &\leq \frac{3}{\varepsilon} + \expect{\sar^{\single}(\textsc{UCB}, T)} \tag{Lemma~\ref{lem:single-queue-learn}}\\
&\leq \frac{3}{\varepsilon} + \frac{16K(\ln T_1 + 2)}{\varepsilon} \tag{Lemma~\ref{lem:sar-ucb} \emph{(i)} and $T \leq T_1$} \\ &\leq \frac{3+32K}{\varepsilon} + \frac{32K\ln\left(\nicefrac{2^{12}K^2}{\varepsilon^4}\right)}{\varepsilon} \tag{$T_1 \leq \left(\frac{2^{12}K^2}{\varepsilon^4}\right)^2$ by definition}\\
&\leq \frac{323K+64K(\ln K + 2\ln \nicefrac{1}{\varepsilon})}{\varepsilon}.
\end{align*}

For $T > T_1$, we have
\begin{align}
\frac{1}{T}\sum_{t\leq T}\expect{Q(t)} &\leq \frac{4}{\varepsilon}+\frac{8}{\varepsilon^2}\frac{\expect{\sar^{\single}(\textsc{UCB},T)^2}}{T} \tag{Lemma~\ref{lem:single-queue-regen}}\\
&\leq \frac{4}{\varepsilon} + \frac{2^{12}K^2(\ln T+2)^2}{\varepsilon^4 T}  \tag{Lemma~\ref{lem:sar-ucb} \emph{(ii)}}\\
&\leq \frac{4}{\varepsilon} + \frac{2^{12}K^2}{\varepsilon^4 \sqrt{T}} \tag{Fact~\ref{fact:lnt-sqrt-prop} \emph{(i)} and $T \geq T_1\geq 50000$} \\
& \leq \frac{5}{\varepsilon} \qquad\leq \frac{323K+64K(\ln K + 2\ln \nicefrac{1}{\varepsilon})}{\varepsilon}. \label{eq:bound-for-later} 
\end{align}
Combining these cases we bound $\colq^{\single}(K,\varepsilon,\textsc{UCB})$ by noting that, for any pair of $\lambda,\bolds{\mu}$ with $\max_{k \in \set{K}}\mu_k = \lambda+\varepsilon$, $\colq^{\single}(\lambda,\bolds{\mu},\textsc{UCB}) \leq \max_T \frac{1}{T}\sum_{t\leq T}\expect{Q(t)} \leq  \colq^{\single}(K,\varepsilon,\textsc{UCB})$.
\end{proof}
\begin{remark}
Although the $\colq$ metric is focused on the entire horizon, our analysis extends to bounding the maximum expected time-averaged queue lengths in the later horizon, which is formalized as $\max_{T \geq T_1} \frac{\sum_{t\leq T} \expect{Q(t)}}{T}$ for any $T_1$. In particular, for $T_1 \geq \left(\frac{2^{12}K^2}{\varepsilon^4}\right)^2$, \eqref{eq:bound-for-later} shows that $\max_{T \geq T_1} \frac{\sum_{t\leq T} \expect{Q(t)}}{T} \leq \frac{5}{\varepsilon}$; UCB thus enjoys the optimal asymptotic queue length scaling of $O(\frac{1}{\varepsilon})$.
\end{remark}
\begin{remark}
Although our main focus is on optimal scaling with respect to $\varepsilon$, Theorem~\ref{thm:colq-ucb-single} involves large constants. Replacing the UCB algorithm by a more efficient bandit algorithm, e.g., the KL-UCB algorithm in \cite{garivier2011kl}, may ease this concern as it may give a better constant for the mean satisficing regret in Lemma~\ref{lem:sar-ucb}. That said, a matching constant between the upper bound of $\colq$ and the lower bound in Theorem~\ref{thm:colq-lowerbound} would necessitate a more refined analysis. This is because our proof requires bounding the second moment of the satisficing regret in Lemma~\ref{lem:sar-ucb}. However, as discussed in \cite[section 4.5]{fan2024fragility}, a bandit algorithm optimized for the mean regret must have its second moment of regret scale roughly as $\Omega(T)$. If the same were to hold for satisficing regret, it would lead to a suboptimal $\Omega(\frac{1}{\varepsilon^2})$ bound on $\colq$ in Lemma~\ref{lem:single-queue-regen}.
\end{remark}

\subsection{Queue length bound in the learning stage (Lemma~\ref{lem:single-queue-learn})}\label{sec:single-queue-learn}
To upper bound the queue length during the learning stage (Lemma~\ref{lem:single-queue-learn}), we couple the queue operated under a policy $\pi$ with a queue that is operated near optimally. Intuitively, the difference in queue lengths between two queues is upper bounded by the difference in their total services. In particular, consider a fictitious single-queue single-server system $\{\tilde{Q}(t)\}$  that has the same arrival realization of the current process $\{Q(t)\}$ but with one server whose service rate is $\mu^\star - \frac{\varepsilon}{2}$. Denote the realization of services by $\tilde{S}(t)$. The server in this fictitious system is slightly less efficient than the optimal server in the original system. The reason to couple with such a worse-performing system is twofold: 1) its queue length is still of the order $O(\frac{1}{\varepsilon})$ and 2) the difference in services between the original system and this system is captured by the satisficing regret $\sar^{\single}(\pi,T)$ \eqref{def:satis-regret}. That is, when the chosen server under $\pi$ is nearly optimal (larger than $\mu^\star-\frac{\varepsilon}{2}$), we do not count the difference in service for this period at all. Such a shift of benchmark from the optimal $\mu^\star$ to $\mu^\star - \frac{\varepsilon}{2}$ is essential for our analysis to avoid a dependence of the minimal gap between servers' service rates and the optimal rate (which must exist if one uses the optimal service rate as benchmark as in multi-armed bandits \cite{lai1985asymptotically}). A final piece of the proof is to couple the service process of $\{Q(t)\}$ and $\{\tilde{Q}(t)\}$ (recall that their arrival process is the same), which we show below.
\begin{proof}[Proof of Lemma~\ref{lem:single-queue-learn}]
Let $\{U(t)\}_{t \geq 1}$ be a sequence of independent random variables of uniform distribution over $[0,1]$. Then we generate $S_k(t)$ by setting $S_k(t) = \indic{U(t) \leq \mu_k}$; similarly $\tilde{S}(t) = \indic{U(t) \leq \mu^\star - \frac{\varepsilon}{2}}$. Initially $\tilde{Q}(1) = 0$. Although this coupling introduces dependency between $\{S_k(t)\}_{k \in \set{K}}$ in a period $t$, it does not affect the distribution of the queue length process since in each period the DM selects at most one server (see the argument in EC. 1.1 of \cite{KrishnasamySJS21}). 

The dynamic of $\{\tilde{Q}(t)\}$ is given by $\tilde{Q}(t+1)=(\tilde{Q}(t)-\tilde{S}(t))^++A(t)$. Then by the dynamic of $\{Q(t)\}$ in \eqref{eq:dynamic-single}, the difference between the two queues for every period $t > 1$ is
\begin{align}
Q(t+1)-\tilde{Q}(t+1) &= Q(t)-S_{J(t)}(t)+A(t) - \left(\tilde{Q}(t)-\tilde{S}(t)+A(t)+\indic{\tilde{Q}(t)-\tilde{S}(t)=-1}\right) \nonumber\\
&\leq\left( Q(t)-\tilde{Q}(t)+\tilde{S}(t)-S_{J(t)}(t)\right)\indic{Q(t)\geq 1}, \label{eq:dynamic-difference}
\end{align}
where the inequality is because we have $Q(t+1) \geq A(t) = Q(t)$ when $Q(t) = 0$. For any fixed period $t$, define $\emp(t)=\max \{\tau < t \colon Q(\tau) = 0\}$ as the latest period before $t$ such that the queue length is zero. Applying \eqref{eq:dynamic-difference} recursively from $t$ backward to $\emp(t)$, we get 
\begin{align}
Q(t)-\tilde{Q}(t) \leq \sum_{\tau=\emp(t)+1}^{t-1} (\tilde{S}(\tau) - S_{J(\tau)}(\tau)) &\leq  \sum_{\tau=\emp(t)+1}^{t-1} (\tilde{S}(\tau) - S_{J(\tau)}(\tau))\indic{\mu_{J(\tau)} \leq \mu^\star - \frac{\varepsilon}{2}} \nonumber\\
&\hspace{-0.5in}\leq \sum_{\tau=1}^{t-1} (\tilde{S}(\tau) - S_{J(\tau)}(\tau))\indic{\mu_{J(\tau)} \leq \mu^\star - \frac{\varepsilon}{2}}\indic{Q(\tau) \geq 1}.\label{eq:single-coupling-bound}
\end{align}
where the first inequality follows because $Q(\tau) \geq 1$ for every $\tau \in \{\emp(t) + 1,\ldots,t-1\}$ and thus the indicator in \eqref{eq:dynamic-difference} is 1. The second inequality holds due to the coupling that ensures $\tilde{S}_{J(\tau)}(\tau) = \indic{U(\tau) \leq \mu_{J(\tau)}}$ and $\tilde{S}(\tau) = \indic{U(\tau) \leq \mu^\star - \frac{\varepsilon}{2}}$ and thus $\tilde{S}(\tau) \geq \tilde{S}_{J(\tau)}(\tau)$ if and only if $\mu_{J(\tau)} \leq \mu^\star - \frac{\varepsilon}{2}$. Since $\tilde{S}(\tau)$ and $\{S_k(\tau)\}_{k \in \set{K}}$ are independent from $Q(\tau)$ and $J(\tau)$, taking expectation on \eqref{eq:single-coupling-bound} gives
\[
\expect{Q(t)-\tilde{Q}(t)} \leq \expect{\sum_{\tau=1}^{t-1} (\mu^\star - \frac{\varepsilon}{2} - \mu_{J(\tau)})^+ \indic{Q(\tau) \geq 1}} \leq \expect{\sar^{\single}(\pi,t)}.
\]
Lemma~\ref{lem:bound-optimal-queue} bounds $\frac{\sum_{t=1}^T \expect{\tilde{Q}(t)}}{T} \leq \frac{2\lambda}{\varepsilon} + \frac{1}{2} \leq \frac{3}{\varepsilon}$. Hence, for any horizon $T$, we have $\frac{\sum_{t=1}^T \expect{Q(t)}}{T} \leq$
\begin{align*}
 \frac{\sum_{t=1}^T \expect{\tilde{Q}(t) + \sar^{\single}(\pi,t)}}{T} &\leq \frac{\sum_{t=1}^T \expect{\tilde{Q}(t)}}{T} + \expect{\sar^{\single}(\pi,T)} 
&\leq \frac{3}{\varepsilon} + \expect{\sar^{\single}(\pi,T)}.
\end{align*}
\end{proof}

\subsection{Queue length bound in the regenerate stage (Lemma~\ref{lem:single-queue-regen})}
We use Lyapunov analysis to bound the queue length in the regenerate stage. Let $V(t) = Q(t)^2$. Classical Lyapunov analysis for queues without learning (see \cite{Srikant_Ying_2014}) relies on the fact that the drift, $\expect{V(t+1) - V(t)}$, is upper bounded by a constant term plus $-\varepsilon' Q(t)$ where $\varepsilon' > 0$ is similar to traffic slackness in our model. As a result, if $Q(t)$ is large, there is a strong negative drift for the system to pull back $V(t)$ which allows for an upper bound on the queue length. However, in our case,  the drift in each period is given by $-\varepsilon Q(t) + (\mu^\star - \mu_{J(\tau)})Q(t)$ where the second term captures the chance of not selecting the optimal server due to learning. We then need to bound the expected total additional drift $\expect{\sum_{t \leq T} (\mu^\star - \mu_{J(\tau)})Q(t)}$. In contrast to classical bandit regret, there is now a dependence on the queue length for every error the DM makes. Since queue lengths are not bounded, a single error can have unbounded effect on the drift. To address this challenge, our key insight is that the additional drift can be approximately bounded by $O(\varepsilon)\sum_{t\leq T}\expect{Q(t)} + \expect{\sar^{\single}(\pi,T)^2}$ where the first term is canceled out by the negative drift already implied by the traffic slackness. As a result, what remains is the second term which does not depend on the queue length. This separation is motivated by \cite{FreundLW22}, but we provide a tighter and more systematic derivation here, which allows our method to improve the final bound and generalize to more difficult settings in later sections.
The following lemma connects the maximum queue length and the sum of queue lengths over the first $T$ periods.
\begin{lemma} \label{lem:single-connect-max-sum}
    For any horizon $T$ and every sample path, we have $\sum_{t=1}^T Q(t) \geq \frac{(\max_{t \leq T} Q(t))^2}{2}$.
\end{lemma}
\begin{proof}
The queueing dynamics guarantee $Q(t+1) = Q(t)+A(t)-S_{J(t)}(t) \leq Q(t)+A(t)\leq Q(t)+1$ for every period $t$, i.e., the queue lengths increase by at most one per period. We denote by $t_{\max}$ a period in which the queue attains its maximum length over the first $T$ periods, i.e., $t_{\max} \in \arg\max_{t \leq T} Q(t)$, and by $Q_{\max}=\max_{t \leq T} Q(t)=Q(t_{\max})$ that length. Then, $t_{\max} \geq Q(t_{\max})$, as queue lengths change by at most 1, and thus
\begin{align*}
\sum_{t=1}^T Q(t) &\geq \sum_{t = t_{\max} - Q_{\max}+1}^{t_{\max}} Q(t) = \sum_{i=0}^{Q_{\max}-1} Q(t_{\max}-i+1) \geq \sum_{i=0}^{Q_{\max} - 1} (Q(t_{\max}) - i + 1)
\end{align*}
where the first inequality is because $t_{\max} \leq T$ and $Q_{\max} \leq t_{\max}$; the equation is by letting $i = t_{\max} - t$; and the last inequality holds because the queue length increases by at most $1$ per period. Bounding the last term from below by $Q_{\max}^2 - \frac{Q_{\max}^2}{2} = \frac{Q_{\max}^2}{2}$ completes the proof.
\end{proof}

\begin{proof}[Proof of Lemma~\ref{lem:single-queue-regen}]
Recall that $V(t) = Q^2(t)$. The drift $\expect{V(t+1) - V(t)}$ is upper bounded by
\begin{align}
\expect{V(t+1)-V(t)} &= \expect{Q(t+1)^2 - Q(t)^2} = \expect{(Q(t)-S_{J(t)}(t)+A(t))^2 - Q(t)^2} \nonumber\\
&\leq \expect{A(t)^2}+\expect{S_{J(t)}^2}+2\expect{Q(t)(A(t) - S_{J(t)}(t))} \\
&\leq 2 + 2\expect{Q(t)(\lambda - \mu_{J(t)})} \leq 2 - \varepsilon\expect{Q(t)} + 2\expect{Q(t)(\mu^\star - \nicefrac{\varepsilon}{2} - \mu_{J(t)})} \nonumber\\
&\leq 2 - \varepsilon\expect{Q(t)} + 2\expect{Q(t)(\mu^\star - \nicefrac{\varepsilon}{2} - \mu_{J(t)})^+\indic{Q(t) \geq 1}}. \label{eq:prof-single-queue-regen-drift-bound}
\end{align}
where the second-to-last inequality uses the fact that $\mu^\star = \lambda + \varepsilon$. Recalling that $Q(1)=V(1)=0$, for a fixed horizon $T$, summing across $t = 1,\ldots,T$ gives 
\begin{align}
0 &\leq \expect{V(T+1)} - \expect{V(1)} = \sum_{t=1}^T \expect{V(t+1)-V(t)} \nonumber\\
&\overset{\eqref{eq:prof-single-queue-regen-drift-bound}}{\leq} 2T - \varepsilon\sum_{t\leq T} \expect{Q(t)}+2\sum_{t\leq T} \expect{Q(t)(\mu^\star - \nicefrac{\varepsilon}{2} - \mu_{J(t)})^+\indic{Q(t) \geq 1}} \\
&= 2T - \frac{\varepsilon}{2}\sum_{t\leq T} \expect{Q(t)} + \expect{-\frac{\varepsilon}{2}\sum_{t\leq T} Q(t)+2\sum_{t\leq T} Q(t)(\mu^\star - \nicefrac{\varepsilon}{2} - \mu_{J(t)})^+\indic{Q(t) \geq 1}} \label{eq:drift-bound-last}\\ 
&\leq 2T - \frac{\varepsilon}{2}\sum_{t\leq T} \expect{Q(t)} + \frac{4\expect{\sar^{\single}(\pi, T)^2}}{\varepsilon} \label{eq:drift-bound-sqr-sar}
\end{align}
where the last inequality is by the following sample-path upper bound of the last term in \eqref{eq:drift-bound-last}
\begin{align}
&\hspace{0.1in}-\frac{\varepsilon}{2}\sum_{t\leq T} Q(t)+2\sum_{t\leq T} Q(t)(\mu^\star - \nicefrac{\varepsilon}{2} - \mu_{J(t)})^+\indic{Q(t) \geq 1} \nonumber\\
&\leq -\frac{\varepsilon}{2}\sum_{t\leq T} Q(t)+2\left(\max_{t \leq T} Q(t)\right)\sum_{t\leq T} (\mu^\star - \nicefrac{\varepsilon}{2} - \mu_{J(t)})^+\indic{Q(t) \geq 1} \nonumber\\
&\leq -\frac{\varepsilon\left(\max_{t \leq T} Q(t)\right)^2}{4} +2\left(\max_{t \leq T} Q(t)\right)\sar^{\single}(\pi, T) \tag{Lemma~\ref{lem:single-connect-max-sum} and the definition of $\sar^{\single}(\pi, T)$} \\
&\leq \frac{4\sar^{\single}(\pi, T)^2}{\varepsilon} \tag{$\max_{x} ax^2+bx = \frac{b^2}{4a}$ for any $a > 0$ and let $x = \max_{t \leq T} Q(t)$}.
\end{align}
Reorganizing terms and dividing both sides in \eqref{eq:drift-bound-sqr-sar} by $\frac{\varepsilon T}{2}$ gives
$\frac{\sum_{t \leq T} \expect{Q(t)}}{T} \leq \frac{4}{\varepsilon} + \frac{8}{\varepsilon^2}\frac{\expect{\sar^{\single}(\pi, T)^2}}{T}$.
\end{proof}
\subsection{Satisficing regret of $\textsc{UCB}$}
The final piece of the proof bounds the first two moments of the satisficing regret for $\textsc{UCB}$. The proof is similar to that of the instance-dependent regret bound of $\textsc{UCB}$ in the classical bandit setting. The caveat in our case is that we need to obtain a bound that is independent of the minimum gap. To do so, we use the property of the satisficing regret that there is no loss when the algorithm chooses a server whose gap in service rate is less than $\frac{\varepsilon}{2}$. Define a good event $\set{G}_t$ for period~$t$ as the event that the DM does not select a server or that the service rate gap of the chosen server is no larger than two times the confidence interval, i.e., $\set{G}_t = \{J(t) = \emptyset\} \cup \left\{\mu^\star - \mu_{J(t)} \leq 2\sqrt{\frac{2\ln t}{C_{J(t)}(t)}}\right\}$. Following classical concentration bounds, the next lemma shows that the probability of $\set{G}_t$ is high.
\begin{lemma}\label{lem:ucb-conc}
For every period $t$ and under policy $\textsc{UCB}$, we have $\Pr\{\set{G}_t\} \geq 1 - 2Kt^{-3}$.
\end{lemma}
\begin{proof}
For each server $k$, by Hoeffding's Inequality and union bound over $C_k(t) = 0,\ldots,t-1$, we have $\Pr\{|\hat{\mu}_k(t) - \mu_k| > \sqrt{\frac{2\ln t}{C_k(t)}}\} \leq 2t(t)^{-4} = 2t^{-3}$. By union bound over all servers, we have with probability at least $1 - 2Kt^{-3}$, $|\hat{\mu}_k(t) - \mu_k| \leq \sqrt{\frac{2\ln t}{C_k(t)}}$ for every server $k$. Denote this event by~$\set{E}$. Under $\set{E}$, we have $ \mu_k \leq \bar{\mu}_k(t) \leq \mu_k + 2\sqrt{\frac{2\ln t}{C_k(t)}}$ since $\bar{\mu}_k(t) = \hat{\mu}_k(t) + \sqrt{\frac{2\ln t}{C_k(t)}}$. Then as long as $J(t)\neq \emptyset$, 
$\mu^\star - \mu_{J(t)} \leq \bar{\mu}_{k^\star}(t) - \bar{\mu}_{J(t)}(t) + 2\sqrt{\frac{2\ln t}{C_{J(t)}(t)}} \leq 2\sqrt{\frac{2\ln t}{C_{J(t)}(t)}}$ since the DM selects $J(t)$ that maximizes $\bar{\mu}_{k}(t)$ over $k \in \set{K}$. As a result, we have $(\set{G}_t)^c \subseteq \set{E}^c$ and $\Pr\{\set{G}_t\} \geq \Pr\{\set{E}\} \geq 1 - 2Kt^{-3}$.
\end{proof}

We now proceed to the proof of Lemma~\ref{lem:sar-ucb}.
\begin{proof}[Proof of Lemma~\ref{lem:sar-ucb}]
Fix a horizon $T$. Note that $\expect{\sum_{t=1}^T \indic{(\set{G}_t)^c}} \leq 2K\sum_{t=1}^T t^{-3} \leq 4K$ by Lemma~\ref{lem:ucb-conc} and the fact that $\sum_{t=1}^T t^{-3} \leq 2$. We first bound the expectation of satisficing regret by using the law of total probability:
\begin{align}
\expect{\sar^{\single}(\textsc{UCB}, T)} &\leq \expect{\sum_{t=1}^T (\mu^\star - \mu_{J(t)} - \frac{\varepsilon}{2})^{+}\indic{J(t) \neq \emptyset}\indic{\set{G}_t}} + \expect{\sum_{t=1}^T \indic{(\set{G}_t)^c}} \nonumber\\
&\leq \expect{\sum_{t=1}^T (\mu^\star - \mu_{J(t)} - \frac{\varepsilon}{2})^{+}\indic{J(t) \neq \emptyset}\indic{\set{G}_t}} + 4K. \label{eq:ucb-sar-condition}
\end{align}
For the first term in the right hand side of \eqref{eq:ucb-sar-condition}, its sample path value can always be upper bounded by separately considering servers whose service rate gap is at least $\frac{\varepsilon}{2}$:
\begin{align}
\sum_{t=1}^T \left(\mu^\star - \mu_{J(t)} - \frac{\varepsilon}{2}\right)^{+}\indic{J(t) \neq \emptyset}\indic{\set{G}_t} &\leq \sum_{t=1}^T \sum_{k \in \set{K}\colon \mu^\star - \mu_k \geq \frac{\varepsilon}{2}} \big(\mu^\star - \mu_k\big)\indic{J(t) = k}\indic{\set{G}_t} \nonumber\\
&\hspace{-1in}= \sum_{t=1}^T \sum_{k \in \set{K}\colon \mu^\star - \mu_k \geq \frac{\varepsilon}{2}} \big(\mu^\star - \mu_k\big)\indic{J(t) = k}\indic{\mu^\star - \mu_k \leq 2\sqrt{\frac{2\ln t}{C_k(t)}}} \nonumber\\
&\hspace{-1in}= \sum_{k \in \set{K}\colon \mu^\star - \mu_k \geq \frac{\varepsilon}{2}} \sum_{t = 1}^T \big(\mu^\star - \mu_k\big)\indic{J(t) = k}\indic{\mu^\star - \mu_k \leq 2\sqrt{\frac{2\ln t}{C_k(t)}}} \nonumber \\
&\hspace{-1in}\overset{(a)}{\leq} \sum_{k \in \set{K}\colon \mu^\star - \mu_k \geq \frac{\varepsilon}{2}} \left(\frac{8\big(\mu^\star - \mu_k\big)\ln T}{\big(\mu^\star - \mu_k\big)^2} + \mu^\star - \mu_k\right) \nonumber\\
&\hspace{-1in}\overset{(b)}{\leq} \frac{16K\ln T}{\varepsilon}+K \leq \frac{16K(\ln T+1)}{\varepsilon}. \label{eq:sar-ucb-sample-upp}
\end{align}
\noindent Here inequality (a) follows by considering the last time $t' \leq T$ such that $J(t') = k$ for a server~$k$: since $J(t)\neq k$ for $t\in(t',T]$ we know that $C_k(t')=C_k(T)-1$, and consequently $\mu^\star - \mu_k \leq 2\sqrt{\frac{2\ln t'}{C_k(t')}} \leq 2\sqrt{\frac{2\ln T}{\sum_{t=1}^T \indic{J(t)=k} - 1}}$  for this period and thus $\sum_{t=1}^T \indic{J(t)=k} \leq \frac{8\ln T}{(\mu^\star - \mu_k)^2} + 1$. Inequality (b) is because the summation only involves servers whose service rates satisfy $\mu^\star - \mu_k \geq \frac{\varepsilon}{2}$. With \eqref{eq:ucb-sar-condition} and \eqref{eq:sar-ucb-sample-upp}, we conclude that
$\expect{\sar^{\single}(\textsc{UCB}, T)} \leq \frac{16K(\ln T+1)}{\varepsilon} + 4K \leq \frac{16K(\ln T+2)}{\varepsilon}$.

For the second moment, similarly, we have
\begin{align}
\expect{\sar^{\single}(\textsc{UCB}, T)^2} &\leq \expect{\left(\sum_{t=1}^T \left(\left(\mu^\star - \mu_{J(t)} - \frac{\varepsilon}{2}\right)^{+}\indic{J(t) \neq \emptyset}\indic{\set{G}_t} + \indic{(\set{G}_t)^c}\right)\right)^2} \nonumber\\
&\leq 2\expect{\left(\sum_{t=1}^T \left(\mu^\star - \mu_{J(t)} - \frac{\varepsilon}{2}\right)^{+}\indic{J(t) \neq \emptyset}\indic{\set{G}_t}\right)^2} + 2\expect{\left(\sum_{t=1}^T \indic{(\set{G}_t)^c}\right)^2} \nonumber\\
&\overset{\eqref{eq:sar-ucb-sample-upp}}{\leq} 2\left(\frac{16K(\ln T+1)}{\varepsilon}\right)^2 + 2\expect{2\sum_{1 \leq t_1 \leq t_2 \leq T}\indic{(\set{G}_{t_1})^c}\indic{(\set{G}_{t_2})^c}} \nonumber\\
&\leq 2\left(\frac{16K(\ln T+1)}{\varepsilon}\right)^2 + 2\expect{2\sum_{t=1}^T t \indic{(\set{G}_t)^c}} \nonumber\\
&\overset{\text{Lemma~\ref{lem:ucb-conc}}}{\leq} \frac{2^9K^2(\ln T+1)^2}{\varepsilon^2}+8K\sum_{t=1}^T t^{-2} \nonumber\\
&\leq \frac{2^9K^2(\ln T+1)^2}{\varepsilon^2}+16K \leq \frac{2^9K^2(\ln T+2)^2}{\varepsilon^2}, \label{eq:sar-ucb-sample-uppsqr}
\end{align}
which completes the proof.
\end{proof}
\section{Optimal transient cost of learning for multi-queue systems}\label{sec:optimal-multi}
This section extends the study of transient cost of learning in queueing to multi-queue multi-server systems. Section~\ref{sec:multi-model} describes the model of such systems compared to single-queue multi-server systems. Section~\ref{sec:multi-res} presents the \textsc{MaxWeight-UCB} ($\textsc{MW-UCB}$) algorithm which we show to have low transient cost of learning for a multi-queue multi-server system.

\subsection{A model of multi-queue multi-server systems}\label{sec:multi-model}
A multi-queue multi-server system extends the single-queue multi-server system in Section~\ref{sec:model-single} by having multiple queues whose servers are coupled by a set of feasible schedules. Specifically, a multi-queue multi-server system is defined by a tuple $(\set{N},\set{K},\bolds{\Lambda},\bolds{\mu}, \bolds{\set{A}}, \bolds{\Sigma}, \bolds{\set{B}})\
$, where $\bolds{\set{B}}= \{\set{B}_n\}_{n \in \set{N}}$. There is a set of queues $\set{N}$ with cardinality $N$. Each queue $n\in\set{N}$ has a set of servers $\set{B}_n$.  A server $k\in \set{B}_n$ belongs to a single queue $n$, and has service rate $\mu_k$. 

In each period $t$, the DM selects a set of servers to serve jobs. Though we model a system with each server being dedicated to a specific queue, the set of feasible schedules allows us to encode fairly general interference between servers. In particular, the selected set of servers must be from the model-specific set of feasible schedules $\bolds{\Sigma} \subseteq \{0,1\}^{\set{K}}$ (see Remark \ref{rem:model_generality}). We require that for any queue, the number of selected servers is no larger than the number of jobs in this queue.\footnote{Though this reflects the feature from the single-queue setting, that $Q(t)=0\implies J(t)=\perp$, it maintains the flexibility to have a queue that has multiple jobs served in a single period. This requirement is also adopted in \cite{tassiulas1992stability}; see Appendix~\ref{app:examples-network}.} 
Formally, letting $\sigma_k = 1$ if schedule $\bolds{\sigma} \in \bolds{\Sigma}$ selects server $k$ and denoting $\bolds{Q}(t) = (Q_n(t))_{n \in \set{N}}$ as the queue length vector at the beginning of period $t$, the set of feasible schedules in this period is 
\begin{equation}\label{eq:feasible-schedule-set}
\bolds{\Sigma}_t = \{\bolds{\sigma} \in \bolds{\Sigma} \colon \sum_{k \in \set{B}_n} \sigma_k \leq Q_n(t), \forall n \in \set{N}\}
\end{equation}
and the DM selects  a schedule $\bolds{\sigma}(t) \in \bolds{\Sigma}_t$. Following \cite{tassiulas1992stability}, we assume that any subset of a feasible schedule is still feasible, i.e., if $\bolds{\sigma}\in\bolds{\Sigma}$ and $\sigma'_k\leq \sigma_k~\forall k\in\set{K}$, then $\bolds{\sigma}'\in\bolds{\Sigma}$.

We now formalize the arrival and service dynamics in every period, which are captured by the independent random variables $\{\bolds{A}(t),\{S_k(t)\}_{k \in \set{K}}\}_t$. The arrival vector $\bolds{A}(t)=\{A_n(t)\}_{n \in \set{N}}$ consists of (possibly correlated) random variables $A_n(t)$ taking value in $\bolds{\set{A}} \subseteq \{0,1\}^{\set{N}}$; we denote its distribution by $\bolds{\Lambda}$ and let $\expect{A_n(t)} = \lambda_n(\bolds{\Lambda})$ with $\bolds{\lambda} = (\lambda_n(\bolds{\Lambda}))_{n \in \set{N}}$.\footnote{This formulation captures settings where arrivals are independent of the history, such as the example of each queue having independent arrivals (e.g., the bipartite queueing model in \cite{FreundLW22}) and the example of feature-based queues \cite{singh2022feature} where jobs have features; each type of feature has one queue; at most one job arrives among all queues in each period. Our formulation cannot capture state-dependent arrivals such as queues with balking \cite{hassin2003queue}. }  The service $S_k(t)$ for each server $k\in\set{K}$ is a Bernoulli random variable indicating whether the selected service request was successful.  The queueing dynamic is given by
\begin{equation}\label{eq:dynamic-multi}
Q_n(t+1) = Q_n(t) - \sum_{k \in \set{B}_n} \sigma_k(t)S_k(t) + A_n(t). 
\end{equation}
We assume that the DM has knowledge of which policies are allowed, i.e., they know $\bolds{\Sigma}$ and $\bolds{\set{B}}$, but has no prior knowledge of the rates $\bolds{\lambda},\bolds{\mu}$. In period $t$, the observed history is the set $\left(\{A_n(\tau)\}_{n \in \set{N}}, \{S_k(\tau)\}_{n \in \set{N},k \in \set{K} \colon \sigma_k(\tau) = 1}\right)_{\tau < t}$ that includes information on realized arrivals and services in previous periods. Similar to before, a non-anticipatory policy $\pi$ maps an observed history to a feasible schedule; we let $Q_n(t,\pi)$ be the length of queue $n$ in period $t$ under this policy.  

Unlike the single-queue case, it is usually hard to find the optimal policy for a multi-queue multi-server system even with known system parameters. Fortunately, if the system is stabilizable, i.e., 
$\lim_{T \to \infty} \frac{1}{T}\sum_{t \leq T} \expect{\|\bolds{Q}(t)\|_1} < \infty$ under some scheduling policy,
then the arrival rate vector must be within the \emph{capacity region} of the servers \cite{tassiulas1992stability}. Formally, let $\bolds{\Phi} = \{\bolds{\phi} \in [0,1]^{\bolds{\Sigma}} \colon \sum_{\bolds{\sigma} \in \bolds{\Sigma}} \bolds{\phi}_{\bolds{\sigma}} = 1\}$ be the probability simplex over $\bolds{\Sigma}$. A distribution $\bolds{\phi}$ in $\bolds{\Phi}$ can be viewed as the frequency of a policy using each schedule $\bolds{\sigma}\in\bolds{\Sigma}$, and the effective service rate queue $n$ can get is given by $\mu_n^{\multi}(\bolds{\phi}) = \sum_{\bolds{\sigma} \in \bolds{\Sigma}} \bolds{\phi}_{\bolds{\sigma}}\sum_{k \in \set{B}_n} \sigma_k\mu_k$. Denoting the effective service rate vector for a schedule distribution $\bolds{\phi}$ by $\bolds{\mu}^{\multi}(\bolds{\phi})$, the capacity region is $\set{S}(\bolds{\mu},\bolds{\Sigma},\bolds{\set{B}}) = \{\bolds{\mu}^{\multi}(\bolds{\phi}) \colon \bolds{\phi} \in \bolds{\Phi}\}$. For a multi-queue multi-server system to be stabilizable, we must have $\bolds{\lambda}(\bolds{\Lambda}) \in \set{S}(\bolds{\mu},\bolds{\Sigma},\bolds{\set{B}})$ \cite{tassiulas1992stability}. As in the single-queue case, we also assume that the system has a positive traffic slackness and let $\bm{1}$ denote a vector of $1$s with suitable dimension. To define the traffic slackness, we introduce the set $\mathscr{E} = \{\varepsilon \in (0,1] : \bolds{\lambda}(\bolds{\Lambda}) + \varepsilon \bm{1} \in \set{S}(\bolds{\mu},\bolds{\Sigma},\bolds{\set{B}})\}$ to capture by how much the arrival rates may increase.
\begin{definition}\label{def:multi-slackness}
The traffic slackness of a multi-queue multi-server system is $\varepsilon = \max \{\varepsilon' \in \mathscr{E}\}.$
\end{definition}

\begin{remark}\label{rem:model_generality}
   The set of feasible schedules $\bolds{\Sigma}$ is a classical modeling element in the queueing literature \cite{tassiulas1992stability}. By setting a suitable $\Sigma$, our model captures existing settings in online learning for multi-queue multi-server systems \cite{FreundLW22, YangSrikantYing}. In particular, in \cite{YangSrikantYing}, there are $N$ queues and $M$ workers; in each period each worker takes on at most one job from any of the queues. In contrast, we model each server as dedicated to a queue. Our model can capture their setting by creating $N$ replicas of each worker (one for each queue) and letting the set of feasible schedules $\bolds{\Sigma}$ disallow two replicas of the same worker providing service at the same time; we thus capture the constraints in \cite{YangSrikantYing}. Beyond these settings, the set $\bolds{\Sigma}$ allows our model to capture more general multi-queue multi-server queueing systems as discussed in Appendix~\ref{app:multi-example}.
\end{remark}
We extend the definition of $\colq$ in \eqref{def:tclq-single} to the multi-queue multi-server setting. Since the optimal policy is difficult to design, we instead define $\colq$ for a policy $\pi$ by comparing it with any non-anticipatory policy (which makes decisions only based on the history):
\begin{equation}\label{eq:def-CLQ-multi}
\colq^{\multi}(\bolds{\Lambda},\bolds{\mu}, \bolds{\Sigma}, \bolds{\set{B}}, \pi) = \max_{\text{non-anticipatory }\pi'}\max_{T \geq 1} \frac{\sum_{t=1}^T\sum_{n \in \set{N}}\expect{Q_n(t,\pi) - Q_n(t,\pi')}}{T}.
\end{equation}
As in the single-queue setting, we can define the worst-case transient cost of learning for a fixed structure $\bolds{\set{A}},\bolds{\Sigma}, \bolds{\set{B}}$ and a traffic slackness $\varepsilon$ as the supremum across any arrival and service rates with this traffic slackness. With the same slight abuse of notation as before, we denote these quantities by $\colq^{\multi}(\bolds{\set{A}}, \bolds{\Sigma}, \bolds{\set{B}},\varepsilon, \pi)$. The goal is to find a non-anticipatory algorithm that has low $\colq^{\multi}$.

\subsection{Algorithm and main results}\label{sec:multi-res}

Our starting point is a well-known heuristic called the  \textsc{MaxWeight} policy (\textsc{MW} in short) \cite{tassiulas1992stability}. In each period $t$, \textsc{MW} selects a feasible schedule $\bolds{\sigma}^{\MW}(t) \in \bolds{\Sigma}_t$ that maximizes the sum of the products of queue lengths and aggregated service rates, i.e.,  
\begin{equation}\label{eq:maxweight-schedule}
\bolds{\sigma}^{\MW}(t) \in \arg\max_{\bolds{\sigma} \in \bolds{\Sigma}_t} \sum_{n \in \set{N}} Q_n(t)\sum_{k \in \set{B}_n} \sigma_k \mu_k.
\end{equation}
\textsc{MW} is known to have a queue length of order $O(\frac{K}{\varepsilon})$ for a multi-queue multi-server system with $K$ servers and a traffic slackness of $\varepsilon$ (see e.g., \cite[lemma 4.1]{georgiadis2006resource}). However, \textsc{MW} is only implementable with full knowledge of service rates. Our policy, $\textsc{MW-UCB}$ (Algorithm~\ref{algo:mw-ucb}), applies it to the learning setting by augmenting it with upper confidence bound estimations. Instead of using the ground-truth service rate $\bolds{\mu}$, $\textsc{MW-UCB}$ uses the upper confidence bound estimator $\bar{\bolds{\mu}}$ to select a schedule in each period, i.e.,
\begin{equation}\label{eq:mx-ucb-rule}
\bolds{\sigma}(t) \in \arg\max_{\bolds{\sigma} \in \bolds{\Sigma}_t} \sum_{n \in \set{N}} Q_n(t)\sum_{k \in \set{B}_n} \sigma_k \bar{\mu}_k(t).
\end{equation}
As in \textsc{UCB} for single-queue systems, we set $\bar{\mu}_k(t) = \min\left(1,\hat{\mu}_k(t)+\sqrt{\frac{2\ln(t)}{C_k(t)}}\right)$ where $\hat{\mu}_k(t)$ is the sample mean and $C_k(t)$ is the number of periods in which server $k$ is selected. The algorithm then selects $\bolds{\sigma}(t)$ just like \textsc{MW} but replacing $\mu_k$ by $\bar{\mu}_k(t)$. In the single queue case, \textsc{MW-UCB} selects the server with the highest estimator and is equivalent to \textsc{UCB}. 

\begin{remark}\label{remark:app-mw}
Although Algorithm~\ref{algo:mw-ucb} assumes an oracle that solves the maximization problem \eqref{eq:mx-ucb-rule} perfectly, in Appendix~\ref{app:ext-app} we show how our result extends to an oracle that is only approximately correct. The proof follows the same steps as in this section with a suitably adjusted definition of traffic slackness, which highlights the modularity and generality of our proof techniques.
\end{remark}

\begin{algorithm}[H]
\LinesNumbered
\DontPrintSemicolon
\caption{\textsc{MW-UCB} policy for a multi-queue multi-server system 
\label{algo:mw-ucb}
}
\SetKwInOut{Input}{input}\SetKwInOut{Output}{output}
\Input{set of arrivals $\bolds{\set{A}}$, set of schedules $\bolds{\Sigma}$, set of servers belonging to queues $\{\set{B}_n\}_{n \in \set{N}}$}
Sample mean $\hat{\mu}_{k}(1) \gets 0$, number of samples $C_k(1) \gets 0$ for $k \in \set{K}$\;
\For{$t = 1\ldots$}{
    $\bar{\mu}_{k}(t) = 
    \min\left(1, \hat{\mu}_{k}(t) + \sqrt{\frac{2\ln(t)}{C_k(t)}}\right), \forall k \in \set{K}$\label{line:mw-ucb-esti}\;
    \tcc{select a max-weight feasible schedule based on UCB estimation}    $\bolds{\Sigma}_t = \{\bolds{\sigma} \in \bolds{\Sigma} \colon \sum_{k \in \set{B}_n} \sigma_k \leq Q_n(t), \forall n \in \set{N}\}$ \\
    $\bolds{\sigma}(t) \in \arg\max_{\bolds{\sigma} \in \bolds{\Sigma}_t} \sum_{n\in \set{N}}Q_n(t)\sum_{k\in \set{B}_n}\sigma_k \bar{\mu}_k(t)$ \label{algoline:mw-ucb-rule}\\
   \tcc{Update queue lengths \& estimates based on $\{A_n(t)\}_{n \in \set{N}}, \{\sigma_k(t)S_k(t)\}_{k \in \set{K}}$}
   $Q_n(t+1) \gets Q_n(t) - \sum_{k \in \set{B}_n} \sigma_k(t)S_k(t)+A_n(t), \forall n\in \set{N}$ \\
    $C_{k}(t+1) \gets C_{k}(t)+1, \hat{\mu}_{k}(t+1) \gets \frac{C_{k}(t)\hat{\mu}_{k}(t)+S_{k}(t)}{C_{k}(t+1)}, \forall k \in \set{K}: \sigma_k(t)=1$
}
\end{algorithm}

Recall that the system structure is given by  $\bolds{\set{A}},\bolds{\Sigma},\bolds{\set{B}}$ corresponding to the set of possible arrivals, feasible schedules and servers belonging to each queue.  Let $\Marr = \max_{\bolds{A} \in \bolds{\set{A}}} \sum_{n \in \set{N}} A_n$ denote the maximum possible number of arrivals per period and $\Mmulti = \max_{\bolds{\sigma} \in \bolds{\Sigma}} \sum_{n \in \set{N}}\sum_{k \in \set{B}_n} \sigma_k$ denote the maximum number of jobs served per period in a feasible schedule. We establish a $\colq$ bound for \textsc{MW-UCB} with a near optimal dependence $\tilde{\bigO}(1/\varepsilon)$ on the traffic slackness $\varepsilon$.
\begin{theorem}\label{thm:mw-ucb}
For any $\bolds{\set{A}},\bolds{\Sigma},\bolds{\set{B}}$ and traffic slackness $\varepsilon \in (0,1]$, we have \[\colq^{\multi}(\bolds{\set{A}},\bolds{\Sigma},\bolds{\set{B}},\varepsilon,\textsc{MW-UCB}) \leq \frac{\sqrt{N}\left(16\Marr+2^{10}K\Mmulti^2(1+\ln(\Marr K\Mmulti/\varepsilon))\right)}{\varepsilon}.\]
\end{theorem}
To show Theorem~\ref{thm:mw-ucb}, we fix the arrival rate vector $\bolds{\Lambda}$ and service rate vector $\bolds{\mu}$ and ease notation by sometimes writing $\bolds{\lambda} = \bolds{\lambda}(\bolds{\Lambda})$. We assume the traffic slackness condition (Definition \ref{def:multi-slackness}) holds true, i.e.,  $\bolds{\lambda} + \varepsilon\bm{1} \in \set{S}(\bolds{\mu},\bolds{\Sigma},\bolds{\set{B}})$, and prove the bound in Theorem~\ref{thm:mw-ucb} for $\colq^{\multi}(\bolds{\Lambda},\bolds{\mu},\bolds{\Sigma},\bolds{\set{B}},\textsc{MW-UCB})$. The proof follows the same strategy as the one of Theorem~\ref{thm:colq-ucb-single}. We again fix a policy $\pi$ and aim to establish
a connection between $\colq^{\multi}(\bolds{\Lambda},\bolds{\mu},\bolds{\Sigma},\bolds{\set{B}},\pi)$ and an appropriate notion of satisficing regret in both the learning stage and the regenerate stage. We first define the satisficing regret for the multi-queue multi-server system. To do so, denote the weight of a schedule $\bolds{\sigma}$ in period $t$ by 
\begin{equation}\label{eq:weight-schedule}
W_{\bolds{\sigma}}(t) = \sum_{n \in \set{N}} Q_n(t)\sum_{k \in \set{B}_n} \sigma_k \mu_k.
\end{equation}
Recall that \textsc{MW} chooses a feasible weight-maximizing schedule, i.e., $\bolds{\sigma}^{\MW}(t) \in \arg \max_{\bolds{\sigma} \in \bolds{\Sigma}_t} W_{\bolds{\sigma}}(t)$. We denote by $\Delta(t)$ the loss of schedule $\bolds{\sigma}(t)$ in period $t$, which is defined as the weight difference between the chosen schedule and the \textsc{MW} schedule, normalized by the maximum queue length: 
\begin{equation}\label{eq:def-multi-loss}
\Delta(t) = \left\{
\begin{aligned}
&0,~\text{if }\|\bolds{Q}(t)\|_{\infty} = 0\\
&\frac{W_{\bolds{\sigma}^{\MW}(t)}(t) - W_{\bolds{\sigma}(t)}(t)}{\|\bolds{Q}(t)\|_{\infty}},~\text{otherwise.}
\end{aligned}
\right.
\end{equation}
We define the loss $\Delta(t)$ in a way that enables it to be a lower bound on the estimation error in period $t$. If all queues had equal lengths, then the right-hand side would be equal to the estimation error $\sum_k \mu_k(\sigma^{\MW}_k(t)-\sigma_k(t))$. To ensure that $\Delta(t)$ is below this quantity,  we divide by $\|\bolds{Q}(t)\|_{\infty}$ in \eqref{eq:def-multi-loss} to account for larger queues being weighted more heavily in the weight of a schedule (note that we could replace the infinity norm in this definition by any $p$-norm with $p \geq 2$ without affecting the analysis). We define the satisficing regret that a policy $\pi$ incurs over a horizon $T$ as
\begin{equation}
\sar^{\multi}(\pi, T) = \sum_{t=1}^T \left(\Delta(t) - \frac{\varepsilon}{2}\right)^+.
\end{equation}
Note that our definition of satisficing regret naturally extends that in the single-queue case: if there is only one queue, $\sar^{\multi}(\pi, T) = \sum_{t=1}^T (\mu^\star-\sum_{k \in \set{K}} \sigma_k(t)\mu_{k}-\frac{\varepsilon}{2})^+\indic{Q(t) \geq 1} = \sar^{\single}(\pi, T)$. 

We first bound the time-averaged $2-$norm of queue lengths via the expected satisficing regret.
\begin{lemma}\label{lem:multi-learning}
For any policy $\pi$ and horizon $T$, $\frac{\sum_{t=1}^T \expect{\|\bolds{Q}(t)\|_2}}{T} \leq \frac{16\Marr+20\Mmulti^2}{\varepsilon} + \expect{\sar^{\multi}(\pi, T)}$.
\end{lemma}
To obtain a bound on the time-averaged $1-$norm, we need to multiply the bound on the $2-$norm by a factor of $\sqrt{N}$ which leads to the $\sqrt{N}$ term in Theorem~\ref{thm:mw-ucb}. Our proof technique relies on a Lyapunov analysis based on the $2-$norm and it is unclear whether we can directly bound the $1-$norm without the additional $\sqrt{N}$ factor. 

Similar to Lemma~\ref{lem:single-queue-learn}, Lemma~\ref{lem:multi-learning} is only useful for the initial learning stage. When $T$ gets large, $\expect{\sar^{\multi}(\pi, T)}$ grows with it, and the bound weakens. Then, in the regenerate stage, we use the following bound, based on the second moment of the satisficing regret.
\begin{lemma}\label{lem:multi-regenerate}
For any policy $\pi$ and horizon $T$, $\frac{\sum_{t=1}^T \expect{\|\bolds{Q}(t)\|_1}}{T} \leq \frac{2\Marr+6\Mmulti^2}{\varepsilon} + \frac{16\Marr}{\varepsilon^2} \frac{\expect{\sar^{\multi}(\pi, T)^2}}{T}$.
\end{lemma}
Lastly, we require bounds for the first and second moments of satisficing regret for $\textsc{MW-UCB}$. The proof (provided in Appendix~~\ref{app:lem-sar-mw-ucb}) builds on techniques from combinatorial bandits \cite{chen2013combinatorial}.
\begin{lemma}\label{lem:sar-mw-ucb}
For any horizon $T$,  
\begin{align*}
\expect{\sar^{\multi}(\textsc{MW-UCB},T)} &\leq \frac{32K\Mmulti^2(\ln T+1)}{\varepsilon}  \\
\expect{\sar^{\multi}(\textsc{MW-UCB},T)^2} &\leq \frac{2^{11}K^2\Mmulti^4(\ln T+1)^2}{\varepsilon^2}.
\end{align*}
\end{lemma}

We combine these lemmas to prove Theorem~\ref{thm:mw-ucb}, similar to our above proof of Theorem~\ref{thm:colq-ucb-single}.
\begin{proof}[Proof of Theorem~\ref{thm:mw-ucb}]
Fix $\bolds{\set{A}},\bolds{\Sigma},\bolds{\set{B}}$, the traffic slackness $\varepsilon$ and a pair of $\bolds{\Lambda},\bolds{\mu}$ such that $\bolds{\lambda}(\bolds{\Lambda}) + \varepsilon \bm{1} \in \set{S}^{\multi}(\bolds{\mu},\bolds{\Sigma},\bolds{\set{B}})$. We first show the upper bound for the $\colq$ of $\textsc{MW-UCB}$ for this particular multi-queue multi-server system. The main result then follows since the worst case $\colq$ is the maximum of the cost of learning over all these systems.

Let $T_1 = \left\lfloor \left(\frac{2^{15} \Marr K^2 \Mmulti^4}{\varepsilon^4}\right)^2\right\rfloor$. For a horizon $T \leq T_1$,
Lemma~\ref{lem:multi-learning} shows that
\begin{align*}
\frac{\sum_{t=1}^T \expect{\|\bolds{Q}(t)\|_2}}{T} &\leq \frac{16\Marr+20\Mmulti^2}{\varepsilon} + \expect{\sar^{\multi}(\textsc{MW-UCB},T)} \tag{Lemma~\ref{lem:multi-learning}} \\
&\leq \frac{16\Marr+20\Mmulti^2}{\varepsilon} + \frac{32K\Mmulti^2(\ln T+1)}{\varepsilon} \tag{Lemma~\ref{lem:sar-mw-ucb}} \\
&\leq \frac{16\Marr+52K\Mmulti^2}{\varepsilon} + \frac{64K\Mmulti^2\ln\left(2^{15} \Marr K^2 \Mmulti^4/\varepsilon^4\right)}{\varepsilon} \tag{$T \leq T_1$}\\
&\leq \frac{16\Marr+52K\Mmulti^2 + 64K\Mmulti^2\left(11+4\ln(\Marr K\Mmulti/\varepsilon)\right)}{\varepsilon} \\
&\leq \frac{16\Marr+2^{10}K\Mmulti^2(1+\ln(\Marr K\Mmulti/\varepsilon))}{\varepsilon}.
\end{align*}
The fact that $\|\bolds{Q}(t)\|_1 \leq \sqrt{N} \|\bolds{Q}(t)\|_2$ then gives for $T \leq T_1$,
\[
\frac{\sum_{t=1}^T \expect{\|\bolds{Q}(t)\|_1}}{T} \leq \frac{\sqrt{N}\left(16\Marr+2^{10}K\Mmulti^2(1+\ln(\Marr K\Mmulti/\varepsilon))\right)}{\varepsilon}.
\]
For $T > T_1$,  
\begin{align*}
\frac{\sum_{t=1}^T \expect{\|\bolds{Q}(t)\|_1}}{T} &\leq \frac{2\Marr +6K\Mmulti^2}{\varepsilon} + \frac{16\Marr\expect{\sar^{\multi}(\textsc{MW-UCB},T)^2}}{\varepsilon^2 T} \tag{Lemma~\ref{lem:multi-regenerate}}\\
&\leq \frac{2\Marr +6K\Mmulti^2}{\varepsilon} + \frac{16\Marr \left(2^{11}K^2 \Mmulti^4(\ln T+1)^2\right)}{\varepsilon^4 T}\tag{Lemma~\ref{lem:sar-mw-ucb}} \\
&\leq \frac{2\Marr +6K\Mmulti^2}{\varepsilon} + \frac{2^{15}\Marr K^2 \Mmulti^4}{\varepsilon^4 \sqrt{T}}  \tag{Fact~\ref{fact:lnt-sqrt-prop} and $T > T_1 \geq 50000$}\\
&\leq \frac{2\Marr +6K\Mmulti^2}{\varepsilon} + 1 \tag{$T > T_1$ implies $T \geq \left(\frac{2^{15}\Marr  K^2 \Mmulti^4}{\varepsilon^4}\right)^2$}.
\end{align*}
Merging the two cases of $T \leq T_1$ and $T > T_1$ proves for any $\bolds{\Lambda},\bolds{\mu}$ with $\bolds{\lambda}(\bolds{\Lambda})+\varepsilon \bm{1} \in \set{S}(\bolds{\mu},\bolds{\Sigma},\bolds{\set{B}})$ that
\begin{align*}
\colq^{\multi}(\bolds{\Lambda},\bolds{\mu},\bolds{\set{A}},\bolds{\Sigma},\bolds{\set{B}},\textsc{MW-UCB}) &\leq \max_{T} \frac{\sum_{t=1}^T \expect{\|\bolds{Q}(t)\|_1}}{T} \\
&\leq \frac{\sqrt{N}\left(16\Marr+2^{10}K\Mmulti^2(1+\ln(\Marr K\Mmulti/\varepsilon))\right)}{\varepsilon},
\end{align*}
which provides a worst-case bound on the $\colq$ in a multi-queue multi-server system. 
\end{proof}

\subsection{Queue length bound in the learning stage (Lemma~\ref{lem:multi-learning})}
For the single-queue case, we show the queue length bound in the learning stage (Lemma~\ref{lem:single-queue-learn}) by coupling the queueing process with a nearly optimal queueing process and bounding their difference by the difference in total services. There are two challenges in extending this proof technique to the multi-queue case. First, the process is now multi-dimensional (due to multiple queues). Second, even if we can couple the process with another one (e.g., the process under \textsc{MW}), it is unclear how to bound the difference in queue lengths by the difference in services. To address the first challenge, we first select a Lyapunov function that translates the multi-dimensional process into a single-dimensional one and still behaves like a single-queue process.  Motivated by \cite{ShahTZ10}, we consider the Lyapunov function $\varphi(t) = \|\bolds{Q}(t)\|_2$. For a period $t$, define $\tilde{\bolds{\sigma}}(t)$ as a schedule with the largest weight, i.e., $\tilde{\bolds{\sigma}}(t) \in \arg\max_{\bolds{\sigma} \in \bolds{\Sigma}} W_{\bolds{\sigma}}(t)$. The traffic slackness, $\bolds{\lambda}+\varepsilon\bm{1} \in \set{S}(\bolds{\mu},\bolds{\Sigma},\bolds{\set{B}})$,
ensures a negative drift in $\varphi(t)$ when it is large if the policy always chooses $\tilde{\bolds{\sigma}}(t)$\cite{ShahTZ10}. However, given that $\bolds{\Sigma}_t$ may be a strict subset of $\bolds{\Sigma}$, the schedule $\tilde{\bolds{\sigma}}(t)$ may be infeasible in period $t$. Instead, we first show that the schedule under $\textsc{MW-UCB}$ has a similar property by upper bounding the weight difference between $\tilde{\bolds{\sigma}}(t)$ and the $\textsc{MW}$ schedule $\bolds{\sigma}^{\MW}(t)$ and that between $\bolds{\sigma}^{\MW}(t)$ and the selected schedule $\bolds{\sigma}(t)$.
\begin{lemma}\label{lem:multi-bound-weight}
For a period $t$, $\sum_{n \in \set{N}} \lambda_n Q_n(t) - W_{\bolds{\sigma}(t)}(t) \leq -\varepsilon\sum_{n \in \set{N}} Q_n(t) + \Mmulti^2 + \Delta(t)\|\bolds{Q}(t)\|_{\infty}$.
\end{lemma}
\begin{proof}
The traffic slackness, $\bolds{\lambda} + \varepsilon \bm{1} \in \set{S}(\bolds{\mu},\bolds{\Sigma},\bolds{\set{B}})$, guarantees that there exists a distribution~$\phi$ over all schedules such that for all $n \in \set{N}$, $\lambda_n + \varepsilon \leq \sum_{\bolds{\sigma} \in \bolds{\Sigma}} \phi_{\bolds{\sigma}}\sum_{k \in \set{B}_n} \sigma_k \mu_k$ and thus
\[
\sum_{n \in \set{N}} (\lambda_n + \varepsilon) Q_n(t) \leq \sum_{\bolds{\sigma} \in \bolds{\Sigma}} \phi_{\bolds{\sigma}}\sum_{n \in \set{N}}Q_n(t)\sum_{k \in \set{B}_n} \sigma_k \mu_k = \sum_{\bolds{\sigma} \in \bolds{\Sigma}} \phi_{\bolds{\sigma}}W_{\bolds{\sigma}}(t) \leq \max_{\bolds{\sigma} \in \bolds{\Sigma}} W_{\bolds{\sigma}}(t) = W_{\tilde{\bolds{\sigma}}(t)}(t).
\]

For any $\tilde{\bolds{\sigma}}(t)$, we construct a feasible schedule $\bolds{\alpha}$ as follows:
if, for a queue $n$, we have $Q_n(t) \geq \sum_{k \in \set{B}_n} \tilde{\sigma}_k(t)$, then we set $\bolds{\alpha}_k = \tilde{\sigma}_k(t)$ for $k \in \set{B}_n$; otherwise, we set $\bolds{\alpha}_k=0$ for all $k \in \set{B}_n$. We have $\bolds{\alpha} \in \bolds{\Sigma}_t$ and thus $W_{\bolds{\alpha}}(t) \leq W_{\bolds{\sigma}^{\MW}(t)}(t)$. Additionally, by its construction,
\begin{align}
W_{\bolds{\sigma}^{\MW}(t)}(t) \geq W_{\bolds{\alpha}}(t) = \sum_{n \in \set{N}}Q_n(t)\sum_{k \in \set{B}_n}\alpha_k\mu_k &\geq \sum_{n \in \set{N}}Q_n(t)\sum_{k \in \set{B}_n}\tilde{\sigma}_k(t)\mu_k - \sum_{n \in \set{N}}\left(\sum_{k \in \set{B}_n}\tilde{\sigma}_k(t)\mu_k\right)^2 \nonumber\\ 
&\geq  W_{\tilde{\bolds{\sigma}}(t)}(t) - \Mmulti^2. \label{eq:mw-feasible-bound}
\end{align}
where the latter inequality holds as $\tilde{\bolds{\sigma}}(t)\in \Sigma$ and $\Mmulti = \max_{\bolds{\sigma} \in \bolds{\Sigma}} \sum_{n \in \set{N}}\sum_{k \in \set{B}_n} \sigma_k$. As a result,
\[
W_{\bolds{\sigma}(t)}(t) + \Delta(t)\|\bolds{Q}(t)\|_{\infty} = W_{\bolds{\sigma}^{\MW}(t)}(t) \geq W_{\tilde{\bolds{\sigma}}(t)}(t) - \Mmulti^2 \geq \sum_{n \in \set{N}} (\lambda_n + \varepsilon) Q_n(t) - \Mmulti^2,
\]
where the initial equality is by the definition of $\Delta(t)$ in \eqref{eq:def-multi-loss}.
Moving $W_{\bolds{\sigma}(t)}(t)$ from the left hand side to the right hand side and $\varepsilon\sum_{n \in \set{N}} Q_n(t)-\Mmulti^2$ to the left hand side gives the desired result. 
\end{proof}

Let $\set{F}_t$ be the $\sigma-$field generated by the sample path before period $t$ $\left(\{A_n(\tau)\}_{n \in \set{N}}, \{S_k(\tau)\}_{k \in \set{K}}\right)_{\tau < t}$. Clearly, $\{\varphi(t)\}$, $\{\Delta(t)\}$ and $\{\bolds{\sigma}(t)\}$ are $\{\set{F}_t\}$-adapted processes. The next lemma analyzes the drift, $\varphi(t+1) - \varphi(t)$, conditioned on the filtration $\set{F}_t$. In particular, the drift is negative when the schedule weight loss $\Delta(t)$ is small. The proof is similar to that of \cite[Theorem 4.4]{ShahTZ10} but our result is for a general multi-queue multi-server system with learning.
\begin{lemma}\label{lem:multi-drift-norm2}
For the process $\{\varphi(t)\}$, we have that for every period $t$,
\begin{enumerate}
\item Bounded difference: $|\varphi(t+1) - \varphi(t)| \leq \sqrt{\Marr + \Mmulti^2}$
\item Drift bound:  
$
\expect{\varphi(t+1) - \varphi(t) \mid \set{F}_t} \leq -\frac{\varepsilon}{4} + \left(\Delta(t) - \frac{\varepsilon}{2}\right)^+
$ if $\varphi(t) \geq \frac{4(\Marr+2\Mmulti^2)}{\varepsilon}$.
\end{enumerate}
\end{lemma}
\begin{proof}
For the bounded difference, triangle inequality gives $|\varphi(t+1) - \varphi(t)| \leq \|\bolds{Q}(t+1)-\bolds{Q}(t)\|_2$ and, given that $\bolds{\sigma}(t) \in \bolds{\Sigma}_t$ is a feasible schedule (recall Eq.~\ref{eq:feasible-schedule-set}),  we have
\begin{equation}\label{eq:drift-norm2-prob1bound}
\begin{aligned}
\|\bolds{Q}(t+1)-\bolds{Q}(t)\|_2 &= \sqrt{\sum_{n\in \set{N}} \left(A_n(t)-\sum_{k \in \set{B}_n}\sigma_k(t)S_k(t)\right)^2} \\
&\leq \sqrt{\sum_{n \in \set{N}} A_n(t) + \sum_{n\in \set{N}} \left(\sum_{k \in \set{B}_n} \sigma_k(t)S_k(t)\right)^2} \leq \sqrt{\Marr + \Mmulti^2}
\end{aligned}
\end{equation}
where the first inequality follows from the triangle inequality and the fact that $A_n(t), S_k(t) \in \{0,1\}$. The last inequality follows from the definitions  $\Marr = \max_{\bolds{A} \in \bolds{\set{A}}} \sum_{n \in \set{N}} A_n$, and  $\Mmulti = \max_{\bolds{\sigma} \in \bolds{\Sigma}} \sum_{n \in \set{N}}\sum_{k \in \set{B}_n} \sigma_k$ and observing that $A_n(t), S_k(t) \in \{0,1\}$.

To show the drift bound, consider the case where $\varphi(t) = \|\bolds{Q}(t)\|_2 > 0$. Condition on the filtration $\set{F}_t$ so that $\varphi(t) \geq \frac{4(\Marr+2\Mmulti^2)}{\varepsilon}$. Applying Fact~\ref{fact:diff-norm2} in the appendix and recalling that $\varphi(t) = \|\bolds{Q}(t)\|_2$ gives
\begin{align}
\expect{\varphi(t+1)-\varphi(t) \mid \set{F}_t} &\overset{\text{Fact}~\ref{fact:diff-norm2}}{\leq} \expect{\frac{\bolds{Q}(t) \cdot (\bolds{Q}(t+1) - \bolds{Q}(t)) +  \|\bolds{Q}(t+1)-\bolds{Q}(t)\|_2^2}{\|\bolds{Q}(t)\|_2} \mid \set{F}_t} \nonumber\\
&= \expect{\frac{\bolds{Q}(t) \cdot (\bolds{Q}(t+1) - \bolds{Q}(t)) }{\|\bolds{Q}(t)\|_2} \mid \set{F}_t} + \frac{\|\bolds{Q}(t+1)-\bolds{Q}(t)\|_2^2}{\varphi(t)} \nonumber\\
&\overset{\eqref{eq:drift-norm2-prob1bound}}{\leq} \expect{\frac{\sum_{n\in \set{N}} Q_n(t)\left(A_n(t) - \sum_{k \in \set{B}_n} \sigma_k(t)S_k(t)\right)}{\|\bolds{Q}(t)\|_2} \mid \set{F}_t} + \frac{\Marr+\Mmulti^2}{\varphi(t)} \nonumber\\
&= \frac{\sum_{n \in \set{N}}\lambda_nQ_n(t) - W_{\bolds{\sigma}(t)}(t)}{\|\bolds{Q}(t)\|_2} + \frac{\Marr+\Mmulti^2}{\varphi(t)} \label{eq:drift-simplify}\\
&\overset{\text{Lemma}~\ref{lem:multi-bound-weight}}{\leq} \frac{-\varepsilon\|\bolds{Q}(t)\|_1 + \Delta(t) \|\bolds{Q}(t)\|_{\infty}}{\|\bolds{Q}(t)\|_2}+\frac{\Marr+2\Mmulti^2}{\varphi(t)} \nonumber\\
&\leq -\varepsilon+\Delta(t)+\frac{\varepsilon}{4} \leq -\frac{\varepsilon}{4}+(\Delta(t)-\frac{\varepsilon}{2})^+, \nonumber
\end{align}
where the last equation is because arrivals and services in period $t$ are independent of $\set{F}_t$ and the second-to-last inequality is by the assumption that $\varphi(t)= \|\bolds{Q}(t)\|_2\geq \frac{4(\Marr+2\Mmulti^2)}{\varepsilon}$.
\end{proof}
We next show that $\Delta(t)$ is bounded, which enables us to bound $\expect{\varphi_t}$.
\begin{lemma}\label{lem:multi-bound-delta}
For every period $t$, $\Delta(t) \leq \Mmulti$.
\end{lemma}
\begin{proof}
Note that for every period $t$, with probability $1$ we have either that $\|\bolds{Q}(t)\|_{\infty} = 0$ and $\Delta(t) 
= 0$ or $\|\bolds{Q}(t)\|_{\infty} > 0$ and $\Delta(t) \leq \frac{W_{\bolds{\sigma}^{\MW}(t)}(t)}{\|\bolds{Q}(t)\|_{\infty}} \leq \sum_{n \in \set{N}}\sum_{k \in \set{B}_n} \sigma^{\MW}_k(t) \leq \Mmulti.$
\end{proof}

The final ingredient of the proof is to upper bound $\varphi(t)$ based on the drift property in Lemma~\ref{lem:multi-drift-norm2}. We show a general lemma that bounds the value of a process if it has the two properties we established in Lemma~\ref{lem:multi-drift-norm2} and the boundedness condition in Lemma~\ref{lem:multi-bound-delta}.
\begin{lemma}\label{lem:bound-from-drift}
Given two $\{\set{F}_t\}$-adapted processes $\{\Psi(t)\}, \{Z(t)\}$ with $\Psi(1) = 0$ and $Z(t) \geq 0~\forall t$, suppose the following properties hold for some positive constants $\kappa, \delta, B$:
\begin{enumerate}
    \item Bounded difference: $|\Psi(t+1) - \Psi(t)| \leq \kappa$;
    \item Drift bound: $\expect{\Psi(t+1)-\Psi(t) \mid \set{F}_t} \leq -\delta + Z(t)$ if $\Psi(t) \geq B$;
    \item Boundedness of $\{Z(t)\}$: $Z(t) \leq \kappa$. 
\end{enumerate}
Then we have    $\expect{\Psi(t)} \leq 4\kappa+B+\frac{2\kappa^2}{\delta}+\expect{\sum_{\tau=1}^{t-1} Z(t)}$ for every $t \geq 1$.
\end{lemma}
Lemma~\ref{lem:bound-from-drift} is similar, and builds upon, the following result by \cite{wei2023constant}:
\begin{lemma}\label{fact:bound-value-from-drift}\cite[Lemma~5]{wei2023constant}
Let $\Gamma(t)$ be an $\{\set{F}_t\}$-adapted process satisfying (i) Bounded difference: $|\Gamma(t+1)| - \Gamma(t) \leq K$, (ii) Expected decrease: $\expect{\Gamma(t+1) - \Gamma(t) \mid \set{F}_t} \leq -\eta$, when $\Gamma(t) \geq D$, and (iii) $\Gamma(0) \leq K+D$,
then we have
$
\expect{\Gamma(t)} \leq K\left(1+\left\lceil\frac{D}{K}\right\rceil\right)+K\left(\frac{K-\eta}{2\eta}\right).
$
\end{lemma}
Comparing our lemma to that of \cite{wei2023constant}, the main difference is that their version effectively fixes $Z(t)=0$, whereas our version  allows non-zero and time-varying $Z(t)$. This is crucial to incorporate the effect of learning. We note that both our proof and the one of \cite[Lemma~5]{wei2023constant} build on an elegant construction in \cite[Lemma~3]{gupta2022greedy} for reflected random processes.

\begin{proof}[Proof of Lemma~\ref{lem:bound-from-drift}.]
Similar to the proof of \cite[Proposition~4]{wei2023constant}, we define a sequence $\Gamma(t)$ such that $\Gamma(t) = \Psi(t)-\sum_{\tau=1}^{t-1}Z(\tau)$ for $t \geq 1$ and let $\Gamma(0) = 0$. For every $t$, we have
\[
|\Gamma(t+1)-\Gamma(t)| = |\Psi(t+1)-\Psi(t)-Z(t)| \leq |\Psi(t+1)-\Psi(t)|+Z(t) \leq 2\kappa,
\]
where the last inequality is by the assumptions of bounded difference and bounded $\{Z(t)\}$. Given that $Z(t) \geq 0$, if $\Gamma(t) \geq B$, we also have $\Psi(t) = \Gamma(t)+\sum_{\tau=1}^{t-1}Z(\tau) \geq B$ and hence
\begin{align*}
\expect{\Gamma(t+1)-\Gamma(t) \mid \set{F}_t} = \expect{\Psi(t+1)-\Psi(t)-Z(t) \mid \set{F}_t} &= \expect{\Psi(t+1)-\Psi(t) \mid \set{F}_t}-Z(t) \\
&\leq -\delta+Z(t)-Z(t) = -\delta,
\end{align*}
where the inequality is by the assumption of drift bound and that $\Psi(t) \geq B$ whenever $\Gamma(t) \geq B$.
Applying Lemma~\ref{fact:bound-value-from-drift} to $\{\Gamma(t)\}$ for every $t$ and setting $K = 2\kappa, \eta = \delta, D = B$, we have
\[
\expect{\Gamma(t)} \leq 2\kappa\left(1+\left\lceil\frac{B}{2\kappa}\right\rceil\right)+\frac{(2\kappa)^2}{2\delta}\leq 4\kappa+B+\frac{2\kappa^2}{\delta}.
\]
Therefore,
\[
\expect{\Psi(t)} = \expect{\Gamma(t)+\sum_{\tau=1}^{t-1}Z(\tau)} \leq 4\kappa+B+\frac{2\kappa^2}{\delta}+\expect{\sum_{\tau=1}^{t-1}Z(\tau)}.
\]
\end{proof}

The proof of Lemma~\ref{lem:multi-learning} combines Lemmas~\ref{lem:multi-drift-norm2},~\ref{lem:multi-bound-delta} and ~\ref{lem:bound-from-drift}.
\begin{proof}[Proof of Lemma~\ref{lem:multi-learning}]
We apply Lemma~\ref{lem:bound-from-drift} to the Lyapunov function $\{\varphi(t)\}$ and let $Z(t) = (\Delta(t)-\frac{\varepsilon}{2})^+$. To verify the conditions in Lemma~\ref{lem:bound-from-drift}, note that $\{\varphi(t)\}$ and $\{Z(t)\}$ are non-negative with $\varphi(1) = 0$ and $\{\set{F}_t\}$-adapted where $\set{F}_t$ is the $\sigma-$field generated by the history before period $t$. Using Lemma~\ref{lem:multi-drift-norm2}, we set $\kappa = \sqrt{\Marr + \Mmulti^2}, \delta = \frac{\varepsilon}{4}$ and $B = \frac{4(\Marr+2\Mmulti^2)}{\varepsilon}$ to satisfy the assumptions of bounded difference and drift bound. By Lemma~\ref{lem:multi-bound-delta}, we also have $Z(t) \leq \Delta(t) \leq \Mmulti \leq \kappa$. Applying Lemma~\ref{lem:bound-from-drift} to all $\varphi(t)$ with $t \leq T$ and taking average, we have
\begin{align*}
\frac{\sum_{t=1}^T \expect{\varphi(t)}}{T} &\leq 4\kappa+B+\frac{2\kappa^2}{\delta}+\expect{\sum_{t=1}^{T-1} Z(t)} \\
&=4\sqrt{\Marr + \Mmulti^2}+\frac{4(\Marr+2\Mmulti^2)}{\varepsilon} + \frac{8(\Marr+\Mmulti^2)}{\varepsilon}+\expect{\sum_{t=1}^{T-1} (\Delta(t)-\frac{\varepsilon}{2})^+} \\
&\leq \frac{16\Marr + 20\Mmulti^2}{\varepsilon} + \expect{\sar^{\multi}(\pi,T)},
\end{align*}
which completes the proof.
\end{proof}

\subsection{Queue length bound in the regenerate stage (Lemma~\ref{lem:multi-regenerate})}
The proof is similar to that of Lemma~\ref{lem:single-queue-regen} for the single-queue case. The difference is that with multiple queues, we use another Lyapunov function $V(t) = \sum_{n \in \set{N}} Q_n^2(t)$ and need the drift result of $\bolds{\sigma}(t)$ that we derived in Lemma~\ref{lem:multi-bound-weight}. We first show the following bound, similar to Lemma \ref{lem:single-connect-max-sum}, between the maximum (across time) total (across queues) queue length and the cumulative total queue length for a fixed horizon.
\begin{lemma}\label{lem:multi-max-total}
For any horizon $T$ and every sample path, we have $\sum_{t=1}^T \|\bolds{Q}(t)\|_1 \geq \frac{\left(\max_{t=1}^T \|\bolds{Q}(t)\|_1\right)^2}{4\Marr}$.
\end{lemma}
\begin{proof}
Since $\sum_{n \in \set{N}} A_n(t) \leq \Marr$, the maximum increase in all queues' combined lengths in a period is $\sum_{n \in \set{N}} (Q_n(t)+A_n(t)) - \sum_{n \in \set{N}} Q_n(t) \leq \Marr$. That is, for any  $t$, $\|\bolds{Q}(t+1)\|_1 \leq \|\bolds{Q}(t)\|_1 + \Marr$. 

Now fix a horizon $T$. Let $t_{\max} \in \arg\max_{t \leq T} \|\bolds{Q}(t)\|_1$ and denote $Q_{\max} = \|\bolds{Q}(t_{\max})\|_1$. Then $t_{\max} \geq \lceil \frac{Q_{\max}}{\Marr}\rceil$. If $Q_{\max} \leq 2\Marr$, we have
$
\sum_{t=1}^T \|\bolds{Q}(t)\|_1 \geq Q_{\max} \geq \frac{Q_{\max}^2}{4\Marr}.
$
Otherwise,
\begin{align*}
\sum_{t = 1}^T \|\bolds{Q}(t)\|_1 \geq \sum_{t = t_{\max} - \lfloor \frac{Q_{\max}}{\Marr} \rfloor + 1}^{t_{\max}} \|\bolds{Q}(t)\|_1 &\geq \sum_{i=0}^{\lfloor \frac{Q_{\max}}{\Marr} \rfloor-1} \left(Q_{\max} - i\Marr\right) \\
&\geq Q_{\max}\left\lfloor \frac{Q_{\max}}{\Marr} \right\rfloor - \frac{\Marr}{2}\left\lfloor \frac{Q_{\max}}{\Marr} \right\rfloor\left(\left\lfloor \frac{Q_{\max}}{\Marr} \right\rfloor-1\right) \\
&\geq Q_{\max}\left(\frac{Q_{\max}}{\Marr} - 1\right)-\frac{\Marr}{2}\frac{Q_{\max}}{\Marr}\left(\frac{Q_{\max}}{\Marr} - 1\right) \\
&= \frac{Q_{\max}}{2}\left(\frac{Q_{\max}}{\Marr}-1\right) \geq \frac{Q_{\max}^2}{4\Marr},
\end{align*}
which completes the proof.
\end{proof}
\begin{proof}[Proof of Lemma~\ref{lem:multi-regenerate}]
For a period $t$, recall that the Lyapunov function $V(t) = \sum_{n \in \set{N}} Q_n^2(t)$ and that the filtration $\set{F}_t$ is the $\sigma-$field generated by the history before period $t$. Condition on $\set{F}_t$, the drift of this Lyapunov function is upper bounded by
\begin{align*}
\expect{V(t+1)-V(t) \mid \set{F}_t} &= \expect{\sum_{n \in \set{N}} (Q_n^2(t+1)-Q_n^2(t)) \mid \set{F}_t}\\
&=  \expect{\sum_{n \in \set{N}} \left(\left(Q_n(t)+A_n(t)-\sum_{k \in \set{B}_n} \sigma_k(t)S_k(t)\right)^2-Q_n^2(t)\right)\mid \set{F}_t} 
\end{align*}
\begin{align}
&\hspace{-0.4in}\leq \expect{\sum_{n \in \set{N}} \left(A^2_n(t)+\left(\sum_{k\in \set{B}_n} \sigma_k(t)\right)^2 + 2\left(A_n(t)-\sum_{k\in\set{B}_n}\sigma_k(t)S_k(t)\right)Q_n(t)\right) \mid \set{F}_t} \nonumber\\
&-\expect{\sum_{n \in \set{N}} 2A_n(t)\sum_{k \in \set{B}_n}\sigma_k(t)S_k(t) \mid \set{F}_t} \nonumber\\
&\leq \Marr + \Mmulti^2 + 2\sum_{n \in \set{N}} \left(\lambda_n - \sum_{k \in \set{B}_n}\sigma_k(t)\mu_k\right)Q_n(t) \label{eq:sqr-drift-simplify}\\
&\leq \Marr + 3\Mmulti^2 -2\varepsilon\sum_{n \in \set{N}} Q_n(t) + 2\Delta(t)\|\bolds{Q}(t)\|_{\infty} \label{eq:multi-sqre-drift}
\end{align}
where the second to last inequality is because $\bolds{\sigma}(t),\bolds{Q}(t) \in \set{F}_t$ and $\{A_n(t)\}_{n \in \set{N}}, \{S_k(t)\}_{k \in \set{K}}$ are independent of $\set{F}_t$; the last inequality uses Lemma~\ref{lem:multi-bound-weight}.

For fixed $T$ we take expectation on both sides of \eqref{eq:multi-sqre-drift} and telescope across $t \leq T$ to obtain
\begin{align}
0 \leq \expect{V(T+1)} &\leq T(\Marr + 3\Mmulti^2)-2\varepsilon\expect{\sum_{t=1}^T \|\bolds{Q}(t)\|_1}+2\expect{\sum_{t=1}^T\Delta(t)\|\bolds{Q}(t)\|_{\infty}} \nonumber\\
&\leq T(\Marr + 3\Mmulti^2)-\varepsilon\expect{\sum_{t=1}^T \|\bolds{Q}(t)\|_1}+2\expect{\sum_{t=1}^T(\Delta(t)-\frac{\varepsilon}{2})^+\|\bolds{Q}(t)\|_{1}}.\label{eq:multi-sqre-drift-decompose}
\end{align}
To complete the proof, we use the following claim (proof follows below):
\begin{claim}\label{claim:sample-path-bound}
$-\frac{\varepsilon}{2}\expect{\sum_{t=1}^T \|\bolds{Q}(t)\|_1}+2\expect{\sum_{t=1}^T(\Delta(t)-\frac{\varepsilon}{2})^+\|\bolds{Q}(t)\|_{1}} \leq \frac{8\Marr\sar^{\multi}(\pi,T)^2}{\varepsilon}$.
\end{claim}
\noindent Using Claim~\ref{claim:sample-path-bound} in \eqref{eq:multi-sqre-drift-decompose} gives $0 \leq T(\Marr + 3\Mmulti^2)-\frac{\varepsilon\expect{\sum_{t=1}^T \|\bolds{Q}(t)\|_1}}{2} + \frac{8\Marr\expect{\sar^{\multi}(\pi,T)^2}}{\varepsilon}$. Therefore,
\[
\frac{\expect{\sum_{t=1}^T \|\bolds{Q}(t)\|_1}}{T} \leq \frac{2\Marr+ 6\Mmulti^2}{\varepsilon} + \frac{16\Marr \expect{\sar^{\multi}(\pi,T)^2}}{\varepsilon^2 T}.
\]
\end{proof}
\begin{proof}[Proof of Claim~\ref{claim:sample-path-bound}]
We prove the claim by
\begin{align*}
-\frac{\varepsilon}{2}\sum_{t=1}^T \|\bolds{Q}(t)\|_1+2\sum_{t=1}^T(\Delta(t)-\frac{\varepsilon}{2})^+\|\bolds{Q}(t)\|_{1}&\leq-\frac{\varepsilon}{2}\sum_{t=1}^T \|\bolds{Q}(t)\|_1 + 2\sar^{\multi}(\pi,T)\max_{t\leq T}\|\bolds{Q}(t)\|_1 \nonumber\\
&\hspace{-2in}\leq \frac{-\varepsilon}{8\Marr}\left(\max_{t\leq T}\|\bolds{Q}(t)\|_1\right)^2 + 2\sar^{\multi}(\pi,T)\max_{t\leq T}\|\bolds{Q}(t)\|_1 \leq \frac{8\Marr\sar^{\multi}(\pi,T)^2}{\varepsilon},
\end{align*}
where the second inequality is by Lemma~\ref{lem:multi-max-total} and the last inequality uses $\max_x -ax^2+bx=\frac{b^2}{4a}$ for any $a>0$. 
\end{proof}

\section{Optimal transient cost of learning for queueing networks}\label{sec:optimal-network}
This section introduces a policy that learns to schedule a queueing network near-optimally. Section~\ref{sec:network-model} formally defines our model of  queueing networks which generalizes multi-queue multi-server system by allowing jobs to probabilistically transition to queues after obtaining service. Section~\ref{sec:network-res} gives the \textsc{BP-UCB} algorithm with near-optimal transient cost of learning in queueing networks.

\subsection{A model of queueing networks}\label{sec:network-model}
Our model of queueing networks extends that in Section~\ref{sec:multi-model} by allowing \emph{job transitions} and captures the canonical network model in \cite{tassiulas1992stability} (see Appendix~\ref{app:examples-network}). Specifically, we define a queueing network as a tuple $(\set{N},\set{K},\bolds{\Lambda},\bolds{\mu}, \bolds{\set{A}}, \bolds{\Sigma}, \bolds{\set{B}}, \bolds{\set{D}}, \bm{P})$, where $\bolds{\set{B}}= \{\set{B}_n\}_{n \in \set{N}}$, $\bolds{\set{D}} = \{\set{D}_k\}_{k \in \set{K}}$ and $\bm{P} = (p_{k,n})_{k \in \set{K},n\in \set{N} \cup \{\perp\}}$ such that $\sum_{n \in \set{N} \cup \{\perp\}} p_{k,n} = 1, \forall k$. In contrast to the previous setting, there is also a virtual queue $\perp$ to which jobs transition upon leaving the system. Each server $k\in \set{B}_n$ has a set of destination queues $\set{D}_k$, which can include the virtual queue $\perp$. As in the  previous setting, the DM selects a feasible schedule $\bolds{\sigma}(t) \in \bolds{\Sigma}_t$ (defined in \eqref{eq:feasible-schedule-set}) for a period~$t$.

The queueing dynamic is similar to that of multi-queue multi-server systems, but includes a new transition component. Formally, let $\bolds{L}_k(t) = (L_{k,n}(t))_{n \in \set{N} \cup \{\perp\}}$ be a random vector over $\{0,1\}^{\set{N} \cup \{\perp\}}$ for server $k$ independent of other randomness such that $\Pr\{L_{k,n}(t)=1\} = p_{k,n}$ and $\sum_{n \in \set{N} \cup \{\perp\}} L_{k,n}(t) = 1$. The queue length of each queue evolves by
\begin{equation}\label{eq:dynamic-network}
Q_n(t+1) = Q_n(t) - \sum_{k \in \set{B}_n} \sigma_k(t)S_k(t) + A_n(t) + \sum_{k' \in \set{K}} \sigma_{k'}(t)S_{k'}(t)L_{k',n}(t). 
\end{equation}
As before, the DM knows $\bolds{\Sigma}$ and $\bolds{\set{B}}$, but not the rates $\bolds{\lambda},\bolds{\mu},\bm{P}$ and the set $\bolds{\set{D}}$. In period~$t$, the observed history is the set $\left(\{A_n(\tau)\}_{n \in \set{N}}, \{S_k(\tau)L_{k,n}(\tau)\}_{n \in \bar{\set{N}},k \in \set{K} \colon \sigma_k(\tau) = 1}\right)_{\tau < t}$ that includes transition information on top of arrivals and services. Note that a job transition is only observed when the server is selected and the service is successful. A non-anticipatory policy $\pi$ maps an observed history to a feasible schedule.

We define the traffic slackness of a queueing network similar to the previous setting. Letting $\bolds{\Phi}$ be the probability simplex over $\bolds{\Sigma}$ and $\bolds{\phi} \in \bolds{\Phi}$ be a distribution over feasible schedules, the effective service rate queue $n$ can get under $\bolds{\phi}$ is $\mu_n^{\net}(\bolds{\phi}) = \sum_{\bolds{\sigma} \in \bolds{\Sigma}} \bolds{\phi}_{\bolds{\sigma}}\left(\sum_{k \in \set{B}_n} \sigma_k\mu_k - \sum_{k' \in \set{K}} \sigma_{k'}\mu_{k'}p_{k',n}\right)$;  this includes both job inflow and outflow. With $\bolds{\mu}^{\net}(\bolds{\phi})$ denoting the vector of effective service rates, the capacity region is $\set{S}(\bolds{\mu},\bolds{\Sigma},\bolds{\set{B}}, \bm{P}) = \{\bolds{\mu}^{\net}(\bolds{\phi}) \colon \bolds{\phi} \in \bolds{\Phi}\}$. As in the previous setting, we introduce $\mathscr{E} = \{\varepsilon \in (0,1] : \bolds{\lambda}(\bolds{\Lambda}) + \varepsilon \bm{1} \in \set{S}(\bolds{\mu},\bolds{\Sigma},\bolds{\set{B}}, \bm{P})\}$ to define the traffic slackness.
\begin{definition}\label{def:network-slackness}
The traffic slackness of a queueing network is defined by $\varepsilon = \max \{\varepsilon' \in \mathscr{E}\}.$
\end{definition}
As in the multi-queue multi-server setting, we define $\colq$ for  queueing networks by comparing the expected time-averaged queue length of a policy with that of any non-anticipatory policy:
\begin{equation}\label{def:colq-network}
\colq^{\net}(\bolds{\Lambda},\bolds{\mu}, \bolds{\Sigma}, \bolds{\set{B}}, \bolds{\set{D}}, \bm{P}, \pi) = \max_{\text{non-anticipatory }\pi'}\max_{T \geq 1} \frac{\sum_{t=1}^T\sum_{n \in \set{N}}\expect{Q_n(t,\pi) - Q_n(t,\pi')}}{T}.
\end{equation}
Accordingly, the worst-case transient cost of learning $\colq^{\net}(\bolds{\set{A}},\bolds{\Sigma}, \bolds{\set{B}},\bolds{\set{D}},\varepsilon, \pi)$ takes the supremum of $\colq^{\net}$ across any arrival, service, and transition rates with the given traffic slackness~$\varepsilon$.

\subsection{Algorithm and main results}\label{sec:network-res}
Even with known parameters, it is known that the \textsc{MaxWeight} policy upon which $\textsc{MW-UCB}$ in Section~\ref{sec:optimal-multi} is developed can fail to stabilize a queueing network with positive traffic slackness \cite{BramsonDW21}. As a result, we develop the $\textsc{BackPressure-UCB}$ algorithm (\textsc{BP-UCB}) to augment the classical $\textsc{BackPressure}$ algorithm (\textsc{BP})\cite{tassiulas1992stability}.Assuming knowledge of parameters, \text{BP} selects a schedule 
\begin{equation}\label{eq:rule-bp}
\bolds{\sigma}^{\BP}(t) \in \arg\max_{\bolds{\sigma} \in \bolds{\Sigma}_t} \sum_{n \in \set{N}} Q_n(t)\left(\sum_{k \in \set{B}_n} \sigma_k \mu_k - \sum_{k' \in \set{K}}\sigma_{k'}\mu_{k'}p_{k',n}\right).
\end{equation}
The term $\sigma_{k'}\mu_{k'}p_{k',n}$ denotes the probability that a job from the queue that server $k'$ belongs to will transit to queue $n$: if $\sigma_{k'}=1$, then the service is successful with probability $\mu_{k'}$ and the job transitions to $n$ with probability $p_{k',n}$. Comparing the weight function of $\text{BP}$ to that of $\textsc{MaxWeight}$, this additional term penalizes job transitions into long queues. It is shown in \cite{tassiulas1992stability} that \textsc{BP} stabilizes a queueing network with traffic slackness $\varepsilon > 0$ (Definition~\ref{def:network-slackness}) and it is known that the time-averaged queue length is $O(\frac{1}{\varepsilon})$ \cite[lemma 4.1]{georgiadis2006resource}. 
It is evident that \textsc{BP} requires knowledge of system parameters. To address the need of learning, we can maintain a upper confidence weight of a schedule like in \textsc{MW-UCB} (Algorithm~\ref{algo:mw-ucb}). Intuitively, one may want to do so by replacing $\mu_k$ by its UCB, the $\mu_{k'}$ and $p_{k',n}$ by associated lower confidence bounds (LCB) in \eqref{eq:rule-bp}. This approach would require collecting $L_{k',n}(t)$ to make the estimation of $p_{k',n}$ better over time. However, the sampling rate of $L_{k',n}(t)$ depends on $\mu_{k'}$: when $\mu_{k'} = 0$, the DM observes no job transitions from server $k'$ (since no job can receive successful service). To bypass this pitfall, we instead collect samples of $R_{k',n}(t) = S_{k'}(t)L_{k',n}(t)$, which is equal to $1$ if there is a job transition. With $r_{k',n} = \mu_{k'}p_{k',n}$, the independence of $S_{k'}(t)$ and $L_{k',n}(t)$, implies that $\expect{R_{k',n}(t)} = r_{k',n}$.

Our algorithm \textsc{BP-UCB} then works as follows. We set $\bar{\mu}_k(t) = \min\left(1,\hat{\mu}_k(t)+\sqrt{\frac{2\ln t}{C_k(t)}}\right)$ where $\hat{\mu}_k(t)$ is the sample mean of service $S_k$ and $C_k(t)$ is the number of periods in which server $k$ is selected. In addition, since $R_{k,n}$ is observable whenever $k$ is selected, we also maintain the sample mean of $R_{k,n}$, denoted $\hat{r}_{k,n}(t)$, and set a LCB estimation $\ubar{r}_{k,n}(t) = \max\left(0,\hat{r}_{k,n}(t)-\sqrt{\frac{2\ln t}{C_k(t)}}\right)$.  If $n \not \in \set{D}_k$, i.e., when $n$ is not a destination queue of server $k$, we always have $\hat{r}_{k,n}(t)$, and thus $\ubar{r}_{k,n}(t) = 0$. In period $t$, the algorithm selects a schedule 
\begin{equation}\label{eq:rule-bp-ucb}
\bolds{\sigma}(t) \in \arg\max_{\bolds{\sigma} \in \bolds{\Sigma}_t} \sum_{n \in \set{N}} Q_n(t)\left(\sum_{k \in \set{B}_n} \sigma_k \bar{\mu}_k(t) - \sum_{k' \in \set{K}}\sigma_{k'}\ubar{r}_{k',n}(t)\right).
\end{equation}
\textsc{BP-UCB} (Algorithm~\ref{algo:bp-ucb}) does not need knowledge of destination queues $\{\set{D}_k\}_{k \in \set{K}}$.

\begin{algorithm}[H]
\LinesNumbered
\DontPrintSemicolon
\caption{\textsc{BP-UCB} policy for a queueing network
\label{algo:bp-ucb}
}
\SetKwInOut{Input}{input}\SetKwInOut{Output}{output}
\Input{set of arrivals $\bolds{\set{A}}$, set of schedules $\bolds{\Sigma}$, set of servers belonging to queues $\{\set{B}_n\}_{n \in \set{N}}$}
Sample mean $\hat{\mu}_{k}(1) \gets 0, \hat{r}_{k,n}(1) \gets 0$, number of samples $C_k(1) \gets 0$ for $k \in \set{K}, n \in \set{N}$\;
\For{$t = 1\ldots$}{
    $\bar{\mu}_{k}(t) = 
    \min\left(1, \hat{\mu}_{k}(t) + \sqrt{\frac{2\ln(t)}{C_k(t)}}\right), \forall k \in \set{K}$\label{line:bp-ucb-esti}\;
    $\ubar{r}_{k,n}(t) = 
    \max\left(0, \hat{r}_{k}(t) -\sqrt{\frac{2\ln(t)}{C_k(t)}}\right), \forall k \in \set{K}, n \in \set{N}$\label{line:bp-lcb-esti}\;
    \tcc{select a max-weight feasible schedule based on UCB estimation}
    $\bolds{\Sigma}_t = \{\sigma \in \bolds{\Sigma} \colon \sum_{k \in \set{B}_n} \sigma_k \leq Q_n(t), \forall n \in \set{N}\}$ \\
    $
\sigma(t) \in \arg\max_{\sigma \in \bolds{\Sigma}_t} \sum_{n \in \set{N}} Q_n(t)\left(\sum_{k \in \set{B}_n} \sigma_k \bar{\mu}_k(t) - \sum_{k' \in \set{K}}\sigma_{k'}\ubar{r}_{k',n}(t)\right)$ \label{line:bp-ucb-choice}\\
   \tcc{Update $\bolds{Q}(t+1)$~\& estimates via $\{A_n(t), \sigma_k(t)S_k(t),\sigma_k(t)R_{k,n}(t)\}_{k \in \set{K},n \in \set{N}}$}
   $Q_n(t+1) \gets Q_n(t) - \sum_{k \in \set{B}_n} \sigma_k(t)S_k(t)+A_n(t)+\sum_{k'\in \set{K}}\sigma_{k'}(t)R_{k',n}(t), \forall n\in \set{N}$ \\
    $C_{k}(t+1) \gets C_{k}(t)+1, \hat{\mu}_{k}(t+1) \gets \frac{C_{k}(t)\hat{\mu}_{k}(t)+S_{k}(t)}{C_{k}(t+1)}, \hat{r}_{k,n}(t+1) \gets \frac{C_{k}(t)\hat{r}_{k,n}(t)+R_{k,n}(t)}{C_{k}(t+1)}\forall k \in \set{K}, n \in \set{N},\sigma_k(t)=1$
}
\end{algorithm}

Fixing the system structure $\bolds{\set{A}}, \bolds{\Sigma}, \bolds{\set{B}}, \bolds{\set{D}}$, {recall that} $\Marr = \max_{\bolds{A}\in \bolds{\set{A}}} \sum_{n \in \set{N}} A_n$ and $\Mmulti = \max_{\bolds{\sigma} \in \bolds{\Sigma}}\sum_{n\in\set{N}} \sum_{k \in \set{B}_n} \sigma_k$ {are} the maximum number of new job arrivals and the maximum number of selected servers per period. {Let} $\Mdep = \sum_{k \in \set{K}} |\set{D}_k|^2$ where we recall that $\set{D}_k$ is the set of queues server $k$ may see its jobs transition to. The following result shows that the worst-case cost of learning of {\textsc{BP-UCB}} has optimal dependence on $\frac{1}{\varepsilon}$.
\begin{theorem}\label{thm:bp-ucb-col}
For any $\bolds{\set{A}},\bolds{\Sigma},\bolds{\set{B}},\bolds{\set{D}}$ and traffic slackness $\varepsilon \in (0,1]$, we have
\[
\colq^{\net}(\bolds{\set{A}},\bolds{\Sigma},\bolds{\set{B}},\bolds{\set{D}},\varepsilon,\textsc{BP-UCB}) \leq \frac{\sqrt{N}\left(32\Marr+2^{12}\Mdep\Mmulti^2\left(1 + \ln(\Marr\Mdep\Mmulti/\varepsilon)\right)\right)}{\varepsilon}.
\]
\end{theorem}

 We fix a particular setting of system parameters $\bolds{\Lambda},\bolds{\mu}, \bm{P}$ such that the traffic slackness condition (Definition~\ref{def:network-slackness}) holds for a fixed $\varepsilon > 0$. We aim to upper bound the time-averaged queue length of this system under an arbitrary policy $\pi$. The main difference to multi-queue multi-server systems is that we need to accommodate job transitions in the drift analysis and define the satisficing regret in terms of a new weight function. Define the weight of a schedule $\bolds{\sigma} \in \bolds{\Sigma}$ in period $t$ by 
\begin{equation}
W^{\net}_{\bolds{\sigma}}(t) = \sum_{n \in \set{N}} Q_n(t)\left(\sum_{k \in \set{B}_n} \sigma_k \mu_k - \sum_{k' \in \set{K}} \sigma_{k'}\mu_{k'}p_{k',n}\right).
\end{equation}
Similar to \eqref{eq:def-multi-loss}, we define the loss of choosing a schedule $\bolds{\sigma}(t)$ compared with the $\BP$ schedule by
\begin{equation}\label{eq:def-network-loss}
\Delta^{\net}(t) = \left\{
\begin{aligned}
&0,~\text{if~}\|\bolds{Q}(t)\|_{\infty}=0\\
&\frac{W^{\net}_{\bolds{\sigma}^{\BP}(t)}(t) - W^{\net}_{\bolds{\sigma}(t)}(t)}{\|\bolds{Q}(t)\|_{\infty}},~\text{otherwise}
\end{aligned}
\right.
\end{equation}
and the satisficing regret by
$
\sar^{\net}(\pi,T) = \sum_{t=1}^T \left(\Delta^{\net}(t)-\frac{\varepsilon}{2}\right)^+
$for a policy $\pi$ over the first $T$ periods. The following lemma, similar to Lemma~\ref{lem:multi-learning}, bounds the time-averaged queue length through the expected satisficing regret. 
\begin{lemma}\label{lem:network-learning}
For any policy $\pi$ and horizon $T$, $\frac{\sum_{t=1}^T \expect{\|\bolds{Q}(t)\|_2}}{T} \leq \frac{32\Marr+64\Mmulti^2}{\varepsilon}+\expect{\sar^{\net}(\pi, T)}.$
\end{lemma}
The proof of Lemma~\ref{lem:network-learning} follows a similar argument to that of Lemma~\ref{lem:multi-learning}: we again study the drift of the Lyapunov function $\varphi(t) = \|\bolds{Q}(t)\|_2$, though the drift analysis is more complicated due to job transitions. We then use Lemma~\ref{lem:bound-from-drift} to establish an upper bound on the time-average of $\varphi(t)$. A full proof is given in Appendix~\ref{app:network-learning}.

As in Sections \ref{sec:optimal-single} and \ref{sec:optimal-multi}, the bound in Lemma~\ref{lem:network-learning} is not useful for the regenerate stage, requiring a lemma that bounds the time-average queue length with the second moment of the satisficing regret divided by the horizon length. We show the following counterpart of Lemma~\ref{lem:multi-regenerate}.
\begin{lemma}\label{lem:network-regenerate}
For any policy $\pi$ and horizon $T$, $\frac{\sum_{t=1}^T \expect{\|\bolds{Q}(t)\|_1}}{T} \leq \frac{4\Marr+10\Mmulti^2}{\varepsilon}+\frac{16\Marr}{\varepsilon^2}\frac{\expect{\sar^{\net}(\pi, T)^2}}{T}.$
\end{lemma}
The proof of Lemma~\ref{lem:network-regenerate} relies on a drift analysis of the Lyapunov function $V(t) = \sum_{n \in \set{N}} Q_n^2(t)$, which expands the proof of Lemma~\ref{lem:multi-regenerate} by considering job transitions. We refer readers to Appendix~\ref{app:network-regenerate} for the full proof.

Finally, similar to Lemma~\ref{lem:sar-mw-ucb}, we bound the satisficing regret of $\textsc{BP-UCB}$. In contrast to the previous two lemmas, this bound has an explicit dependence on the transition structure $\{\set{D}_k\}_{k \in \set{K}}$. \begin{lemma}\label{lem:sar-bp-ucb}
For any horizon $T$,
\begin{align*}
    \expect{\sar^{\net}(\textsc{BP-UCB},T)} &\leq \frac{128\Mdep\Mmulti^2(\ln T+1)}{\varepsilon} \\
    \expect{\sar^{\net}(\textsc{BP-UCB},T)^2} &\leq \frac{2^{15}\Mdep^2\Mmulti^4(\ln T+1)^2}{\varepsilon^2}.
\end{align*}
\end{lemma}
The proof of Lemma~\ref{lem:sar-bp-ucb} has a similar structure as that of Lemma~\ref{lem:sar-mw-ucb}. The additional complexity comes from estimating $r_{k,n} = \mu_k p_{k,n}$, the probability a job transitions to queue $n$ given that server $k$ is selected. See Appendix~\ref{app:sar-bp-ucb} for the detailed proof.

The proof of Theorem~\ref{thm:bp-ucb-col} combines Lemmas~\ref{lem:network-learning}, ~\ref{lem:network-regenerate}, and~\ref{lem:sar-bp-ucb}.

\begin{proof}[Proof of Theorem~\ref{thm:bp-ucb-col}]
Fix the system structure $\bolds{\set{A}}, \bolds{\Sigma}, \bolds{\set{B}}, \bolds{\set{D}}$ and the traffic slackness $\varepsilon > 0$. Let $\bolds{\Lambda},\bolds{\mu}, \bm{P}$ be any tuple of an arrival probability distribution over $\bolds{\set{A}}$, a service rate vector, and a transition probability matrix such that $\bolds{\lambda} + \varepsilon \bm{1} \in \set{S}^{\net}(\bolds{\mu}, \bolds{\Sigma},\bolds{\set{B}}, \bm{P})$ and that $\bm{P}$ has positive probability from server $k$ to queue $n$ only if $n \in \set{D}_k$. We next upper bound $\colq^{\net}(\bolds{\Lambda},\bolds{\mu},\bolds{\Sigma},\bolds{\set{B}}, \bolds{\set{D}},\bm{P},\textsc{BP-UCB})$ by bounding the maximum time-averaged queue length over all periods.

Define $T_1 = \left\lfloor\left(\frac{2^{19}\Marr\Mdep^2\Mmulti^4}{\varepsilon^4}\right)^2\right\rfloor$. We first bound the time-averaged queue length for the first $T$ periods with $T \leq T_1$ via Lemma~\ref{lem:network-learning} by
\begin{align}
\frac{\sum_{t=1}^T \expect{\|\bolds{Q}(t)\|_2}}{T} &\leq \frac{32\Marr + 64\Mmulti^2}{\varepsilon} + \expect{\sar^{\net}(\textsc{MW-UCB},T} \tag{Lemma~\ref{lem:network-learning}} \\
&\leq \frac{32\Marr + 64\Mmulti^2}{\varepsilon} + \frac{128\Mdep\Mmulti^2(\ln T+1)}{\varepsilon} \tag{Lemma~\ref{lem:sar-bp-ucb}}\\
&\leq \frac{32\Marr+192\Mdep\Mmulti^2}{\varepsilon}+\frac{256\Mdep\Mmulti^2\ln\left(2^{19}\Marr\Mdep^2\Mmulti^4/\varepsilon^4\right)}{\varepsilon} \tag{$T \leq T_1$} \\
&\leq \frac{32\Marr+192\Mdep\Mmulti^2}{\varepsilon}+\frac{256\Mdep\Mmulti^2\left(\ln(2^{19})+4\ln(\Marr\Mdep\Mmulti/\varepsilon)\right)}{\varepsilon} \nonumber\\
& \leq \frac{32\Marr + 2^{12}\Mdep\Mmulti^2(1+\ln(\Marr\Mdep\Mmulti/\varepsilon))}{\varepsilon}.\nonumber
\end{align}
Using the fact that $\|\bolds{Q}(t)\|_1 \leq \sqrt{N}\|\bolds{Q}(t)\|_2$ for every period $t$ gives for every $T \leq T_1$,
\[
\frac{\sum_{t=1}^T \expect{\|\bolds{Q}(t)\|_1}}{T} \leq \frac{\sqrt{N}\left(32\Marr + 2^{12}\Mdep\Mmulti^2(1+\ln(\Marr\Mdep\Mmulti/\varepsilon))\right)}{\varepsilon}.
\]
For $T > T_1$, applying Lemma~\ref{lem:network-regenerate} gives
\begin{align*}
\frac{\sum_{t=1}^T \expect{\|\bolds{Q}(t)\|_1}}{T} &\leq \frac{4\Marr+10\Mmulti^2}{\varepsilon}+\frac{16\Marr}{\varepsilon^2}\frac{\expect{\sar^{\net}(\textsc{BP-UCB},T)^2}}{T} \tag{Lemma~\ref{lem:network-regenerate}} \\
&\leq \frac{4\Marr+10\Mmulti^2}{\varepsilon}+\frac{2^{19}\Marr\Mdep^2\Mmulti^4(\ln T+1)^2}{\varepsilon^4 T}  \tag{Lemma~\ref{lem:sar-bp-ucb}}\\
&\leq \frac{4\Marr+10\Mmulti^2}{\varepsilon} +\frac{2^{19}\Marr\Mdep^2\Mmulti^4}{\varepsilon^4 \sqrt{T}} \tag{Fact~\ref{fact:lnt-sqrt-prop} and $T > T_1 \geq 50000$} \\
&\leq \frac{4\Marr+10\Mmulti^2}{\varepsilon} + 1 \tag{$T > T_1$ implies $\sqrt{T} \geq \frac{2^{19}\Marr\Mdep^2\Mmulti^4}{\varepsilon^4}$}.
\end{align*}
Summarizing the two cases for $T \leq T_1$ and $T > T_1$ then gives 
\begin{align*}
\colq^{\net}(\bolds{\Lambda},\bolds{\mu},\bolds{\Sigma},\bolds{\set{B}}, \bolds{\set{D}},\bm{P},\textsc{BP-UCB}) &\overset{\eqref{def:colq-network}}{\leq} \max_{T} \frac{\sum_{t=1}^T \expect{\|\bolds{Q}(t)\|_1}}{T} \\
&\leq \frac{\sqrt{N}\left(32\Marr + 2^{12}\Mdep\Mmulti^2(1+\ln(\Marr\Mdep\Mmulti/\varepsilon))\right)}{\varepsilon},
\end{align*}
which finishes the proof by the definition of worst-case $\colq$.
\end{proof}

\section{Numerics}
In this section we complement our theoretical findings through numerical results. Our goal is two-fold: to compare our UCB-based algorithms with others suggested in the literature and to evaluate the tightness of our bounds where our analysis involves large constants. We divide our analysis into one part on the single-queue multi-server setting and one on the multi-queue multi-server setting.

\subsection{Single-queue multi-server setting}
For the single-queue setting (Sections~\ref{sec:model-single} and \ref{sec:optimal-single}), we aim to (i) evaluate both queue length and time-averaged queue length of different algorithms and (ii) explore the scaling of $K$ and $\varepsilon$ in our lower and upper bounds on $\colq$ (which involve large constants in Theorems~\ref{thm:colq-lowerbound} and \ref{thm:colq-ucb-single}). In all simulations below, we consider a time horizon with $T = 10^7$ periods and approximate expectation of a random variable by the sample average of $30$ independent sample paths.

\paragraph{Comparison between algorithms.}
Based on the setting from Figure~\ref{fig:col-motivation}, we numerically compare the transient performance of UCB (Algorithm~\ref{algo:single-ucb}), Q-UCB (Algorithm~\ref{algo:q-ucb}) proposed in \cite{KrishnasamySJS21}, and the algorithm in \cite[figure 7]{StahlbuhkSM21}. Our analysis established that UCB enjoys a $\colq$ bound of $O(\frac{K(\ln K+\ln(1/\varepsilon))}{\varepsilon})$ (Theorem~\ref{thm:colq-ucb-single}); Q-UCB has a comparable scaling but with a larger constant (Proposition~\ref{prop:colq-qucb} in Appendix~\ref{app:qucb}). In Appendix~\ref{app:ssm} we show that the upper bound in \cite{StahlbuhkSM21} implies a $\colq$ no better than $\Theta(1 / \varepsilon^4)$ for their algorithm (we do not claim that this is a lower bound on their algorithm's $\colq$ but rather that their analysis cannot be leveraged into a better upper bound). Combining these results, we expect UCB to have a better transient performance than Q-UCB and both of them to have a much better performance than the algorithm in \cite{StahlbuhkSM21}. The results in Figure~\ref{fig:alg-compare} confirm this for both the queue length and its time-average. Though the queue lengths of all algorithms converge to the optimal benchmark (see the left of Figure~\ref{fig:alg-compare}), their transient performance varies significantly as suggested by the different bounds on $\colq$.
\begin{figure}
\centering
\includegraphics[width=5in]{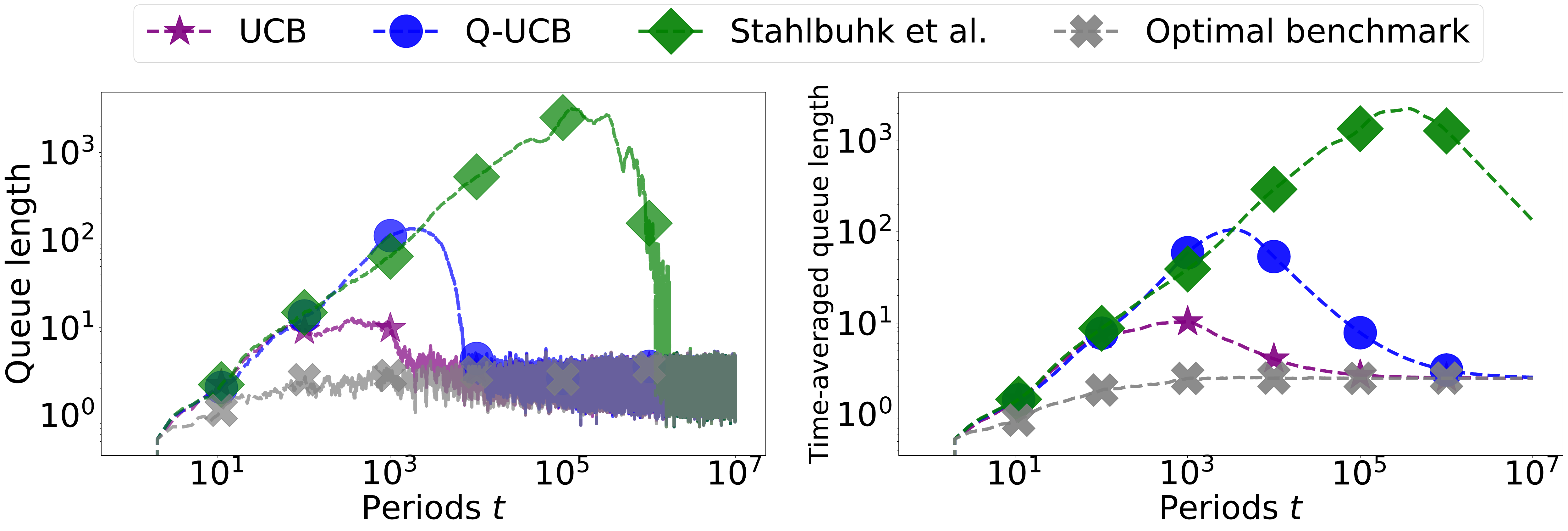}
\caption{Comparison between UCB, Q-UCB, and \cite{StahlbuhkSM21} algorithm in the setting of Figure~\ref{fig:col-motivation}}
\label{fig:alg-compare}
\end{figure}

\paragraph{Scaling with respect to $\varepsilon$ and $K$.} We next explore the scaling of the $\colq$ of UCB with respect to $\varepsilon$ and $K$. In  Theorem~\ref{thm:colq-ucb-single} we established that the $\colq$ of UCB is at most $O(\frac{K(\ln K+\ln(1/\varepsilon))}{\varepsilon})$ with a constant of around $400$ hidden in the big-O notation. Our theoretical lower bound $\Omega(\frac{K}{\varepsilon})$ on the $\colq$ of any policy (Theorem~\ref{thm:colq-lowerbound}) requires $K \geq 2^{14}$ and involves a constant of $2^{-14}$ hidden in the big-O notation. We now show that these constants, necessary for our analytical derivations, are actually conservative and the big-O scaling dependence -- with very small constants -- closely captures the actual $\colq$. Our numerics extend the setting in Figure~\ref{fig:col-motivation} and are close to the bad instance used in the proof of Theorem~\ref{thm:colq-lowerbound}. We focus on such settings to stress test the performance of considered learning algorithms. Specifically, for a pair $(K,\varepsilon)$, the arrival rate is $\lambda = 0.5 - \varepsilon / 2$ and there are~$K$ servers with service rates~$\mu_1 = \lambda / 10, \mu_j = \lambda - \varepsilon$ for~$j \in \{2,\ldots,K - 1\}$ and $\mu_K = \lambda + \varepsilon$.

Figure~\ref{fig:tclq_eps} displays the $\colq$ of UCB across exponentially decreasing $\varepsilon \in \{2^{-i}, i \in \{1,\ldots,9\}\}$ with $K = 5$ or $K = 10$. We find that it is consistently upper bounded by, and parallel to, the line $2K / \varepsilon$; this verifies the accuracy of the dependence on $1 / \varepsilon$ for the $\colq$ of UCB. Similarly, we display in Figure~\ref{fig:tclq_k}  the $\colq$ of UCB for exponentially increasing $K \in \{2^i, i \in \{2,\ldots, 9\}\}$ and~$\varepsilon \in \{0.1, 0.3\}$. Except for the point $K = 4$ and $\varepsilon = 0.1$, the $\colq$ of UCB is well sandwiched by an upper bound $2K\ln(K) / \varepsilon$ and a lower bound $K / (4\varepsilon)$. This demonstrates that a dependence of $K\ln(K)$ on the number of servers is accurate for the $\colq$ of UCB, and the lower bound result in Theorem~\ref{thm:colq-lowerbound} is valid even for small $K$. We highlight that the lower bound result holds for any policy whereas Figure~\ref{fig:tclq_k} only considers UCB. However, since we cannot enumerate all possible policies and since UCB vastly outperforms other existing algorithms with respect to  $\colq$ (Figure~\ref{fig:alg-compare}), it is a natural choice against which to test  the lower bound result.

\begin{figure}[hbtp]
\centering
\begin{minipage}{.45\textwidth}
\includegraphics[width=2.5in]{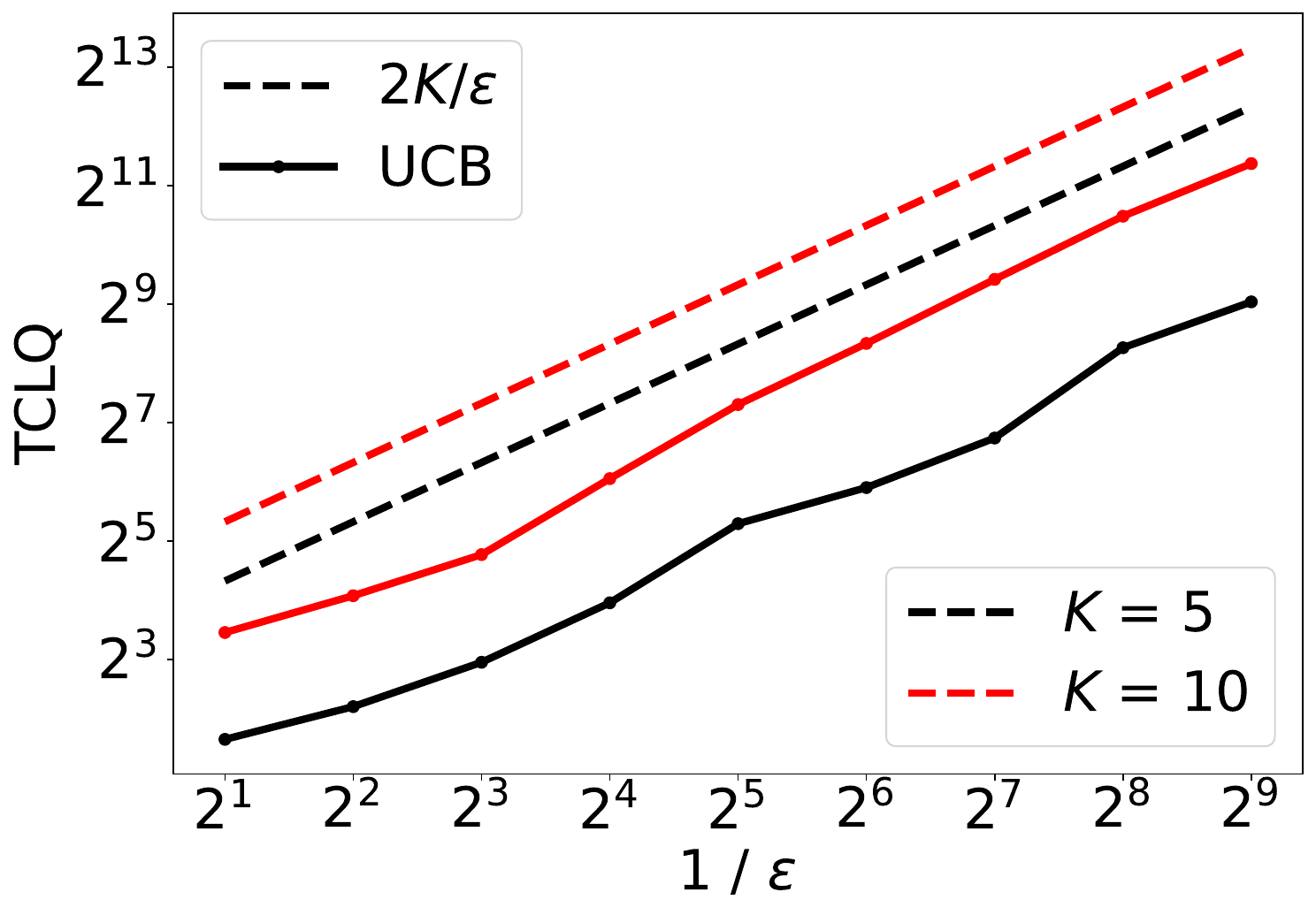}
\caption{Scaling with respect to $\varepsilon$}
\label{fig:tclq_eps}
\end{minipage}
\begin{minipage}{.45\textwidth}
\includegraphics[width=2.5in]{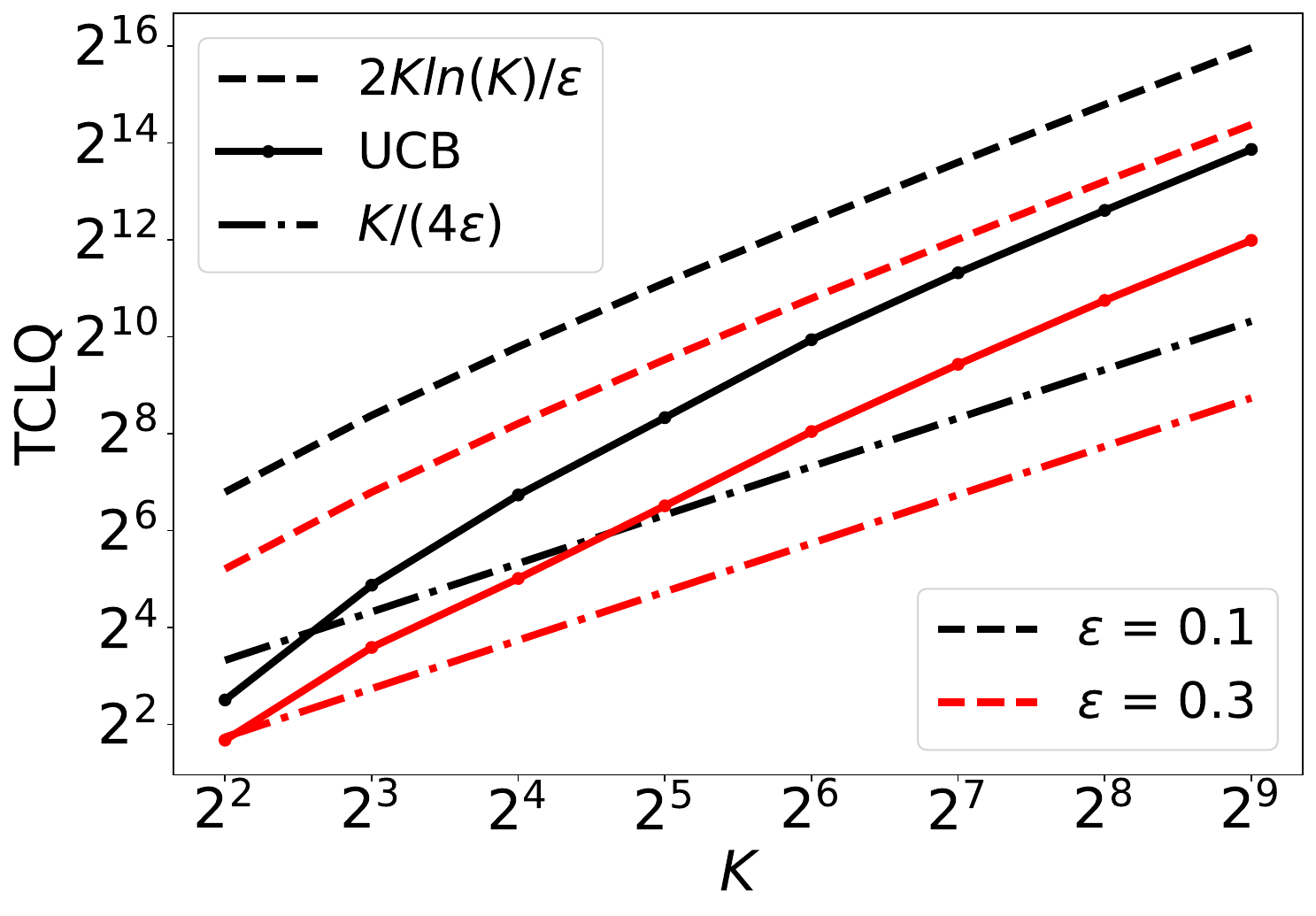}
\caption{Scaling with respect to $K$}
\label{fig:tclq_k}
\end{minipage}
\end{figure}

\subsection{Multi-queue multi-server setting}
Theorem~\ref{thm:mw-ucb} (Theorem~\ref{thm:bp-ucb-col}) shows that the $\colq$ of $\textsc{MW-UCB}$ ($\textsc{BP-UCB}$) has an optimal $\tilde{O}(1 / \varepsilon)$ scaling dependence on $\varepsilon$, but the constant in the bound is large. We now simulate the stationary multi-queue multi-server setting from \cite[appendix D]{YangSrikantYing} and show that the optimal scaling dependence requires only a small constant and translates into an improved transient performance.

Following \cite{YangSrikantYing}, we consider a setting with $N = 10$ queues and $M = 10$ workers, where worker~$j$ has service rates $\tilde{\mu}_{i,j}$ for jobs from queue $i$ with 
\[
\forall u,v \in \{0,1,2,3,4\},~\tilde{\mu}_{2u+1,2v+1}=0.9, \tilde{\mu}_{2u+1,2v+2}=0.6, \tilde{\mu}_{2u+2,2v+1} = 0.5, \tilde{\mu}_{2u+2,2v+2} = 1.
\]
In each period, each queue $i$ has a new job with probability $\lambda_i = 0.9 - \varepsilon$. As discussed by \cite{YangSrikantYing}, who choose~$\varepsilon=.15$, $\varepsilon$ captures the traffic slackness of this system; our results vary this quantity. The DM selects for each worker a queue to work on. If worker $j$ is assigned to queue $i$, a job from queue $i$ leaves with probability $\tilde{\mu}_{i,j}.$ The DM cannot assign more workers to queue $i$ than its current number of waiting jobs. This setting is a special case of our model in Section~\ref{sec:multi-model}; see Appendix~\ref{app:multi-example}. 

\begin{figure}
    \centering   \includegraphics[width=0.5\linewidth]{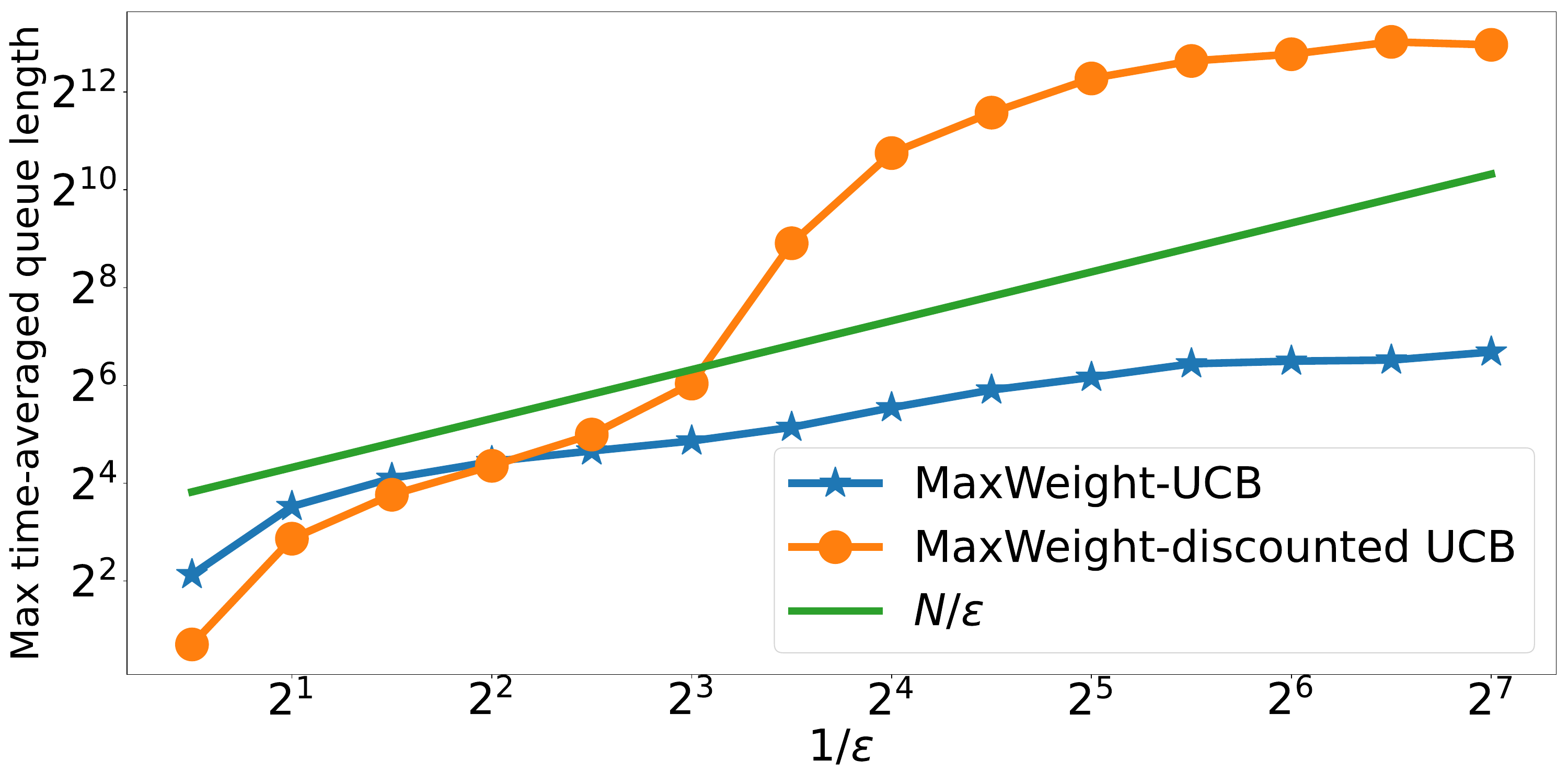}    \caption{Comparing the transient performance of MaxWeight-UCB (Algorithm~\ref{algo:mw-ucb}) with MaxWeight-Discounted UCB when the traffic slackness $\varepsilon$ varies.}
    \label{fig:net-compare}
\end{figure}

We implement $\textsc{MW-UCB}$ (Algorithm~\ref{algo:mw-ucb}) and MaxWeight-Discounted UCB (\textsc{MW-DUCB}) in \cite[algorithm 1]{YangSrikantYing}, whose hyper-parameters are set according to \cite[appendix D]{YangSrikantYing}. With $30$ independent runs of a time horizon $T = 10^6$, either algorithm ($\pi \in \{\textsc{MW-UCB}, \textsc{MW-DUCB}\}$) produces trajectories of queue lengths $\tilde{Q}_r(t, \pi)$ for the $r-$th run, period $t$. To measure their transient performance, we use the maximum time-averaged queue length, defined by $$\max_{t \leq T} \frac{1}{t}\cdot \frac{1}{30}\sum_{r\leq 30} \sum_{\tau \leq t} \tilde{Q}_r(\tau, \pi).$$  Figure~\ref{fig:net-compare} shows the maximum time-averaged queue length of $\textsc{MW-UCB}$ and $\textsc{MW-DUCB}$ as a function of the traffic slackness $\varepsilon$, which takes value in $\{2^{-(i+1)/2}, i \in \{0,\ldots,13\}\}.$ In Appendix~\ref{app:ysy}, we discuss that in the stationary setting of \cite{YangSrikantYing}, $\textsc{MW-UCB}$ has a $\colq$ bound of $\tilde{O}(\frac{N^{3.5}M^3}{\varepsilon})$ while the bound in \cite{YangSrikantYing} for $\textsc{MW-DUCB}$ translates into a $\colq$ bound of $\Theta(\frac{N^4M^4}{\varepsilon^3})$. Therefore, theory suggests that the former would have a better scaling dependence on $\varepsilon$. The figure displays this improvement in performance bounds in actual transient performance: $\textsc{MW-UCB}$ has better performance when $\varepsilon < 0.25$ and the gap between the two algorithms increases as $\varepsilon$ decreases.  Finally, although our $\colq$ bound involves a large constant in Theorem~\ref{thm:mw-ucb}, the figure shows that the maximum time-averaged queue length of $\textsc{MW-UCB}$ is consistently below $N / \varepsilon$.

\section{Conclusions}
Motivated by the observation that queueing regret is an asymptotic metric that does not capture the transient performance of learning algorithms in queueing systems, we propose an alternative metric, $\colq$. For a single-queue multi-server system with $K$ servers and a traffic slackness $\varepsilon$, we derive a lower bound $\Omega(\frac{K}{\varepsilon})$ on $\colq$, thus establishing that learning incurs a non-negligible increase in queue lengths. We then show that the classical UCB algorithm has a matching upper bound of $\tilde{O}(\frac{K}{\varepsilon})$. Finally, we extend our result to multi-queue multi-sever systems and general queueing networks by providing algorithms, \textsc{MaxWeight-UCB} and \textsc{BackPressure-UCB}, whose $\colq$ has a near optimal $\tilde{O}(1/\varepsilon)$ dependence on traffic slackness.

Having introduced a metric that captures the transient performance of learning in queueing systems, our work can serve as a starting point for interesting extensions that can help shed further light on the area. In particular, non-stationarity is an important consideration in queueing systems as job arrival patterns, server characteristics, and the feasible set of schedules can change over time. Although \cite{YangSrikantYing, nguyen2023learning} have explored similar settings, understanding the implication of $\colq$ and how our results would change due to non-stationarity is an interesting direction. Future research may also focus on beyond worst case guarantees for $\colq$, obtaining tight bounds with respect to the system's size, or improving bounds using contextual information, etc.

\bibliographystyle{alpha}
\bibliography{references}

\newpage
\appendix

\section{The transient cost of learning of previous algorithms (Section~\ref{sec:intro})}\label{app:related-work}
\subsection{The transient cost of learning for \cite{KrishnasamySJS21}}\label{app:qucb}
The bound in \cite{KrishnasamySJS21} implies a $\colq$ bound of at least $\Theta(\frac{K}{\varepsilon^2})$ for \textsc{Q-UCB}. This is the case because the bound in \cite[Proposition 2]{KrishnasamySJS21} requires $t \geq \frac{24K\log t}{\varepsilon^2}$ and there is no guarantee for $t$ smaller than that. As a result, the queue length can grow linearly for $t \leq \Theta(\frac{K}{\varepsilon^2})$, and the claim follows. Indeed, \cite[Proposition 2]{KrishnasamySJS21} further requires that $t \geq \exp(4/\Delta^2)$ where $\Delta$ is the minimum gap between service rates and thus their bound may scale as $\exp(1/\varepsilon^2)$ when $\Delta = 2\varepsilon$, which is the case for our hard instance for the $\colq$ lower bound in Section~\ref{app:proof-lowerbound}. The reason for such a high $\colq$ for \textsc{Q-UCB} is because their proof only utilizes samples from the forced exploration and disregards samples collected by the UCB exploration.

However, our analytical framework in Section~\ref{sec:optimal-single} indeed allows us to show a much better $\colq$ bound for \textsc{Q-UCB}, which is near optimal up to logarithmic factors (but higher order logarithmic term than that of \textsc{UCB} in Theorem~\ref{thm:colq-ucb-single}). For completeness, we give the algorithmic description of \textsc{Q-UCB} in Algorithm~\ref{algo:q-ucb} (note that in accordance to our model, we add Line~\ref{algoline:emptyqueue} to avoid service when there is no job). Our $\colq$ bound of \textsc{Q-UCB} is given in Proposition~\ref{prop:colq-qucb}. Combining this result with the $\tilde{O}(1/T)$ queueing regret of \textsc{Q-UCB} in \cite[Proposition 2]{KrishnasamySJS21} then shows \textsc{Q-UCB} exhibits near optimal performance for both early and late stages.

\begin{algorithm}[H]
\LinesNumbered
\DontPrintSemicolon
\caption{\textsc{Q-UCB} for a single-queue multi-server system 
\label{algo:q-ucb}
}
\SetKwInOut{Input}{input}\SetKwInOut{Output}{output}
Sample mean $\hat{\mu}_{k}(1) \gets 0$, number of samples $C_k(1) \gets 0$ for $k \in \set{K}\cup \{\perp\}$, queue $Q(1)\gets 0$\;
\For{$t = 1\ldots$}{
 \nl    Let $E(t)$ be a Bernoulli random variable with mean $\min(1,\frac{3K\log^2 t}{t})$\;
\nl     $\bar{\mu}_{k}(t) = 
    \min\left(1, \hat{\mu}_{k}(t) + \sqrt{\frac{2\ln(t)}{C_k(t)}}\right), \forall k \in \set{K}$\;
\nl     \textbf{if }{$E(t) = 1$ } \textbf{then }        $J(t) \gets \{1,\ldots,K\}$ uniformly at random \;
\nl \textbf{else }
    $J(t) \gets \arg\max_k \bar{\mu}_k(t)$ \;
 \nl    \textbf{if } {$Q(t) = 0$ } \textbf{then } {$J(t) \gets \perp$} \label{algoline:emptyqueue}

   \tcc{Update queue length \& estimates based on  $S_{J(t)}(t)$, $A(t)$, and $J(t)$}
   $Q(t+1) \gets Q(t) - S_{J(t)}(t) + A(t)$ \\
    $C_{J(t)}(t+1) \gets C_{J(t)}(t)+1, \quad \hat{\mu}_{J(t)}(t+1) \gets \frac{C_{J(t)}(t)\hat{\mu}_{J(t)}(t)+S_{J(t)}(t)}{C_{J(t)}(t+1)}$\\
    \textbf{for } $k\neq J(t)$ \textbf { set } $C_k(t+1)\gets C_k(t), \quad \hat{\mu}_k(t+1)\gets \hat{\mu}_k(t)$
}
\end{algorithm} 

\begin{proposition}\label{prop:colq-qucb}
For any $K \geq 1, \varepsilon \in (0,1]$, \[\colq^{\single}(K,\varepsilon,\textsc{Q-UCB}) \leq \frac{960K+64K(\ln K+2\ln(1/\varepsilon))}{\varepsilon} + 2^9\ln^3(K/\varepsilon) + 2^{22}K.\]
\end{proposition}
We use Lemmas~\ref{lem:single-queue-learn} and ~\ref{lem:single-queue-regen} to show Proposition~\ref{prop:colq-qucb}. To do so, we first need to bound the satisficing regret of \textsc{Q-UCB}. Compared with the bound of \textsc{UCB} (Lemma~\ref{lem:sar-ucb}), \textsc{Q-UCB} has higher order dependence on $\ln T$ due to its use of forced exploration. 
\begin{lemma}\label{lem:sar-qucb}
For any horizon $T$, we have
\begin{align*}
(i) \expect{\sar^{\single}(\textsc{Q-UCB},T)} &\leq \frac{16K(\ln T+2)}{\varepsilon}+3K\ln^3 T \\
(ii) \expect{\sar^{\single}(\textsc{Q-UCB},T)^2} &\leq \frac{2^{10}K^2(\ln T+2)^2}{\varepsilon^2}+72K^2\ln^6 T.
\end{align*}
\end{lemma}
The proof of this lemma considers the additional regret brought by forced exploration on top of the proof of Lemma~\ref{lem:sar-ucb}. The idea is similar, and we include it here for completeness.
\begin{proof}
Define an event $\tilde{\set{G}}_t = \{J(t) = \emptyset\} \cup \{\mu^\star - \mu_{J(t)}\leq 2\sqrt{\frac{2\ln t}{C_{J(t)}(t)}}\}\cup \{E(t)=1\}$ (recall from Algorithm \ref{algo:q-ucb} that $E(t)$ denotes whether the algorithm does forced exploration in round $t$). Following the same proof of Lemma~\ref{lem:ucb-conc} shows $\Pr\{\set{G}_t^c\} \leq 2Kt^{-3}$. Then
\begin{align*}
\expect{\sar^{\single}(\textsc{Q-UCB},T)} &\leq \expect{\sum_{t=1}^T (\mu^\star-\mu_{J(t)}-\frac{\varepsilon}{2})^+\indic{J(t) \neq \emptyset, E(t)=0,\tilde{\set{G}}_t}} + \expect{\sum_{t=1}^T E(t)} \\
&\hspace{0.2in}+ \expect{\sum_{t=1}^T \indic{\tilde{\set{G}}_t^c}}.
\end{align*}
The first term can be bounded for any sample path by $\frac{16K(\ln T+1)}{\varepsilon}$ as in \eqref{eq:sar-ucb-sample-upp} because it is the loss from doing UCB exploration; the second term is bounded by $\sum_{t=1}^T \frac{3K\ln^2 t}{t} \leq 3K\ln^3 t$ since $\expect{E(t)}=\frac{3K\ln^2 t}{t}$; the last term is bounded by $\sum_{t=1}^T 2Kt^{-3} \leq 4K$. As a result, 
\[
\expect{\sar^{\single}(\textsc{Q-UCB},T)} \leq \frac{16K(\ln T + 2)}{\varepsilon} + 3K\ln^3 T.
\]
Similarly, because $(a+b+c)^2 \leq 4(a^2+b^2+c^2)$, we have
\begin{align*}
\expect{\sar^{\single}(\textsc{Q-UCB},T)^2} &\leq 4\expect{\left(\sum_{t=1}^T (\mu^\star-\mu_{J(t)}-\frac{\varepsilon}{2})^+\indic{J(t) \neq \emptyset, E(t)=0,\tilde{\set{G}}_t}\right)^2} \\
&\hspace{0.1in}+
4\expect{\left(\sum_{t=1}^T E(t)\right)^2} +4\expect{\left(\sum_{t=1}^T \indic{\tilde{\set{G}}_t^c}\right)^2}.
\end{align*}
The first term is upper bounded by $4\left(\frac{16K(\ln T+1)}{\varepsilon}\right)^2$ by the sample path upper bound; the last term is upper bounded by $32K$ as in \eqref{eq:sar-ucb-sample-uppsqr}; the second term is upper bounded by
\begin{align*}
4\expect{\left(\sum_{t=1}^T E(t)\right)^2} &= 4\expect{\sum_{t=1}^T E(t)+\sum_{t_1 < t_2} E(t_1)E(t_2)} \\
&\overset{E(t_1) \text{independent of }E(t_2)}{\leq} 4\expect{\sum_{t=1}^T E(t)}\left(1+\expect{\sum_{t=1}^T E(t)}\right) \\
&\leq 72K^2\ln^6 T.
\end{align*}
Thus, 
\[
\expect{\sar^{\single}(\textsc{Q-UCB},T)^2} \leq \frac{2^{10}K^2(\ln T+1)^2}{\varepsilon^2}+72K^2\ln^6 T + 32K \leq \frac{2^{10}K^2(\ln T+2)^2}{\varepsilon^2}+72K^2\ln^6 T.
\]
\end{proof}
We are ready to prove Proposition~\ref{prop:colq-qucb} following a similar approach of the proof of Theorem~\ref{thm:colq-ucb-single}.
\begin{proof}[Proof of Proposition~\ref{prop:colq-qucb}]
Fix $K,\varepsilon$. Take $T_2 = \max\left(12^{19}, \left\lfloor\left(\frac{2^{13}K^2}{\varepsilon^4}\right)^2\right\rfloor\right)$. Note that $\ln(T_2) \leq 56+4\ln K+8\ln(1/\varepsilon)$. For $T \leq T_2$,
\begin{align*}
\frac{1}{T}\sum_{t\leq T}\expect{Q(t)} &\leq \frac{3}{\varepsilon}+\expect{\sar^{\single}(\textsc{Q-UCB},T)} \tag{Lemma~\ref{lem:single-queue-learn}} \\
&\leq \frac{3}{\varepsilon}+\frac{16K(\ln T_2 + 2)}{\varepsilon} + 3K\ln^3(T_2) \tag{Lemma~\ref{lem:sar-qucb} \emph{(i)} and $T \leq T_2$} \\
&\leq \frac{3}{\varepsilon}+\frac{16K(58+4\ln K + 8\ln(1/\varepsilon)}{\varepsilon}+3K(60+8\ln\nicefrac{K}{\varepsilon})^3 \\
&\leq \frac{960K+64K(\ln K+2\ln(1/\varepsilon))}{\varepsilon} + 12K((60)^3+8^3\ln^3(K/\varepsilon)) \tag{Fact~\ref{fact:generalized-mean}} \\
&\leq \frac{960K+64K(\ln K+2\ln(1/\varepsilon))}{\varepsilon} + 2^9\ln^3(K/\varepsilon) + 2^{22}K.
\end{align*}
For $T > T_2$,
\begin{align*}
\frac{1}{T}\sum_{t\leq T}\expect{Q(t)} &\leq \frac{4}{\varepsilon}+\frac{8}{\varepsilon^2}\frac{\expect{\sar^{\single}(\textsc{Q-UCB},T)^2}}{T} \tag{Lemma~\ref{lem:single-queue-regen}} \\
&\leq \frac{4}{\varepsilon}+\frac{8}{\varepsilon^2}\frac{2^{10}K^2(\ln T+2)^2/\varepsilon^2 + 72K^2\ln^6 T}{T} \tag{Lemma~\ref{lem:sar-qucb} \emph{(ii)}} \\
&\leq \frac{4}{\varepsilon}+16 \leq \frac{20}{\varepsilon}\tag{By definition of $T_2$ and Fact~\ref{fact:lnt-sqrt-prop}}.
\end{align*}
Combining the above two cases concludes the proof.
\end{proof}

\subsection{The transient cost of learning for \cite{StahlbuhkSM21}}\label{app:ssm}
The setting in \cite{StahlbuhkSM21} that resembles our single-queue setting is \cite[Theorem 4]{StahlbuhkSM21}, which gives a bound on $\sum_{t\leq T} Q(t)-Q^{\star}(t)$ of order $1/\varepsilon^8$. Since $\frac{\sum_{t\leq T} Q(t)-Q^{\star}(t)}{T} \leq T$, this implies a $\colq$ bound of $O(1/\varepsilon^4)$. We next discuss how to derive the bound on $\sum_{t\leq T} Q(t)-Q^{\star}(t)$. In the proof of Theorem 4 on page 1770 in \cite{StahlbuhkSM21}, the authors show in Eq. (16) and thereafter that 
\[
R^{\pi_3}_{(\lambda,\mu)}(T) \leq \sum_{p=1}^{p_0-1} (M_2p^2+\beta_2)+\sum_{p=p_0}^T (M_2p^2+\beta_2)M_0e^{-\chi p}
\]
where $R^{\pi_3}_{(\lambda,\mu)}(T) = \sum_{t \leq T} (Q^{\pi_3}(t) - Q^{\star}(t))$ and $\pi_3$ is their policy, described in \cite[Fig. 7]{StahlbuhkSM21}. The constants $M_0,\chi$ are from \cite[Lemma 6]{StahlbuhkSM21} and $M_2,\beta_2$ are from \cite[Lemma 8]{StahlbuhkSM21}. We obtain lower bounds of these constants from the corresponding proofs as follows.

For $M_0, \chi$, the last inequality (after Eq. (28)) in the proof of their Lemma 6 (pp. 1776) requires $M_0e^{-\chi p}$ to be at least $e^{-\frac{1}{2}p\delta^2} + 2(N-1)e^{-2\frac{p}{N}\gamma^2}$. Here $N$ is the number of servers (i.e., $K$ in our paper), $\gamma$ is the minimum service rate suboptimality gap, and $\delta$ is at most the expected decrease in queue lengths by choosing the fastest server, which is at most the traffic slackness $\varepsilon$ in our paper.

For $M_2p^2 + \beta_2$, Eq. (34) in the proof of their Lemma 8 on page 1778 requires $M_2 \geq 0$ and $\beta_2 \geq 2\sum_{n=0}^{\infty} n^2e^{-c_3 n}$ and $c_3$ is a constant from their Lemma 10. Checking the proof of Lemma 10 on page 1780, after Eq. (41), one can see that $c_3$ is at least $2\delta^2$ and $\delta$ is at most the traffic slackness $\varepsilon$ in our paper (see the definition of $\delta$ after their Eq. (40)). Therefore, $\beta_2 \geq \sum_{n=0}^{\infty} n^2e^{-\varepsilon^2 n} = \Theta(\frac{1}{\varepsilon^6})$ since $\int_0^{\infty} x^2e^{-\alpha x} dx = \frac{2}{\alpha^3}$ for any $\alpha > 0$.

Combining these constants, the upper bound in \cite[Theorem 4]{StahlbuhkSM21} is at least (for $T \geq \frac{1}{\varepsilon^2}$)
\[
\sum_{p=1}^{p_0-1} (M_2p^2+\beta_2)+\sum_{p=p_0}^T (M_2p^2+\beta_2)M_0e^{-\chi p}  \geq \Theta\left(\sum_{p=1}^{T} \frac{e^{-\frac{1}{2}p\varepsilon^2}}{\varepsilon^6}\right) \geq \Theta\left(\frac{(1-e^{-\frac{1}{2}T\varepsilon^2})}{\varepsilon^8}\right)\geq \Theta\left(\frac{1}{\varepsilon^8}\right).
\]

\subsection{The transient cost of learning for \cite{NeelyRP12}}\label{app:nrp}
The model in \cite{NeelyRP12} captures our single-queue multi-server setting by considering (in their notations) $N=1$ queue, fixed events $\beta(t)$ across time (so no information), action $k(t)$ corresponding the server chosen, and random events $\omega(t)$ is Bernoulli distributed with mean $\mu_{k(t)}$ where $\mu_k$'s are initially unknown. There is no utility to consider and we are focused on the queue length. 
Apply \cite[Theorems~1,2]{NeelyRP12} to our single queue setting. In their notations, $\varepsilon_{\max}$ is equal to our slackness $\varepsilon$, and under the single-queue setting, we have $\tilde{K} = K$. To satisfy the condition that $\varepsilon_Q < \varepsilon_{\max}$ in \cite[Theorem~1]{NeelyRP12} by their guarantee in Theorem~2, one needs to set a window size $W=\Theta(\frac{K^2}{\varepsilon^2})$, which set their constant $C = \Theta(\frac{K^4}{\varepsilon^2})$. Putting $C$ back into their Theorem~1 shows that the limited time-averaged queue length scales at least $O(\frac{K^4}{\varepsilon^3}).$

\subsection{The transient cost of learning for \cite{YangSrikantYing}}\label{app:ysy}
For a stationary multi-server system with $N$ queues and $M$ workers, the guarantee in \cite[Theorem 1]{YangSrikantYing} can be translated into a $\colq$ bound of $\Theta(\frac{N^4M^4}{\varepsilon^3})$. To see this, on the top right of page 6 in \cite{YangSrikantYing}, the authors derive an any-time time-averaged queue length bound of $O(I^2J^2g(\gamma)/\delta_{\max})$, where $I,J$ correspond to $N,M$ here, $\delta_{\max}$ is equal to $\varepsilon$ in our paper, and $g(\gamma)$ is at least $I^2J^2/\delta^2_{\max}$ by the requirement in \cite[Theorem 1]{YangSrikantYing}. Putting these together gives the $\colq$ bound. Note that our guarantee for \textsc{MW-UCB} in Theorem~\ref{thm:mw-ucb} corresponds to a $\colq$ bound of $\tilde{O}(\frac{N^{3.5}M^3}{\varepsilon})$; see the discussion in Appendix~\ref{app:examples}.

\subsection{The transient cost of learning for \cite{FreundLW22}}\label{app:flw}
For a bipartite queueing system with $N$ agents and $M$ workers, \cite[Theorem~3]{FreundLW22} shows a queue length bound of $\tilde{O}(\frac{M^3}{\varepsilon^3} + \frac{(ML_{\text{epoch})^2}}{\varepsilon^2 T})$ for the expected time-averaged queue length of horizon $T$ (note that queue lengths are weighted by arrival rates in their bound). By Eq. (2) in \cite{FreundLW22}, $L_{\text{epoch}} = \tilde{O}(\frac{M^2}{\varepsilon^2})$. As a result, their queue length bound gives a $\colq$ bound of $\tilde{\Theta}(\frac{M^3}{\varepsilon^3}+\frac{M^6}{\varepsilon^6 T}).$ Since the time-averaged queue length is at most $O(T)$, one can get a $\colq$ bound of $\tilde{O}(\frac{M^3}{\varepsilon^3})$. Our Theorem~\ref{thm:mw-ucb} gives a $\colq$ bound of $\tilde{O}(\frac{N^{1.5}M^3}{\varepsilon})$ by the mapping we provide in Appendix~\ref{app:examples}.

\subsection{The transient cost of learning for \cite{nguyen2023learning}}\label{app:nm}
\cite[Corollary~1]{nguyen2023learning} gives a queue length bound for their algorithm. The proof of this Corollary in \cite[Appendix~E]{nguyen2023learning} shows that the bound is about $\tilde{O}(\tau)$ where $\tau$ is a parameter such that $\varepsilon \gtrapprox \mu_{\min}\tau^{-1/3}$ in the stationary setting and $\mu_{\min}$ is $\min_{k \in \set{K}} \mu_k$. Therefore, their queue length bound at best shows a $\colq$ of $\tilde{O}(\frac{1}{(\varepsilon \min_{k \in \set{K}} \mu_k)^3})$.

\section{Additional details for modeling examples}\label{app:examples}
\subsection{Generality of our multi-queue multi-server model (Section~\ref{sec:multi-model})}\label{app:multi-example}
We discuss in detail how our multi-queue multi-server systems capture the settings in \cite{FreundLW22,YangSrikantYing} and beyond. 

In a bipartite queueing system~\cite{FreundLW22}, there are $N$ agents and $M$ workers. In period $t$, a new job arrives to each agent $n$’s queue with probability $\lambda_n$. The DM selects a matching $\bolds{\sigma} = (\sigma_{n,m})_{n \in [N], m \in [M]}$ between agents and workers such that $\sum_{m’\in[M]} \sigma_{n,m’}\leq 1,\sum_{n’\in [N]}$ for all $ n \in [N]$ and $\sigma_{n’,m} \leq 1$ for all $m\in [M]$. If $\sigma_{n,m}=1$, then the first job in agent $n$’s queue (if any) is cleared with probability $\mu_{n,m}$. Otherwise, the job stays in the queue. All arrivals and services are independent across queues, servers and periods. To translate this model into our ]multi-queue multi-server formulation, let $\mathcal{N} = [N], \mathcal{K} = [N] \times [M] = \{(n,m), n\in [N],m\in [M]\}$. $\Lambda$ corresponds to $N$ independent Bernoulli random variables. For each $k=(n,m) \in \mathcal{K}$, we have $\mu_k = \mu_{n,m}$. The set $\mathcal{A} = \{0,1\}^{\mathcal{N}}$.
$\Sigma$ is the set of subsets $\boldsymbol{\sigma}$ of $\mathcal{K}$, such that no two queues can be served simultaneously by the same server, i.e., $\sum_{k=(n,m’)} \sigma_k \leq 1, \sum_{k=(n’,m)} \sigma_k \leq 1, \forall n\in[N],m\in[M]$. $\mathcal{B}_n = \{k=(n,m), m \in [M]\}$. 
 
The multi-server system \cite{YangSrikantYing} is similarly defined but instead of selecting a matching between agents and workers, the DM can match an agent with multiple workers, as long as there are sufficiently many jobs assigned to different workers. Therefore, the only change compared with the bipartite queueing system would be that $\Sigma$ consists of all $\boldsymbol{\sigma}$ such that $\sum_{k=(n’,m)} \sigma_k \leq 1, \forall m \in [M]$.

Our model allows us to capture additional constraints that can arise in a multi-queue multi-server system. One possible extension, in the context of the above setting in \cite{YangSrikantYing}, is that workers cannot take on jobs from two particular queues because jobs from the two queues require the same external resources. Our feasible set of schedules can incorporate such interference by having $\Sigma$ remove any schedule that selects servers from these two queues simultaneously. 

\subsection{Generality of our queueing network model (Section~\ref{sec:network-model})}\label{app:examples-network}
Our queueing network model captures the setting in \cite{tassiulas1992stability}, which consists of a directed
graph $G = (\set{V},\set{E})$ with a set of nodes $\set{V}$, a set of links $\set{E}$,  a set of job classes $\set{J}$, and an activation set $\set{S}$.
\begin{enumerate}
\item Each node $v \in \set{V}$ maintains a backlog and in each period, a random number of class-$j$ jobs arrive to the backlog of node $v$.
\item A link $i\in\set{E}$ goes from a source node $q(i)$ to a sink node $h(i)$. 
\item A class $j\in\set{J}$ has a set of destination nodes $\set{V}_j$. If a class-$j$ job is routed to a node in $\set{V}_j$, it leaves the system.
\item For each period $t$, the DM makes two decisions. First, the DM activates a set of links. Second, for each activated link $i$, the DM selects a single job from the backlog of the source node $q(i)$ to serve. Let $e_{i,j}(t) \in \{0,1\}$ be the indicator on whether the DM activates link $i$ and select a class-$j$ job to serve for that link in period $t$. \label{item:network-decision}
\item The DM has three constraints.
\begin{enumerate}
    \item [(a)] the set of links DM activates must be an element from the activation set $\set{S}$;
    \item [(b)] each activated link serves at most one job;
    \item [(c)] for any node $v$, the DM cannot select more class-$j$ jobs to serve than the number of class-$j$ jobs, $Q_{v,j}(t)$, in the backlog, i.e., $\sum_{i:q(i) = \ell} e_{i,j}(t) \leq Q_{v,j}(t)$.
\end{enumerate}
\label{item:network-constraint}
\item If $e_{i,j}(t) = 1$, service is successful with probability $m_i$. In this event, the served job moves to the backlog of the sink node $h(i).$ If $h(i)$ is a destination node of class $j$, this job leaves the system.
\end{enumerate}
The model above can be instantiated in our model from Section~\ref{sec:network-model}. For each pair of node $v \in \set{V}$ and class $j \in \set{J}$, we create a queue labeled by $(v,j)$, corresponding to the class-$j$ jobs in node $v$. For each tuple $(v,j,i)$ where the source node of link $i \in \set{E}$ is $v$, we create a server $(v,j,i)$ dedicated to queue $(v,j)$. The service rate of server $(v,j,i)$ is equal to $m_i$, which is the probability of successful service of link $i$. In each period, the DM selects a set of servers $\{(v_k, j_k, i_k)\}_k$, which corresponds to the decision $\{e_{i,j}(t)\}$ in item \ref{item:network-decision} above. This set of servers must be an element of the set of feasible schedules $\bolds{\Sigma}_t$ which abides by the three constraints of item \ref{item:network-constraint} above: 
\begin{enumerate}
\item[(a)] the set of selected links $\{i_k\}$ must be an element in the activation set $\set{S}$; 
\item[(b)] selected links are different: $i_k \neq i_{k'}$ for $k \neq k'$ (so no two jobs are served by the same link). 
\item[(c)] no more servers are selected than the number of jobs in a queue.
\end{enumerate}
The final component is job transitions. If server $(v,j,i)$ has a successful service, then a job in queue $(v,j)$ transitions to queue $(h(i), j)$. If the sink node $h(i)$ is a destination node of class $j$, the job transitions to the virtual queue $\perp$, i.e., it leaves the system. The above discussion shows how our model captures the one in \cite{tassiulas1992stability} by defining suitable sets of queues, servers, and feasible schedules.

\section{Supplementary materials of Section~\ref{sec:model-single}}\label{app:model-single}
\subsection{Bound on job-averaged time in system}\label{app:clq-clw}
Our results focus on TCLQ, the maximum increase in expected time-averaged queue length; in this section we show that upper bounding TCLQ is essentially equivalent to upper bounding the maximum increase in expected \emph{job-averaged time in the system}. Therefore, a bound on TCLQ is desirable as it connects to the degradation in job (customer) experience of a service system. Such a connection is not surprising by Little's Law; we formally derive the connection in \Cref{prop:clw}. Our focus is on the single-queue multi-server setting from Section~\ref{sec:model-single}, but the derivation generalizes to other queueing systems as long as certain concentration bound holds for the arrival process.

We first introduce some notation: suppose that jobs are indexed by labels $1,2,3,\ldots$ according to the order of their arrivals. The arrival time of job $i$ is $\textsc{Arr}_i$ and its departure time is $\textsc{Dep}_i$, which may depend on the policy $\pi$. For a period $t$, let $W_i(t,\pi)$ be the time that job $i$ has spent in the system until period $t$, that is, $W_i(t,\pi) = \min(\textsc{Dep}_i, t) - \textsc{Arr}_i.$ Let $N(t)$ be the number of arrivals before period $t$. There are then two quantities of interest over the first $T$ periods:
\begin{itemize}
\item Time-averaged queue length: $\bar{Q}(T,\pi) \coloneqq \frac{1}{T}\sum_{t=1}^T Q(t,\pi)$;
\item Job-averaged time in system: $\bar{W}(T,\pi) \coloneqq \frac{1}{\max(N(T),1)}\sum_{i=1}^{N(T)} W_i(T, \pi)$.
\end{itemize} 
Our paper focuses on the transient cost of learning in queueing (TCLQ) metric defined by the maximum increase in expected time-averaged queue length compared to the optimal policy~$\pi^\star$:
\[
\colq(\pi) = \max_{T \geq 1} \expect{\bar{Q}(T,\pi) - \bar{Q}(T,\pi^\star)}
\]
A similar metric, which we call the transient wait cost of learning in queueing (TWCL), may accordingly be defined by the maximum increase in expected job-averaged time in system:
\[
\textsc{TWCL}(\pi) = \max_{T \geq 1} \expect{\bar{W}(T, \pi) - \bar{W}(T,\pi^\star)}.
\]
The following result implies that, up to constant factors depending on the arrival probability $\lambda$, an upper bound for any of those two metrics yields an upper bound for the other metric.
\begin{proposition}\label{prop:clw}
The $\textsc{TCLQ}$ and $\textsc{TWCL}$ metrics can be bounded as a function of each other by:
\begin{itemize}
\item $\textsc{TWCL}(\pi) \leq \max\left(100,\frac{16\sqrt{2}}{\lambda^3},\frac{2}{\lambda}\colq(\pi)+1\right)$;
\item $\colq(\pi) \leq \max\left(100,\frac{16\sqrt{2}}{\lambda^3},1.5\lambda \textsc{TWCL}(\pi)+1\right)$.
\end{itemize}
\end{proposition}
As long as $\lambda$ is not too small and the system is heavily loaded, the dominant terms in Proposition~\ref{prop:clw} imply $\textsc{TWCL}(\pi) \leq \frac{2}{\lambda}\colq(\pi)+1$ and $\colq(\pi) \leq 2\lambda \textsc{TWCL}(\pi) + 1$. This matches our intuition from Little's Law: the long-run time-averaged queue length is equal to the product of the arrival rate and the long-run job-averaged wait time, i.e., $\bar{Q}(T,\pi) = \lambda \bar{W}(T,\pi)$ as $T \to \infty.$ Proposition~\ref{prop:clw} incurs a loss, in the form of the constant factor $2$, compared to this intuition because it looks at the maximum increase rather than the long-run average.

Our proof of Proposition~\ref{prop:clw} relies on a coupling argument similar to that in Section~\ref{sec:single-queue-learn}.  Specifically, for each period $t$ there is a  random variable $U(t)$ uniform in $[0,1]$ such that the service of server $k$ is generated by $S_k(t) = \indic{U(t) \leq \mu_k}.$ Since a policy $\pi$ selects at most one server in a period, the coupling does not change the distribution for the time-averaged queue length $\bar{Q}(T,\pi)$ and job-averaged time in system $\bar{W}(T,\pi)$.
\begin{proof}[Proof of Proposition~\ref{prop:clw}]
Fix a time horizon $T$. The sample path version of Little's Law (\cite[Theorem LL.2]{little2011or}) shows
$N(T)\bar{W}(T,\pi') = T\bar{Q}(T,\pi')$ for any policy $\pi'$. Applying this equation to both $\pi$ and $\pi^\star$ gives
\begin{equation}\label{eq:wait-queue}
N(T)\left(\bar{W}(T,\pi) - \bar{W}(T,\pi^\star)\right) = T\left(\bar{Q}(T,\pi) - \bar{Q}(T,\pi^\star)\right).
\end{equation}
We make the following case distinction based on the magnitude fo $T$:
\begin{itemize}
\item $T \leq \max(100,\frac{16\sqrt{2}}{\lambda^3})$: We use that $T$ is an upper bound on $\bar{Q}(T,\pi)$ and $\bar{W}(T,\pi)$; this holds because there is at most one arrival in each period for the single queue setting and the time in system for a job is at most $T$. As a result, both $\bar{Q}(T,\pi) - \bar{Q}(T,\pi^\star)$ and $\bar{W}(T,\pi) - \bar{W}(T,\pi^\star)$ are upper bounded by $\max(100,\frac{16\sqrt{2}}{\lambda^3})$;
\item $T > \max(100,\frac{16\sqrt{2}}{\lambda^3})$:
By Hoeffding's Inequality (Fact~\ref{fact:hoeffding}), the event \[\set{E}_T = \left\{N(T) \in \left[\lambda T - \sqrt{2T\ln T}, \lambda T + \sqrt{2T\ln T}\right]\right\}\] happens with probability $1 - 2T^{-4}.$ Moreover, $T \geq 100$ and Fact~\ref{fact:lnt-sqrt-prop}(iii) imply that $\ln(T) \leq T^{1/3}$, which implies $\sqrt{2T\ln T} \leq \sqrt{2T^{4/3}} \leq \sqrt{2}T^{2/3} \leq \lambda T / 2$ where the last inequality follows from $T \geq \frac{16\sqrt{2}}{\lambda^3}.$ As a result, the event $\set{E}_T$ implies $N(T) \in [0.5\lambda T, 1.5\lambda T].$ Conditioning on $\set{E}_T$ and applying \eqref{eq:wait-queue} gives:
\begin{align*}
\bar{W}(T,\pi) - \bar{W}(T,\pi^\star) = \frac{T}{N(T)}(\bar{Q}(T,\pi) - \bar{Q}(T,\pi^\star)) &\leq \frac{T}{0.5\lambda T}(\bar{Q}(T,\pi) - \bar{Q}(T,\pi^\star)) \\
&= \frac{2}{\lambda}(\bar{Q}(T,\pi) - \bar{Q}(T,\pi^\star))
\end{align*}
\begin{align*}
\text{ and } \bar{Q}(T,\pi) - \bar{Q}(T,\pi^\star) = \frac{N(T)}{T}(\bar{W}(T,\pi) - \bar{W}(T,\pi^\star)) &\leq \frac{1.5\lambda T}{T}(\bar{W}(T,\pi) - \bar{W}(T,\pi^\star)) \\
&= 1.5\lambda (\bar{W}(T,\pi) - \bar{W}(T,\pi^\star)),
\end{align*}
where the two inequalities rely on the fact that $\bar{Q}(T,\pi) \geq \bar{Q}(T,\pi^\star)$ under the above-mentioned coupling (see e.g. the proof of \cite[Lemma 1]{KrishnasamySJS21}). Since $\bar{W}(\pi, T) \leq T, \Pr\{\set{E}_T\} \geq 1 - 2T^{-4}$ and $T \geq 100$, we then have 
\begin{align*}
\expect{\bar{W}(T,\pi)-\bar{W}(T,\pi^\star)} &= \expect{\bar{W}(T,\pi)-\bar{W}(T,\pi^\star) \mid \set{E}_T}\Pr\{\set{E}_T\} + T\Pr\{\set{E}_T^c\} \\
&\hspace{-0.3in}\leq \frac{2}{\lambda}\expect{\bar{Q}(T,\pi) - \bar{Q}(T,\pi^\star) \mid \set{E}_T}\Pr\{\set{E}_T\} + 1 \leq \frac{2}{\lambda}\expect{\bar{Q}(T,\pi) - \bar{Q}(T,\pi^\star)} + 1, 
\end{align*}
where we again use that $\bar{Q}(T,\pi) \geq \bar{Q}(T,\pi^\star).$ We similarly obtain $$\expect{\bar{Q}(T,\pi) - \bar{Q}(T,\pi^\star)} \leq 1.5\lambda\expect{\bar{W}(T,\pi) - \bar{W}(T,\pi^\star)} + 1.$$
\end{itemize}
Summarizing the above two cases shows that for any $T$, 
\[
\expect{\bar{W}(T,\pi)-\bar{W}(T,\pi^\star)} \leq \max\left(100,\frac{16\sqrt{2}}{\lambda^3},\frac{2}{\lambda}\expect{\bar{Q}(T,\pi) - \bar{Q}(T,\pi^\star)} + 1\right), \text{ and}
\]
\[
\expect{\bar{Q}(T,\pi) - \bar{Q}(T,\pi^\star)} \leq \max\left(100,\frac{16\sqrt{2}}{\lambda^3},1.5\lambda\expect{\bar{W}(T,\pi) - \bar{W}(T,\pi^\star)} + 1\right).
\]
Taking the maximum over $T$ on both sides for both inequalities gives the desired result.
\end{proof}

\subsection{Statistical inference of traffic slackness  (Remark~\ref{remark:slackness})}\label{app:slackness-inference}
Before designing an optimal learning algorithm, a valid question the DM may have is whether the system is even stabilizable. Specifically, given a constant $\kappa > 0$, how can the DM know if the traffic slackness of the system is below $\kappa$, indicating more resources are needed? We next show how our theory on $\colq$ can help address such a statistical inference problem by identifying a procedure of hypothesis testing and analyzing its significance level and power. Although we restrict to a single-queue multi-server system, the discussion is also valid for the multi-queue multi-server and the queueing network settings.

Let the null hypothesis $H_0$ be that the traffic slackness $\varepsilon$ is greater or equal to $\kappa$ and the alternative hypothesis $H_1$ be that $\varepsilon < \kappa$. Rejecting the null hypothesis implies that the DM can (reasonably) infer that the system is under-provisioned. Suppose the desired significance level is $\alpha \in (0,1),$ i.e., the probability of rejecting $H_0$ is no larger than $\alpha$ if $H_0$ is indeed true.

To conduct the hypothesis test, the DM runs some learning algorithm $\pi$ (e.g., \textsc{UCB}). Assume that this algorithm enjoys a $\colq$ bound of $\theta(\varepsilon)$ when the system is $\varepsilon$. Then if $H_0$ is true ($\varepsilon \geq \kappa$), combining the $\colq$ bound and Lemma~\ref{lem:bound-optimal-queue} shows that the expected time-averaged queue length $\bar{Q}(T,\pi)$ is upper bounded by $\theta(\kappa) + \frac{\lambda}{\kappa}+\frac{1}{2}$ for any horizon $T$. For ease of exposition, suppose that $\theta(\kappa) \geq \frac{\lambda}{\kappa} + \frac{1}{2}$; otherwise we can adjust its definition. To design our test statistics, let $T_1$ be a pre-specified number of periods the DM runs the algorithm $\pi$, who thus observes a sample path of the queue length $Q(1,\pi),\ldots,Q(T_1,\pi).$ Let the test statistic be the observed time-averaged queue length $\bar{Q}(T_1,\pi).$ The DM rejects the null hypothesis if and only if
\begin{equation}\label{eq:test}
\bar{Q}(T_1,\pi) \geq \frac{2 \theta(\kappa)}{\alpha}.
\end{equation}
The following result identifies the significance level (probability of  rejecting $H_0$ when $H_0$ is 
\emph{true}) and the power (probability of rejecting $H_0$ when $H_0$ is \emph{false}) of this test.
\begin{proposition}\label{prop:test}
Suppose $\theta(\kappa) \geq \frac{\lambda}{\kappa}+\frac{1}{2}$. The test \eqref{eq:test} has significance level $\alpha$ and power $1 - \beta$ with
\[
\beta \leq \Pr_{H_1}\left\{\bar{Q}(T_1,\pi^\star) < \frac{2\theta(\kappa)}{\alpha}\right\}.
\]
\end{proposition}
The above result shows that the power of the test is at least $\Pr_{H_1}\left\{\bar{Q}(T_1,\pi^\star) \geq \frac{2\theta(\kappa)}{\alpha}\right\}$, which is the probability that the queue length under the optimal policy will have a queue length no smaller than $\frac{2\theta(\kappa)}{\alpha}$ when the traffic slackness is $\varepsilon$. Although we do not formally lower bound this probability, we provide some intuition on why it is approximately $(1-\varepsilon)^{2\theta(\kappa)/\alpha}$ when $T_1$ is large enough. Suppose $T_1$ is large enough so that the distribution of $\bar{Q}(T_1,\pi^\star)$ converges to its stationary distribution. Since its stationary distribution is geometrically distributed with rate $\frac{\lambda(1-\mu)}{(1-\lambda)\mu}$ (see e.g., \cite[example 3.3.5]{Srikant_Ying_2014}), \[\Pr_{H_1}\left\{\bar{Q}(T_1,\pi^\star) \geq \frac{2\theta(\kappa)}{\alpha}\right\} \approx \Pr_{H_1}\left\{\bar{Q}(\infty,\pi^\star) \geq \frac{2\theta(\kappa)}{\alpha}\right\} \approx \left(\frac{\lambda(1-\mu)}{(1-\lambda)\mu}\right)^{\frac{2\theta(\kappa)}{\alpha}} = \left(1 - \frac{\varepsilon}{\mu - \lambda \mu}\right)^{\frac{2\theta(\kappa)}{\alpha}}.
\]
This result demonstrates that for the proposed hypothesis testing \eqref{eq:test} to have non-zero power, the traffic slackness $\varepsilon$ should satisfy $\varepsilon < (\mu - \lambda \mu) \frac{2\alpha}{\theta(\kappa)}.$ This highlights the importance of having an efficient learning algorithm with low $\colq$ bound $\theta(\kappa)$, which enables the inference for a larger range of $\varepsilon$.  
\begin{proof}[Proof of Proposition~\ref{prop:test}]
This test has significance level $\alpha$ because by Markov's Inequality, 
\[
\Pr_{H_0}\left\{\bar{Q}(T_1,\pi) \geq  \frac{2\theta(\kappa)}{\alpha}\right\} \leq \frac{\expect{\bar{Q}(T_1,\pi)}}{\frac{2\theta(\kappa)}{\alpha}} \leq \frac{\left(\theta(\kappa) + \frac{\lambda}{\kappa}+\frac{1}{2}\right)}{\frac{2\theta(\kappa)}{\alpha}} \leq \alpha,
\]
where we use $\Pr_{H_0}\{\cdot\}$ to denote the probability of an event when the hypothesis $H_0$ is true and the assumption that $\theta(\kappa) \geq \frac{\lambda}{\kappa}+\frac{1}{2}$.

Understanding the power of this test, i.e., the probability of correctly rejecting the null hypothesis when $H_1$ is true, is more difficult. The goal is to lower bound the probability $\Pr_{H_1}\left\{\bar{Q}(T_1,\pi) \geq \frac{2\theta(\kappa)}{\alpha}\right\}$, which is $1 - \beta$ with $\beta = \Pr_{H_1}\left\{\bar{Q}(T_1,\pi) < \frac{2\theta(\kappa)}{\alpha}\right\}$ being the Type-II error of the test. Similar to the coupling argument in the proof of Proposition~\ref{prop:clw}, the time-averaged queue length under $\pi$ stochastically dominates that under $\pi^\star$. As a result,
\[
\beta =  \Pr_{H_1}\left\{\bar{Q}(T_1,\pi) < \frac{2\theta(\kappa)}{\alpha}\right\} \leq \Pr_{H_1}\left\{\bar{Q}(T_1,\pi^\star) < \frac{2\theta(\kappa)}{\alpha}\right\}.
\]
\end{proof}

\subsection{Optimality of $\textsc{UCB}$ when the traffic slackness is negative (Remark~\ref{remark:slackness})}\label{app:slackness-negative} 
Consider a single-queue multi-server system with a strictly negative traffic slackness $\varepsilon$. We argue that this setting is almost identical to the canonical multi-armed bandit setting. For a policy $\pi$, we define its regret by $\regret(\pi, T) = \expect{\sum_{t=1}^T (\mu^\star - \mu_{J(t)})}.$ Different from Sections~\ref{sec:model-single} and \ref{sec:optimal-single}, we no longer require the chosen server $J(t) = 0$ when $Q(\pi, t) = 0$. This allows the policy to behave like any multi-armed bandit policy. The below result shows that the difference in expected queue length between policy $\pi$ and the optimal policy is equal to its regret up to a constant.

\begin{proposition}\label{prop:regret-negative}
 Given a negative traffic slackness $\varepsilon < 0$, for any period $T$ and policy $\pi$, 
 \[
 \regret(\pi,T) -  \frac{2}{1 - \exp(-\varepsilon^2 / 2)} \leq \expect{Q(\pi,T+1) - Q(\pi^\star, T+1)} \leq \regret(\pi, T).
 \]
 \end{proposition}
 Since the regret must scale as $\Omega(\ln(T))$ for any policy by the Lai-Robbins lower bound, Proposition~\ref{prop:regret-negative} shows that the difference in expected queue length between any policy $\pi$ and the optimal policy must also scale as $\Omega(\ln(T))$ when the traffic slackness is negative. Moreover, optimizing this difference is (orderly) equivalent to optimizing the regret of the policy. Since \textsc{UCB} has the optimal $O(\ln (T))$ scaling \cite{AuerCF02}, we show that $\textsc{UCB}$ is also an optimal learning algorithm for the queueing setting when the traffic slackness is negative.  
 \begin{proof}[Proof of Proposition~\ref{prop:regret-negative}]
We apply the coupling in the proof of Proposition~\ref{prop:clw} such that for each period $t$ there is a uniform random variable $U(t) \in [0,1]$ and the random service satisfies $S_k(t) = \indic{U(t) \leq \mu_k}$ for any server $k$. This coupling does not change the distribution of $Q(\pi, T+1)$ and $Q(\pi^\star, T+1)$, but allows $Q(\pi, T+1) \geq Q(\pi^\star,T+1)$ for any sample path and $S_k(t) \geq S_{j^\star}(t)$ for any server $k$ and the optimal server $j^\star$. 

Since the queue length of a policy $\pi$ in period $T+1$ is equal to 
\begin{equation}\label{eq:queue-dynamic}
Q(\pi, T+1) = \sum_{t=1}^T A(t) - S_{J(t)}(t)\indic{Q(\pi, t) > 0},
\end{equation}
the suboptimality in queue length of policy $\pi$ satisfies
\begin{equation}\label{eq:queue-diff}
Q(\pi, T + 1) - Q(\pi^\star, T+1) = \sum_{t=1}^T S_{j^\star}(t)\indic{Q(\pi^\star,t) > 0} - S_{J(t)}(t)\indic{Q(\pi, t) > 0}.
\end{equation}
Noting that $Q(\pi^\star,t) > 0$ implies $Q(\pi, t) > 0$ due to the coupling and thus $\indic{Q(\pi^\star, t) > 0} \leq \indic{Q(\pi, t) > 0}$ for any $t$, the above equation gives
\begin{align*}
Q(\pi, T + 1) - Q(\pi^\star, T+1) &\leq \sum_{t=1}^T  S_{j^\star}(t)\indic{Q(\pi, t) > 0} - S_{J(t)}(t)\indic{Q(\pi, t) > 0} \\
&= \sum_{t=1}^T  (S_{j^\star}(t)- S_{J(t)}(t))\indic{Q(\pi, t) > 0} \\
&\leq \sum_{t=1}^T  S_{j^\star}(t)- S_{J(t)}(t),
\end{align*}
where the last inequality uses $S_{j^\star}(t) \geq S_{J(t)}(t)$ from the coupling. Taking expectation thus gives 
\[
\expect{Q(\pi, T + 1) - Q(\pi^\star, T+1)} \leq \expect{\sum_{t=1}^T  S_{j^\star}(t)- S_{J(t)}(t)} = \expect{\sum_{t=1}^T \mu^\star - \mu_{J(t)}} = \regret(\pi, T).
\]
For the other direction, \eqref{eq:queue-diff} implies
 \begin{align*}
 Q(\pi, T+1) - Q(\pi^\star, T+1) &\geq \sum_{t=1}^T S_{j^\star}(t)\indic{Q(\pi^\star,t) > 0} - S_{j(t)}(t) \\
 &\geq \sum_{t=1}^T (S_{j^\star}(t) - S_{j(t)}(t)) - \sum_{t=1}^T \indic{Q(\pi^\star,t) = 0}.
 \end{align*}
Taking expectation on both sides gives
\begin{equation}\label{eq:queue-diff-lower}
\expect{Q(\pi, T+1) - Q(\pi^\star, T+1)} \geq \regret(\pi, T) - \sum_{t=1}^T \Pr\{Q(\pi^\star,t) = 0\}.
\end{equation}
For $Q(\pi^\star,t) = 0$, it is necessary that $\sum_{\tau=1}^{t-1} A(\tau) - S_{j^\star}(\tau) < 0$ by applying \eqref{eq:queue-dynamic} at $t$. Therefore,
\[
\Pr\{Q(\pi^\star,t) = 0\} \leq \Pr\left\{\sum_{\tau=1}^{t-1} A(\tau) - S_{j^\star}(\tau) < 0\right\} \leq 2\exp(-\varepsilon^2 (t-1) / 2),
\]
where the last inequality uses the Hoeffding's Inequality (Fact~\ref{fact:hoeffding}) and the fact that $A(\tau) - S_{j^\star}(\tau) \in [-1,1]$ and has mean $-\varepsilon$. Plugging it back to \eqref{eq:queue-diff-lower} gives
\begin{align*}
\expect{Q(\pi, T+1) - Q(\pi^\star, T+1)} &\geq \regret(\pi, T) - \sum_{t=1}^T 2\exp(-\varepsilon^2 (t-1) / 2) \\
&\geq \regret(\pi, T) - \frac{2}{1 - \exp(-\varepsilon^2 / 2)}.
\end{align*}
\end{proof}

\section{Supplementary materials of Section~\ref{sec:optimal-single}}\label{app:single}
\subsection{Proof of the lower bound (Theorem~\ref{thm:colq-lowerbound})}\label{app:proof-lowerbound}
Consider $K$ sequences of service rates $\{\bolds{\mu}^k=(\mu_j^k)_{j \in \set{K}}\}_{k \in \set{K}}$ such that $\mu_j^k = \frac{1}{2} - \trafficslack$ for $j \neq k$ and $\mu_{k}^k = \frac{1}{2}+\trafficslack$. The arrival rate is $\lambda = \frac{1}{2}$. Fix a policy $\pi$. Let $\Pr^{k}\{\cdot\}$ and $\mathbb{E}^k[\cdot]$ denote the probability distribution and the expectation when the service rate is set as $\bolds{\mu}^k$. Our goal is to show that for a particular $T$, we have 
$\frac{1}{K}\sum_{k \in \set{K}}\frac{\sum_{t=1}^T \left(\mathbb{E}^k[Q(t)] - \mathbb{E}^k[Q^\star(t)]\right)}{T} = \Omega(\frac{K}{\trafficslack})$ and thus $\colq^{\single}(\lambda,\bolds{\mu}^k,\pi) = \Omega(\frac{K}{\trafficslack})$ for at least one $k \in \set{K}$. This then implies $\colq^{\single}(K,\trafficslack,\pi) = \Omega(\frac{K}{\trafficslack})$ since the traffic slackness of the system is $\trafficslack$ for every $\bolds{\mu}^k$. For every $t$, define the arrival sequence $\bolds{A}_t = (A(1),\ldots,A(t))$ and service sequence $\bolds{S}_t = (S_{J(1)}(1),\ldots,S_{J(t)}(t))$. Without loss of generality, we assume that $\pi$ is a deterministic function such that for every $t$, it maps $(\bolds{A}_{t-1},\bolds{S}_{t-1})$ to a server from $\set{K} \cup \{\perp\}$. The case where $\pi$ allows randomization does not affect the result by Fubini's Theorem (see Section 3.3 of \cite{BubeckC12}). We denote the history $(\bolds{A}_t,\bolds{S}_t)$ by $\bolds{H}_t$; the selected server $J(t)$ is a function of $\bolds{H}_{t-1}$ since $\pi$ is a deterministic policy.

For any horizon $T$, let $N_k(T)$ be the random variable of the number of times the policy selects server $k$ in the first $T$ periods. Define a sequence of uniform service rates $\bolds{\mu}^{\mathrm{unif}}$ where $\bolds{\mu}^{\mathrm{unif}}_j = \frac{1}{2}-\trafficslack$ for $j \in \set{K}$ and $\Pr^{\mathrm{unif}}\{\cdot\}$ and $\mathbb{E}^{\mathrm{unif}}[\cdot]$ are defined similarly as before. In the bandit setting, it was shown that $\mathbb{E}^k[N_k(T)]$ is not much greater than $\mathbb{E}^{\mathrm{unif}}[N_k(T)]$ \cite{AuerCFS02}. The following lemma shows that this is also the case for the considered queueing system.
\begin{lemma}\label{lem:link-count}
For any $k \in \set{K}, \trafficslack \leq \frac{1}{4}, T \geq 1$, we have $\mathbb{E}^k[N_k(T)] \leq \mathbb{E}^{\mathrm{unif}}[N_k(T)] + 4\trafficslack T\sqrt{\mathbb{E}^{\mathrm{unif}}[N_k(T)]}$.
\end{lemma}
\begin{proof}
The proof is similar to that of Lemma A.1 in \cite{AuerCFS02} and is based on KL divergence. We provide a detailed proof here because we use a different setting of service rates and because of the presence of arrival events. Fix $k,\trafficslack, T$. For ease of notations, we denote $\bolds{A} = \bolds{A}_T, \bolds{S}=\bolds{S}_T$ and let $\bolds{H} = (\bolds{A},\bolds{S})$ denote the history of horizon $T$. The realization of $\bolds{H}$ is denoted by $\bolds{h} \in \{0,1\}^{2T}$. Then for any distribution $G,P$ on $\{0,1\}^{2T}$, we denote $\tv{G}{P}$ to be the total variation between $G,P$ defined by the maximum of $|G(\set{E})-P(\set{E})|$ for all measurable event $\set{E}$. In addition, we use $\kl{G}{P}$ for the KL divergence between $G,P$ defined by $\kl{G}{P} = \sum_{\bolds{h} \in \{0,1\}^{2T}} G(\bolds{h})\ln\left(\frac{G(\bolds{h})}{P(\bolds{h})}\right)$. For any $t \leq T$, let $G(\bolds{h}_t)$ denote the marginal distributions over $\bolds{H}_t$. The conditional KL divergence between them conditioned on the history $\bolds{h}_{t-1}$, $\kl{G(\bolds{h}_t | \bolds{h}_{t-1})}{P(\bolds{h}_t | \bolds{h}_{t-1})}$ is defined as $\sum_{\bolds{h}_{t} \in \{0,1\}^{2t}} G(\bolds{h}_t)\ln\frac{G(\bolds{h}_t | \bolds{h}_{t-1})}{P(\bolds{h}_t | \bolds{h}_{t-1})}$ with $\bolds{h}_{t-1}$ being the first $t-1$ entries of $\bolds{h}_t$. 

To bound the difference between $\mathbb{E}^k[N_k(T)]$ and $\mathbb{E}^{\mathrm{unif}}[N_k(T)]$, note that $N_k(T)$ must be a deterministic function of the history $\bolds{H}$ when the policy $\pi$ is fixed. Write it as a mapping $N_k \colon \{0,1\}^{2T} \to [0,T]$. Then 
\begin{equation}\label{eq:link-count}
\begin{aligned}
\mathbb{E}^k[N_k(T)]-\mathbb{E}^{\mathrm{unif}}[N_k(T)] = \sum_{\bolds{h} \in \{0,1\}^{2T}} N_k(\bolds{h})\left(\Pr^k\{\bolds{h}\} - \Pr^{\mathrm{unif}}\{\bolds{h}\}\right) &\leq T\tv{\Pr^{\mathrm{unif}}}{\Pr^k} \\
&\leq T\sqrt{\frac{\kl{\Pr^{\mathrm{unif}}}{\Pr^k}}{2}}
\end{aligned}
\end{equation}
where the last inequality is by Pinsker's Inequality (Fact \ref{prop:pinsker}).

By the chain rule of KL divergence (Fact~\ref{prop:chain}) and the independence between the history up to period $t-1$ and the arrivals and services in period $t$, 
\begin{align*}
\kl{\Pr^{\mathrm{unif}}}{\Pr^k}&=\sum_{t=1}^T \kl{\Pr^{\mathrm{unif}}(\bolds{h}_t | \bolds{h}_{t-1})}{\Pr^{k}(\bolds{h}_t | \bolds{h}_{t-1})} \\
&= \sum_{t=1}^T \sum_{\bolds{h}_t \in \{0,1\}^{2t}} \Pr^{\mathrm{unif}}(\bolds{h}_t)\ln\left(\frac{\Pr^{\mathrm{unif}}(\bolds{h}_t | \bolds{h}_{t-1})}{\Pr^k(\bolds{h}_t | \bolds{h}_{t-1})}\right) \\
&= \sum_{t=1}^T \sum\limits_{\substack{\bolds{h}_{t-1} \in \{0,1\}^{2(t-1)} \\ a,s \in \{0,1\}, j \in \set{K}}} \Pr^{\mathrm{unif}}(\bolds{h}_{t-1})\Pr^{\mathrm{unif}}(A(t)=a,J(t)=j,S_j(t)=s | \bolds{h}_{t-1}) \\
&\mspace{32mu}\cdot\ln\frac{\Pr^{\mathrm{unif}}(A(t)=a,J(t)=j,S_j(t)=s | \bolds{h}_{t-1})}{\Pr^k(A(t)=a,J(t)=j,S_j(t)=s | \bolds{h}_{t-1})}.
\end{align*}
When $J(t) \neq k$, the probability of event $(A(t) = a,S_j(t) = s)$ is the same for both $\Pr^{\mathrm{unif}}$ and $\Pr^k$ condition on $\bolds{h}_{t-1}$. As a result, the last term becomes zero. Letting $\kl{g}{p}$ for the KL divergence between two Bernoulli distributions with mean $g, p \in [0,1]$ respectively, we then have
\begin{align}
\kl{\Pr^{\mathrm{unif}}}{\Pr^k}&=\sum_{t=1}^T \Pr^{\mathrm{unif}}(J(t) = k)\kl{\nicefrac{1}{2}-\varepsilon}{\nicefrac{1}{2}+\varepsilon} \nonumber\\
&= \mathbb{E}^{\mathrm{unif}}[N_k(T)]\kl{\nicefrac{1}{2}-\varepsilon}{\nicefrac{1}{2}+\varepsilon} \nonumber\\
&\leq \mathbb{E}^{\mathrm{unif}}[N_k(T)]\frac{4\varepsilon^2}{(\nicefrac{1}{2}+\varepsilon)(\nicefrac{1}{2}-\varepsilon)} = \frac{64}{3}\varepsilon^2\mathbb{E}^{\mathrm{unif}}[N_k(T)], \label{eq:bound-dkl}
\end{align}
where the last inequality is because of Fact~\ref{prop:kl-bound} and the assumption that $\varepsilon \leq 0.25$. Combining \eqref{eq:bound-dkl} with \eqref{eq:link-count}, we have
\[
\mathbb{E}^k[N_k(T)]-\mathbb{E}^{\mathrm{unif}}[N_k(T)] \leq T\sqrt{\frac{32}{3}\varepsilon^2\mathbb{E}^{\mathrm{unif}}[N_k(T)]} \leq 4\varepsilon T\sqrt{\mathbb{E}^{\mathrm{unif}}[N_k(T)]}.
\]
\end{proof}
The following lemma gives an upper bound on the time-averaged queue length for the optimal policy over any horizon. For ease of notations, we use $Q^{\star}(t) = Q(t,\pi^{\star})$.
\begin{lemma}\label{lem:bound-optimal-queue}
For any $\lambda,\varepsilon \in (0,1]$ and $T \geq 1$, it holds that $\frac{1}{T}\sum_{t=1}^T \expect{Q^\star(t)} \leq \frac{\lambda}{\varepsilon}+\frac{1}{2}$.
\end{lemma}
\begin{proof}
For any period $t$, consider the Lyapunov function $V^\star(t) = (Q^\star(t))^2$. We can rewrite the queueing dynamic \eqref{eq:dynamic-single} by $Q^\star(t+1) = \max(0,Q^\star(t) - S_{k^\star}(t)) + A(t)$. As a result,
\begin{align*}
\expect{V^\star(t+1) - V^\star(t)} &= \expect{\left(\max(0,Q^\star(t) - S_{k^\star}(t)) + A(t)\right)^2 - (Q^\star(t))^2} \\
&\leq \expect{ A^2(t)+S_{k^\star}^2(t) + 2(A(t)-S_{k^\star}(t))Q^\star(t)} \\
&= \lambda + \lambda + \varepsilon - 2\varepsilon \expect{Q^\star(t)} \\
&= 2\lambda+\varepsilon - 2\varepsilon\expect{Q^\star(t)}
\end{align*}
where the first inequality can be verified by considering $Q^\star(t) = 0$ or $Q^\star(t) \geq 1$. Then for any $T \geq 1$, since $Q(1) = 0$, we have
\[
0\leq \expect{V^\star(T+1)} = \sum_{t=1}^T \expect{V^\star(t+1) - V^\star(t)} \leq (2\lambda+\varepsilon)T-2\varepsilon\sum_{t=1}^T\expect{Q^\star(t)},
\]
which gives $\frac{1}{T}\sum_{t=1}^T \expect{Q^\star(t)} \leq \frac{\lambda}{\varepsilon} + \frac{1}{2}$.
\end{proof}
We are ready to prove Theorem~\ref{thm:colq-lowerbound}.
\begin{proof}[Proof of Theorem~\ref{thm:colq-lowerbound}]
For a setting of service rates $\bolds{\mu}^k$, the number of total serviced jobs in time horizon $T$ is $\mathbb{E}^k[\sum_{t=1}^T \mu^{k}_{J(t)}] = (\nicefrac{1}{2}+\varepsilon)\mathbb{E}^k[N_k(T)] + (\nicefrac{1}{2}-\varepsilon)(T-\mathbb{E}^k[N_k(T)]) = (\nicefrac{1}{2}-\varepsilon)T + 2\varepsilon \mathbb{E}^k[N_k(T)]$. By Lemma~\ref{lem:link-count}, we have 
\begin{equation}\label{eq:upperbound-service}
\begin{aligned}
\sum_{k \in \set{K}} \mathbb{E}^k\left[\sum_{t=1}^T \mu^{k}_{J(t)}\right] &= \sum_{k \in \set{K}} \left((\nicefrac{1}{2}-\varepsilon)T + 2\varepsilon \mathbb{E}^k[N_k(T)]\right) \\ &\leq (\nicefrac{1}{2}-\varepsilon)KT + 2\varepsilon\sum_{k \in \set{K}}\left(\mathbb{E}^{\mathrm{unif}}[N_k(T)] + 4\varepsilon T\sqrt{\mathbb{E}^{\mathrm{unif}}[N_k(T)]}\right) \\
&\leq (\nicefrac{1}{2}-\varepsilon)KT + 2\varepsilon T + 8\varepsilon^2 T\sqrt{KT}
\end{aligned}
\end{equation}
where the last inequality is by $\sum_{k \in \set{K}} \mathbb{E}^{\mathrm{unif}}[N_k(T)] \leq T$ and Cauchy-Schwarz inequality. Let $T_1 = \lceil \frac{K}{(32\varepsilon)^2} \rceil$. For any $T \leq T_1$, by the queueing dynamic \eqref{eq:dynamic-single}, 
\begin{equation}\label{eq:lowerbound-queue}
\begin{aligned}
\frac{1}{K}\sum_{k \in \set{K}}\mathbb{E}^k[Q(T)] \geq \frac{1}{K}\sum_{k \in \set{K}} \left(\lambda T - \mathbb{E}^k\left[\sum_{t=1}^T \mu^{k}_{J(t)}\right]\right) &\overset{\eqref{eq:upperbound-service}}{\geq} \frac{T}{2} - \frac{1}{K}\left((\nicefrac{1}{2}-\varepsilon)KT + 2\varepsilon T + 8\varepsilon^2 T\sqrt{KT}\right) \\
&\geq \frac{\varepsilon T}{2} - 8\varepsilon^2 T \sqrt{\frac{T}{K}} \geq \frac{\varepsilon T}{4}
\end{aligned}
\end{equation}
where the second inequality is because $K \geq 4$ and the last inequality is because $8\varepsilon^2T\sqrt{\frac{T}{K}} \leq \frac{\varepsilon T}{4}$ for any $T \leq \frac{K}{(32\varepsilon)^2}$. Therefore, for any feasible policy $\pi$,
\begin{align*}
\colq^{\single}(K,\varepsilon,\pi) \geq \frac{1}{K}\sum_{k \in \set{K}} \colq^{\single}(\lambda,\bolds{\mu}^k,\pi) &\geq \frac{1}{K}\left(\frac{1}{T_1}\sum_{T=1}^{T_1}\left(\mathbb{E}^k[Q(T)] - \mathbb{E}^k[Q^\star(T)]\right)\right) \\
&\geq \frac{\varepsilon T_1}{8} - \frac{1}{2\varepsilon}-\frac{1}{2} \geq \frac{\varepsilon}{8}\left(\frac{K}{(32\varepsilon)^2}\right)-\frac{1}{2\varepsilon}-1\geq \frac{K}{2^{14}\varepsilon}
\end{align*}
where the third inequality is by Lemma~\ref{lem:bound-optimal-queue} and \eqref{eq:lowerbound-queue}; the forth inequality is by the setting of $T_1$ and $\varepsilon \leq \frac{1}{4}$; and the last inequality is by $K \geq 2^{14}$.
\end{proof}

\subsection{Comparison with the lower bound in \cite{KrishnasamySJS21} (Remark~\ref{remark:compare})}\label{app:lowerbound-compare}
Let $\mu_{\min} = \min_k \mu_k, \Delta = \min_{k\colon \mu_k \neq \mu^\star} \mu^\star - \mu_k$ and $D(\bolds{\mu}) = \frac{\Delta}{\kl{\mu_{\min}}{\frac{\mu^\star + 1}{2}}}$. For any $\alpha-$consistent policy $\pi$ (see Definition 1 in \cite{KrishnasamySJS21}), \cite[Proposition 3]{KrishnasamySJS21} establishes the following lower bound.
\begin{proposition}
Given any single-queue system $(\lambda,\bolds{\mu})$ and any $\alpha-$consistent policy $\pi$ and $\gamma > \frac{1}{1-\alpha}$, there exists constant $\tau, C_1$ independent of $\lambda,\bolds{\mu}$ such that for $\eta = \frac{(K-1)D(\bolds{\mu})}{2\max(C_1K^{\gamma},\tau)}$, if $\varepsilon < \eta$, then 
\begin{equation}\label{eq:ksjs-bound}
\expect{Q(t,\pi) - Q^\star(t)} \geq \frac{D(\bolds{\mu})}{2}(K-1)\frac{\log t}{\log \log t}
\end{equation}
for $t \in [\max(C_1K^{\gamma}, \tau), (K-1)\frac{D(\bolds{\mu})}{2\varepsilon}]$.
\end{proposition}
Applying the lower bound \eqref{eq:ksjs-bound} gives a lower bound on $\colq$ that $\colq^{\single}(\lambda,\bolds{\mu},\pi) \geq \tilde{\Omega}(K\ln(1/\varepsilon)D(\bolds{\mu})).$ To compare the tightness of \eqref{eq:ksjs-bound} with our bound (Theorem~\ref{thm:colq-lowerbound}), it remains to compare $D(\bolds{\mu})$ with $O(\frac{1}{\varepsilon})$. The following result shows that $D(\bolds{\mu}) \leq 9$ and thus the bound from \cite{KrishnasamySJS21} is weaker than ours when $\varepsilon$ is sufficiently small.
\begin{proposition}
For any $\bolds{\mu}$, we have $D(\bolds{\mu}) \leq 9.$
\end{proposition}
\begin{proof}
Note that if $\Delta = 0$, we simply have $D(\bolds{\mu}) = 0$; so let us focus on $\Delta \in (0,\mu^\star]$. In addition, we have $\mu_{\min} \leq \mu^\star - \Delta \leq \frac{\mu^\star+1}{2}$ and thus $D(\bolds{\mu}) \leq \frac{\Delta}{\kl{\mu^\star - \Delta}{\frac{\mu^\star+1}{2}}}$. We next consider the minimum value of $\kl{\mu^\star - \Delta}{\frac{\mu^\star+1}{2}}$ when $\Delta$ is fixed but $\mu^\star$ is allowed to change. 

Define the triangular discrimination $D_{\triangle}(g \| p)$ between two Bernoulli random variables of mean $g,p$ respectively by \[D_{\triangle}(g \| p) \coloneqq \frac{(g-p)^2}{g+p}+\frac{(1-g-(1-p))^2}{1-g+1-p}.\]
Let $q = \frac{\mu^\star-1}{2}$. A refined Pinsker's inequality (see \cite[Eq.(16)]{FosterK21} and \cite{Topsoe00}) shows 
\begin{align}
\kl{\mu^\star - \Delta}{\frac{\mu^\star+1}{2}} &\geq \frac{1}{2}D_{\triangle}(\mu^\star - \Delta \| \frac{\mu^\star+1}{2}) \nonumber\\
&= \frac{1}{2}D_{\triangle}(\mu^\star - \Delta \| \mu^\star - q) = \frac{(\Delta-q)^2}{(2+3q-\Delta)(\Delta-3q)}.\label{eq:connect-kl-triangle}
\end{align}
Since $1 \geq \mu^\star \geq \Delta$, we require $q \in [\frac{\Delta-1}{2},0]$. We next consider the minimum value of $f(q) \coloneqq \frac{(\Delta-q)^2}{(2+3q-\Delta)(\Delta-3q)}$ for $q \in [\frac{\Delta-1}{2},0]$. Taking derivative of $f$ gives
\[
f'(q) = \frac{(-12\Delta-6)q^2+(16\Delta^2+4\Delta)q-4\Delta^3+2\Delta^2}{(2+3q-\Delta)^2(\Delta-3q)^2}.
\]
Solving $f'(q^\star) = 0$ gives $q^\star = \frac{4\Delta^2+\Delta \pm (2\Delta^2+2\Delta)}{6\Delta+3}$. There are two cases:
\begin{itemize}
\item $\Delta > 0.5$. Then zero points of $f'$ are all positive and thus $f'(q)$ is negative for $q \leq 0$, implying that $f(q)$ is minimized at $q = 0$ over $[\frac{\Delta-1}{2},0]$. For this case, $f(q) \geq \frac{\Delta^2}{(2-\Delta)\Delta} \geq \frac{\Delta}{2}$;
\item $\Delta \leq 0.5$. There is one zero point $q^\star = \frac{2\Delta^2-\Delta}{6\Delta+3}$ for $f'(q)$ over $q \leq 0$. One can verify that $q^\star \in [\frac{\Delta-1}{2},0]$. Therefore, over $[\frac{\Delta-1}{2},0]$, 
\[
f(q) \geq f(q^\star)=\frac{(4\Delta(\Delta+1)/(6\Delta+3))^2}{(2-2\Delta/(2\Delta+1))(2\Delta / (2\Delta+1)} \geq \frac{16\Delta(\Delta+1)^2}{36(2\Delta+1)} \geq \frac{\Delta}{9}.
\]
\end{itemize}
Combining the above two cases and by \eqref{eq:connect-kl-triangle}, we have $\kl{\mu^\star-\Delta}{\frac{\mu^*+1}{2}} \geq \min_{q \in [\frac{\Delta-1}{2},0] f(q)} \geq \frac{\Delta}{9}$. Therefore, $D(\bolds{\mu}) \leq \frac{\Delta}{\kl{\mu^\star-\Delta}{\frac{\mu^*+1}{2}}} \leq 9$.
\end{proof}

\section{Supplementary materials of Section~\ref{sec:optimal-multi}}\label{app:multi}
\subsection{Satisficing regret of \textsc{MW-UCB} (Lemma~\ref{lem:sar-mw-ucb})} \label{app:lem-sar-mw-ucb}

To analyze the satisficing regret, define $\Delta_k(t) = \sqrt{\frac{2\ln t}{C_k(t)}}$ for the confidence bound of server $k$ in period $t$ as in \textsc{MW-UCB} (Algorithm~\ref{algo:mw-ucb}). This can be viewed as the loss of choosing server~$k$; we next show that, with high probability, the loss in selecting $\bolds{\sigma}(t)$, $\Delta(t)$, is upper bounded by the sum of $\Delta_k(t)$ over selected servers. Formally, define the good event for period $t$ by $\set{G}_t = \left\{\Delta(t) \leq 2\sum_{k \in \set{K}}\sigma_k(t) \Delta_k(t)\right\}$. The following result extends Lemma~\ref{lem:ucb-conc} (which is for \textsc{UCB} in single-queue cases) for \textsc{MW-UCB} and obtains a lower bound on the probability of good events.
\begin{lemma}\label{lem:multi-good-event}
For every period $t$ and under $\textsc{MW-UCB}$, we have $\Pr\{\set{G}_t\} \geq 1 - 2Kt^{-3}$.
\end{lemma}
\begin{proof}
Fix a period $t$. Recall that when $\|\bolds{Q}(t)\|_{\infty} > 0$, we define $\Delta(t) = \frac{W_{\bolds{\sigma}^{\MW}(t)}(t) - W_{\bolds{\sigma}(t)}(t)}{\|\bolds{Q}\|_{\infty}}$. For a server $k$, recall that $\hat{\mu}_k(t)$ is the sample mean of its service rates by period $t$. Using Hoeffding's Inequality (Fact~\ref{fact:hoeffding}) and a union bound over $C_k(t) \in \{0,\ldots,t-1\}$ gives $\Pr\{\left|\mu_k - \hat{\mu}_k(t)\right| \geq \Delta_k(t)\} \leq 2t^{-3}$. Then by union bound over all servers, with probability at least $1-2Kt^{-3}$, we have for every $k \in \set{K}$ that $\left|\mu_k-\hat{\mu}_k(t)\right| \leq \Delta_k(t)$. Denote this event by $\set{E}$. Conditioned on $\set{E}$, we have $\bar{\mu}_k(t) - 2\Delta_k(t) \leq \mu_k \leq \bar{\mu}_k(t)$ by the definition of upper confidence bound $\bar{\mu}_k(t)$ in Algorithm~\ref{algo:mw-ucb}. In addition, conditioned on $\set{E}$, the weight of $\bolds{\sigma}(t)$ is lower bounded by
\begin{align*}
W_{\bolds{\sigma}(t)}(t) = \sum_{n \in \set{N}} Q_n(t)\sum_{k \in \set{B}_n} \sigma_k(t)\mu_k &\overset{(a)}{\geq} \sum_{n\in \set{N}} Q_n(t)\sum_{k \in \set{B}_n} \sigma_k(t)\left(\bar{\mu}_k(t)-2\Delta_k(t)\right) \\
&= \sum_{n\in \set{N}} Q_n(t)\sum_{k \in \set{B}_n} \sigma_k(t)\bar{\mu}_k(t) - 2\sum_{n \in \set{N}} Q_n(t)\sum_{k \in \set{B}_n} \sigma_k(t)\Delta_k(t) \\
&\overset{(b)}{\geq} \sum_{n\in \set{N}} Q_n(t)\sum_{k \in \set{B}_n} \sigma^{\MW}_k(t)\bar{\mu}_k(t) - 2\sum_{n \in \set{N}} Q_n(t)\sum_{k \in \set{B}_n} \sigma_k(t)\Delta_k(t) \\
&\overset{(c)}{\geq} \sum_{n\in \set{N}} Q_n(t)\sum_{k \in \set{B}_n} \sigma^{\MW}_k(t)\mu_k - 2\sum_{n \in \set{N}} Q_n(t)\sum_{k \in \set{B}_n} \sigma_k(t)\Delta_k(t) \\
&= W_{\bolds{\sigma}^{\MW}(t)}(t) - 2\sum_{n \in \set{N}} Q_n(t)\sum_{k \in \set{B}_n} \sigma_k(t)\Delta_k(t)
\end{align*}
where inequality (a) is by the bound that $\mu_k \ge \bar{\mu}_k(t) - 2\Delta_k(t)$ under $\set{E}$; inequality (b) is by the definition that $\bolds{\sigma}(t) \in \arg\max_{\bolds{\sigma} \in \bolds{\Sigma}_t} \sum_{n \in \set{N}} Q_n(t)\sum_{k\in \set{B}_n} \sigma_k\bar{\mu}_k(t)$ in \textsc{MW-UCB}; inequality (c) is by the bound that $\mu_k \leq \bar{\mu}_k(t)$ under $\set{E}$. Conditioned on $\set{E}$, for $\|\bolds{Q}(t)\|_{\infty} > 0$ we then obtain
\begin{align*}
\Delta(t) = \frac{W_{\bolds{\sigma}^{\MW}(t)}(t) - W_{\bolds{\sigma}(t)}(t)}{\|\bolds{Q}(t)\|_{\infty}} &\leq \frac{2\sum_{n \in \set{N}} Q_n(t)\sum_{k \in \set{B}_n} \sigma_k(t)\Delta_k(t)}{\|\bolds{Q}(t)\|_{\infty}}
&\leq 2\sum_{k \in \set{K}} \sigma_k(t)\Delta_k(t),
\end{align*}
where the last inequality is because each server belongs to exactly one queue and $\Delta_k(t) \geq 0$ for every $k \in \set{K}$. When $\|\bolds{Q}(t)\|_{\infty} = 0$, the above ineqality trivially holds as $\Delta(t) = 0$ by defintion. This shows that $\set{E} \subseteq \set{G}_t$ and thus $\Pr\{\set{G}_t\} \geq \Pr\{\set{E}\} \geq 1 - 2Kt^{-3}$ for every period $t$. 
\end{proof}
For every horizon $T$, we can decompose the satisficing regret by conditioning on whether $\set{G}_t$ holds and using the upper bound of $\Delta(t)$ in Lemma~\ref{lem:multi-bound-delta}
\begin{align}
\sar^{\multi}(\textsc{MW-UCB},T) &\leq \sum_{t=1}^T \left(\Delta(t) - \frac{\varepsilon}{2}\right)^+\indic{\set{G}_t} + \sum_{t=1}\Delta(t)\indic{\set{G}_t^c} \nonumber\\
&\leq \sum_{t=1}^T \left(\Delta(t) - \frac{\varepsilon}{2}\right)^+\indic{\set{G}_t} + \Mmulti\sum_{t=1}\indic{\set{G}_t^c}. \label{eq:multi-sar-decomp}
\end{align}
We next give a sample path upper bound on $\sum_{t=1}^T \left(\Delta(t) - \frac{\varepsilon}{2}\right)^+\indic{\set{G}_t}$. The proof is motivated by the regret bound of UCB algorithms for combinatorial bandits \cite{chen2013combinatorial,kveton2015tight}. 
\begin{lemma}\label{lem:multi-sar-path-bound}
For horizon $T$ and under \textsc{MW-UCB}, 
$\sum_{t=1}^T \left(\Delta(t) - \frac{\varepsilon}{2}\right)^+\indic{\set{G}_t} \leq \frac{32K\Mmulti^2(\ln(T)+0.5)}{\varepsilon}.$
\end{lemma}
\begin{proof}
For a period $t$, define $\set{E}_t = \left\{\exists k \in \set{K}\colon \sigma_k(t)=1, C_k(t) \leq \frac{8\Mmulti^2\ln(t)}{\Delta(t)^2}\right\}$. We next show that we have $\set{G}_t \subseteq \set{E}_t$. To see this, observe that, conditioned on $\set{G}_t$, we have $\Delta(t) \leq 2\sum_{k \in \set{K}}\sigma_k(t)\Delta_k(t)$. Recall that $\Mmulti$ is the maximum number of severs chosen in a period. As a result, there must exist a server $k$ with $\sigma_k(t)=1$, such that $2\Delta_k(t) \geq \frac{\Delta(t)}{\Mmulti}$, implying that $C_k(t) \leq \frac{8\Mmulti^2\ln(t)}{\Delta(t)^2}$, and thus event $\set{E}_t$.

Now, for a fixed horizon $T$, we find
\begin{align}
\sum_{t=1}^T \left(\Delta(t) - \frac{\varepsilon}{2}\right)^+\indic{\set{G}_t} &\leq \sum_{t=1}^T \left(\Delta(t) - \frac{\varepsilon}{2}\right)^+\indic{\set{E}_t} \nonumber\\
&\leq \sum_{t=1}^T \sum_{k \in \set{K}} \left(\Delta(t) - \frac{\varepsilon}{2}\right)^+ \indic{\sigma_k(t)=1, C_k(t) \leq \frac{8\Mmulti^2\ln(t)}{\Delta(t)^2}}.\label{eq:multi-sar-transform}
\end{align}
Recall that non-zero $\Delta(t) = \frac{W_{\bolds{\sigma}^{\MW}(t)}(t) - W_{\bolds{\sigma}(t)}(t)}{\|\bolds{Q}(t)\|_{\infty}}$.
In addition, the support of $\{W_{\bolds{\sigma}^{\MW}(t)}(t)\}_{t \leq T}$, $\{W_{\bolds{\sigma}(t)}(t)\}_{t \leq T}$ and $\{\|\bolds{Q}(t)\|_{\infty}\}_{t \leq T}$ are finite. As a result, the support of $\{\Delta(t)\}_{t \leq T}$ is finite and is bounded by $\Mmulti$ by Lemma~\eqref{lem:multi-bound-delta}.
Let $\Omega = \{\omega_1,\ldots,\omega_C\}$ denote the support of $\{\Delta(t)\}_{t \leq T}$ such that $\omega_1 > \ldots > \omega_C$ and define $\tilde{C}$ such that $\omega_{\tilde{C}} \geq \frac{\varepsilon}{2}$ but $\omega_{\tilde{C}+1} < \frac{\varepsilon}{2}$. We can then rewrite the right hand side of  \eqref{eq:multi-sar-transform} by enumerating all possible realizations of $\Delta(t)$
\begin{align}
 \sum_{t=1}^T \left(\Delta(t) - \frac{\varepsilon}{2}\right)^+\indic{\set{G}_t} &\leq \sum_{t=1}^T \sum_{k \in \set{K}} \sum_{i=1}^C \left(\omega_i - \frac{\varepsilon}{2}\right)^+ \indic{\Delta(t)=\omega_i,\sigma_k(t)=1, C_k(t) \leq \frac{8\Mmulti^2\ln(t)}{\omega_i^2}} \nonumber\\
 &= \sum_{k \in \set{K}}\sum_{i=1}^C \left(\omega_i - \frac{\varepsilon}{2}\right)^+  \sum_{t=1}^T \indic{\Delta(t)=\omega_i,\sigma_k(t)=1, C_k(t) \leq \frac{8\Mmulti^2\ln(t)}{\omega_i^2}} \nonumber\\
 &\leq \sum_{k \in \set{K}}\sum_{i=1}^{\tilde{C}} \omega_i \sum_{t=1}^T \indic{\Delta(t)=\omega_i,\sigma_k(t)=1, C_k(t) \leq \frac{8\Mmulti^2\ln(T)}{\omega_i^2}}. \label{eq:multi-sar-sequence}
 \end{align}
 Fix a server $k$. We next argue that for any $x$, 
 \begin{equation}\label{eq:multi-sar-reordering}
 \sum_{i=1}^{\tilde{C}} \omega_i \sum_{t=1}^T \indic{\Delta(t)=\omega_i,\sigma_k(t)=1, C_k(t) \leq \frac{x}{\omega_i^2}} \leq \left(\omega_1+\frac{x}{\omega_1} + \sum_{i=2}^{\tilde{C}} \omega_i\left(\frac{x}{\omega_i^2}-\frac{x}{\omega_{i-1}^2}\right)\right).
 \end{equation}
 To see this, consider the sequence of periods up to $T$ when $\indic{\Delta(t)=\omega_i,\sigma_k(t)=1, C_k(t) \leq \frac{x}{\omega_i^2}} = 1$ and the list of associated $\Delta(t)$. Denote this list by $\omega'_1,\ldots,\omega'_{\ell}$ where $C_k(T) = \ell$. We want to upper bound the sum of all such possible lists. For a list to be feasible, we must have $j-1 \leq \frac{x}{(\omega'_j)^2}$ for every position $j \leq \ell$. In addition, if we have $\omega'_j < \omega'_{j+1}$, the sequence would still be feasible after 
 swapping them and the sum would not change. As a result, we only need to look at lists that are non-increasing, i.e, $\omega'_1 \geq \ldots \geq \omega'_{\ell}$. The list with largest possible sum is exactly the sequence that follows the order $\omega_1$ to $\omega_C$, and the sum is at most  $\omega_1+\frac{\omega_1x}{\omega_1^2}+\sum_{i=2}^{\tilde{C}}\left(\frac{\omega_i x}{\omega_i^2} - \frac{\omega_i x}{\omega_{i-1}^2}\right)$, which thus shows \eqref{eq:multi-sar-reordering}. Since $\Delta(t) \leq \Mmulti$ by Lemma~\eqref{lem:multi-bound-delta}, we have $\omega_1 \leq \Mmulti$. Using \eqref{eq:multi-sar-reordering} in \eqref{eq:multi-sar-sequence} by setting $x = 8K\Mmulti^2\ln(T)$ gives 
 \begin{align*}
 \sum_{t=1}^T \left(\Delta(t) - \frac{\varepsilon}{2}\right)^+\indic{\set{G}_t}&\leq K\omega_1+8K\Mmulti^2\ln(T) \left(\frac{1}{\omega_1} + \sum_{i=2}^{\tilde{C}} \omega_i\left(\frac{1}{\omega_i^2}-\frac{1}{\omega_{i-1}^2}\right)\right) \\
 &\leq \frac{16K\Mmulti^2(\ln(T)+0.5)}{\omega_{\tilde{C}}} \leq \frac{32K\Mmulti^2(\ln(T)+0.5)}{\varepsilon}
\end{align*} 
where the third inequality follows from the definition of $\omega_{\tilde{C}}$ and the second inequality is by Fact~\ref{fact:bound-reciprocal}, restated from \cite[Lemma~3]{KvetonWAEE14}.
\end{proof}
\begin{proof}[Proof of Lemma~\ref{lem:sar-mw-ucb}]
Using \eqref{eq:multi-sar-decomp} and Lemma~\ref{lem:multi-sar-path-bound}, we have
\begin{align*}
\expect{\sar^{\multi}(\textsc{MW-UCB},T)} &\leq \frac{32K\Mmulti^2(\ln(T)+0.5)}{\varepsilon}+\Mmulti\sum_{t=1}^T \Pr\{\set{G}_t^c\} \\
&\leq \frac{32K\Mmulti^2(\ln(T)+0.5)}{\varepsilon}+2K\Mmulti\sum_{t=1}^T \frac{1}{t^3} \\
&\leq \frac{32K\Mmulti^2(\ln(T)+0.5)}{\varepsilon}+4K\Mmulti \leq \frac{32K\Mmulti^2(\ln(T)+1)}{\varepsilon}
\end{align*}
where the second inequality uses Lemma~\ref{lem:multi-good-event}. In addition, using \eqref{eq:multi-sar-decomp} again gives
\begin{align*}
\expect{\sar^{\multi}(\textsc{MW-UCB},T)^2} &\leq \expect{\left(\sum_{t=1}^T \left(\Delta(t) - \frac{\varepsilon}{2}\right)^+\indic{\set{G}_t} + \Mmulti\sum_{t=1}^T\indic{\set{G}_t^c}\right)^2} \\
&\leq 2\expect{\left(\sum_{t=1}^T \left(\Delta(t) - \frac{\varepsilon}{2}\right)^+\indic{\set{G}_t}\right)^2}+2\Mmulti^2\expect{\left(\sum_{t=1}^T\indic{\set{G}_t^c}\right)^2} \\
&\leq \frac{2^{11}K^2\Mmulti^4(\ln(T)+0.5)^2}{\varepsilon^2} + 4\Mmulti^2\sum_{t=1}^T t\Pr\{\set{G}_t^c\} \\
&\leq  \frac{2^{11}K^2\Mmulti^4(\ln(T)+0.5)^2}{\varepsilon^2} + 8K\Mmulti^2\sum_{t=1}^T t^{-2} \tag{Lemma~\ref{lem:multi-good-event}}\\
&\leq \frac{2^{11}K^2\Mmulti^4(\ln(T)+1)^2}{\varepsilon^2}
\end{align*}
where the second inequality is by the fact that $(a+b)^2 \leq 2(a^2+b^2)$; the third inequality uses the bound in Lemma~\ref{lem:multi-sar-path-bound} and that $\left(\sum_{t=1}^T \indic{\set{G}_t^c}\right)^2 \leq 2\sum_{t_1\leq t_2 \leq T}\indic{\set{G}_{t_1}^c}\indic{\set{G}_{t_2}^c} \leq 2\sum_{t\leq T}t\indic{\set{G}_t^c}$.
\end{proof}

It is possible to improve our bound by a factor of $\Mmulti$ using the method in \cite{kveton2015tight} but it involves a much more complicated analysis and is out of the scope of this paper.

\subsection{Extension to approximate max-weight scheduling (Remark \ref{remark:app-mw})}\label{app:ext-app}
Theorem~\ref{thm:mw-ucb} in Section~\ref{sec:optimal-multi} requires the scheduling algorithm to optimally solve a maximum selection problem \eqref{eq:mx-ucb-rule}. When the feasible set of schedules $\bolds{\Sigma}_t$ has a complicated structure, this selection problem may be computationally hard to solve. The literature has systematically addressed this challenge when all system parameters are known. For example, it is well-known that if an algorithm finds an $\alpha-$approximate solution to \eqref{eq:maxweight-schedule}, it can achieve bounded queue length when the system still has traffic slackness when the capacity region $\set{S}(\bolds{\mu},\bolds{\Sigma}, \set{B})$ shrinks by a factor of $\alpha$; see \cite[Proposition 3]{lin2006impact} and the references therein. In this section we show how our analysis in Section~\ref{sec:optimal-multi} naturally extends to such a setting when there is a need to learn system parameters. Our main result is that Theorem~\ref{thm:mw-ucb} remains almost identical except for a new definition of traffic slackness in the approximate setting. Moreover, our analysis here also applies to the network setting in Section~\ref{sec:optimal-network} when we can only approximately solve Line~\ref{line:bp-ucb-choice} of Algorithm~\ref{algo:bp-ucb}. For simplicity, we choose to focus on the multi-server setting.

Formally, we define an $(\alpha,\eta)$-approximate schedule, $\bolds{\sigma}(t) \in \bolds{\Sigma}_t$, in period $t$ as one that fulfills
\begin{equation}\label{eq:approx-def}
\sum_{n \in \set{N}} Q_n(t)\sum_{k \in \set{B}_n} \sigma_k(t)\bar{\mu}_k(t) \geq \alpha \arg\max_{\bolds{\sigma} \in \bolds{\Sigma}_t}\sum_{n \in \set{N}} Q_n(t)\sum_{k \in \set{B}_n} \sigma_k\bar{\mu}_k(t) - \eta.
\end{equation}
Moreover, we extend the notion of traffic slackness in Definition~\ref{def:multi-slackness} as follows: consider the set of feasible increases in arrival rates with an $\alpha-$scaled capacity region defined by $$\mathscr{E}_{\alpha} = \{\varepsilon \in (0,1] \colon \bolds{\lambda}(\bolds{\Lambda}) + \varepsilon \bm{1} \in \alpha \set{S}(\bolds{\mu}, \bolds{\Sigma}, \bolds{\set{B}})\}.$$ Abusing notation, we define $\varepsilon \coloneqq \max\{\varepsilon' \in \mathscr{E}_{\alpha}\}$ as the $\alpha-$traffic slackness of a multi-queue multi-server system. Our next result upper bounds, for a given $\alpha-$traffic slackness, the $\colq$ of an algorithm $\textsc{App-UCB}$ that finds in each period an $(\alpha,\eta)$-approximate rather than a max-weight schedule.

\begin{theorem}\label{thm:app-ucb}
Suppose algorithm $\textsc{App-UCB}$ operates like $\textsc{MW-UCB}$ except for that it identifies in each period $t$ an $(\alpha,\eta)$-approximate rather than a max-weight schedule $\bolds{\sigma}(t) \in \bolds{\Sigma}_t$; then, for any $\bolds{\Lambda}, \bolds{\mu}, \bolds{\Sigma}, \bolds{\set{B}}$ with $\alpha-$traffic slackness $\varepsilon \in (0,1]$, we have
\[
\colq^{\multi}(\bolds{\Lambda},\bolds{\mu},\bolds{\Sigma},\bolds{\set{B}}, \textsc{App-UCB}) \leq \frac{\sqrt{N}\left(4\eta + 16\Marr+2^{10}K\Mmulti^2(1+\ln(\Marr K\Mmulti/\varepsilon))\right)}{\varepsilon}.
\]
\end{theorem}
Comparing Theorem~\ref{thm:app-ucb} with Theorem~\ref{thm:mw-ucb}, we observe the following differences: (1) we adapted the statement to the new definition of traffic slackness to account for the approximate scheduling; (2) there is a new term $4\eta$ in the numerator that captures the incremental queue length that arises from the approximation. Finally, (3) though the bound has no dependence on $\bolds{\Lambda}$ and $\bolds{\mu}$, we incorporate them on the left-hand side and rely on the definition in \eqref{eq:def-CLQ-multi} to avoid defining a new $\colq$ notion.

\subsubsection{Proof of Theorem~\ref{thm:app-ucb}}
The proof of Theorem~\ref{thm:app-ucb} resembles that of Theorem~\ref{thm:mw-ucb}. We next outline elements that require adaptation for the approximate scheduling component. Recalling the weight of a schedule defined in \eqref{eq:weight-schedule}, we denote the (\emph{scaled}) loss of schedule $\bolds{\sigma}(t)$ in period $t$ by $\Delta^{(\alpha,\eta)}(t)$, which is defined by
\begin{equation}\label{eq:scaled-loss}
\Delta^{(\alpha,\eta)}(t) = \left\{
\begin{aligned}
&0,~\text{if }\|\bolds{Q}(t)\|_{\infty} = 0\\
&\frac{\alpha W_{\bolds{\sigma}^{\MW}(t)}(t) -\eta - W_{\bolds{\sigma}(t)}(t)}{\|\bolds{Q}(t)\|_{\infty}},~\text{otherwise.}
\end{aligned}
\right.
\end{equation}
The definition in \eqref{eq:scaled-loss} generalizes that in \eqref{eq:def-multi-loss} for the exact algorithm own the weight of max-weight matching accordingly. Similarly, we define the (scaled) satisficing regret by
\begin{equation}
\sar^{(\alpha,\eta)}(\pi, T) = \sum_{t=1}^T \left(\Delta^{(\alpha,\eta)}(t) - \frac{\varepsilon}{2}\right)^+.
\end{equation}
Throughout the proof, we assume the system has an $\alpha-$traffic slackness $\varepsilon \in [0,1]$. The following result resembles Lemmas~\ref{lem:multi-learning}.

\begin{lemma}\label{lem:appr-learning}
For any policy $\pi$ and horizon $T$, $\frac{\sum_{t=1}^T \expect{\|\bolds{Q}(t)\|_2}}{T} \leq \frac{16\Marr+20\Mmulti^2+4\eta}{\varepsilon} + \expect{\sar^{(\alpha, \eta)}(\pi, T)}$.
\end{lemma}
Compared to Lemma~\ref{lem:multi-learning}, Lemma~\ref{lem:appr-learning} uses the scaled satisficing regret, $\alpha-$ traffic slackness, and includes a new term $\eta$ to account for the approximation setting. Lemma~\ref{lem:multi-regenerate} similarly extends to this setting as follows.
\begin{lemma}\label{lem:appr-regenerate}
For any policy $\pi$ and horizon $T$, $\frac{\sum_{t=1}^T \expect{\|\bolds{Q}(t)\|_1}}{T} \leq \frac{2\Marr+6\Mmulti^2+4\eta}{\varepsilon} + \frac{16\Marr}{\varepsilon^2} \frac{\expect{\sar^{(\alpha,\eta)}(\pi, T)^2}}{T}$.
\end{lemma}
The below result bounds the scaled satisficing regret for $\textsc{App-UCB}$ like Lemma~\ref{lem:sar-mw-ucb} for $\textsc{MW-UCB}.$
\begin{lemma}\label{lem:sar-app-ucb}
For any horizon $T$, 
\begin{align*}
\expect{\sar^{(\alpha,\eta)}(\textsc{App-UCB},T)} &\leq \frac{32K\Mmulti^2(\ln T+1)}{\varepsilon}  \\
\expect{\sar^{(\alpha,\eta)}(\textsc{App-UCB},T)^2} &\leq \frac{2^{11}K^2\Mmulti^4(\ln T+1)^2}{\varepsilon^2}.
\end{align*}
\end{lemma}
\begin{proof}[Proof of Theorem~\ref{thm:app-ucb}]
The proof follows the exact same steps as that of Theorem~\ref{thm:mw-ucb} with Lemmas~\ref{lem:appr-learning}, \ref{lem:appr-regenerate} and \ref{lem:sar-app-ucb} replacing Lemmas~\ref{lem:multi-learning}, \ref{lem:multi-regenerate} and \ref{lem:sar-mw-ucb}. As Lemma~\ref{lem:sar-app-ucb} gives the exact same form of guarantee as Lemma~\ref{lem:sar-mw-ucb} (though $\varepsilon$ is the $\alpha-$traffic slackness), whereas Lemmas~\ref{lem:appr-learning},~\ref{lem:appr-regenerate} differ from their corresponding lemmas by a new term $\frac{4\eta}{\varepsilon}$, the bound in Theorem~\ref{thm:app-ucb} follows by including the new term $\frac{4\eta}{\varepsilon}$; we multiply this term by $\sqrt{N}$ because Lemma~\ref{lem:appr-learning} bounds the $2-$norm of queue lengths which we translate, using $\|\bolds{Q}(t)\|_1 \leq \sqrt{N} \|\bolds{Q}(t)\|_2$, into the $1-$norm of queue lengths as we did in the proof of Theorem \ref{thm:mw-ucb}. \end{proof}
In the following subsections we highlight the changes in the proofs of Lemmas~\ref{lem:appr-learning}, \ref{lem:appr-regenerate} and \ref{lem:sar-app-ucb} compared to Lemmas~\ref{lem:multi-learning}, Lemma~\ref{lem:multi-regenerate} and Lemma~\ref{lem:sar-mw-ucb}.

\subsubsection{Queue length bound in the learning state (Lemma~\ref{lem:appr-learning})}
The key change is extending Lemma~\ref{lem:multi-bound-weight} to the setting with $\alpha-$traffic slackness.
\begin{lemma}\label{lem:app-bound-weight}
Given $\alpha-$traffic slackness $\varepsilon$, for a period $t$, $$\sum_{n \in \set{N}} \lambda_n Q_n(t) - W_{\bolds{\sigma}(t)}(t) \leq -\varepsilon\sum_{n \in \set{N}} Q_n(t) + \Mmulti^2 + \eta + \Delta^{(\alpha,\eta)}(t)\|\bolds{Q}(t)\|_{\infty}.$$
\end{lemma}
\begin{proof}
The $\alpha-$traffic slackness ensures the existence of a distribution $\phi$ over schedules in $\Sigma$ such that
\begin{align*}
\sum_{n \in \set{N}} (\lambda_n + \varepsilon)Q_n(t) &\leq \alpha \sum_{\sigma \in \Sigma} \phi_{\sigma}\sum_{n \in \set{N}} Q_n(t) \sum_{k \in \set{B}_n} \sigma_k\mu_k \\
&= \alpha \sum_{\sigma \in \Sigma} \phi_{\sigma} W_{\sigma}(t) \\
&\leq \alpha \max_{\sigma \in \Sigma} W_{\sigma}(t) = \alpha W_{\tilde{\sigma}(t)}(t),
\end{align*}
where the schedule $\tilde{\sigma}(t)$, as in the proof of Lemma~\ref{lem:multi-bound-weight},  is the one in $\Sigma$ that has the largest weight in period~$t$. The same derivation in \eqref{eq:mw-feasible-bound} gives $W_{\sigma^{\MW}(t)}(t) \geq W_{\tilde{\sigma}(t)}(t) - \Mmulti^2.$ Therefore, the definition of $\Delta^{(\alpha,\eta)}(t)$ gives
\[
W_{\sigma(t)}(t) + \eta + \Delta^{(\alpha,\eta)}(t)\|\bolds{Q}(t)\|_{\infty} \geq \alpha W_{\sigma^{\MW}(t)}(t) \geq \alpha W_{\tilde{\sigma}(t)}(t)-\Mmulti^2 \geq \sum_{n \in \set{N}}(\lambda_n+\varepsilon) Q_n(t) - \Mmulti^2.
\]
Rearranging terms implies the result.
\end{proof}
Lemma~\ref{lem:app-bound-weight} implies the below drift bound which is similar to Lemma~\ref{lem:multi-drift-norm2}. Recall the Lyapunov function $\varphi(t) = \|\bolds{Q}(t)\|_{2}.$
\begin{lemma}\label{lem:app-drift-norm2}
Given $\alpha-$traffic slackness $\varepsilon$, for the process $\{\varphi(t)\}$, we have that for every period $t$, 
\begin{enumerate}
\item Bounded difference: $|\varphi(t+1) - \varphi(t)| \leq \sqrt{\Marr + \Mmulti^2}$
\item Drift bound: $\expect{\varphi(t+1) - \varphi(t) \mid \set{F}_t} \leq -\frac{\varepsilon}{4} + \left(\Delta^{(\alpha,\eta)}(t) - \frac{\varepsilon}{2}\right)^+
$ if $\varphi(t) \geq \frac{4(\Marr+2\Mmulti^2 + \eta)}{\varepsilon}$.
\end{enumerate}
\end{lemma}
\begin{proof}
We omit the proof on the bounded difference as the proof of bounded difference in Lemma~\ref{lem:multi-drift-norm2} applies to any schedule.
For the drift bound, condition on a filtration $\set{F}_t$ with $\varphi(t) \geq \frac{4(\Marr+2\Mmulti^2 + \eta)}{\varepsilon}$. 
Using, respectively, \eqref{eq:drift-simplify}, \Cref{lem:app-bound-weight}, the fact that $\phi(t)\leq \|\bolds{Q}(t)\|_{1}$ and $\phi(t)\geq \|\bolds{Q}(t)\|_{\infty}$, and the conditioning on $\varphi(t)\geq \frac{4(\Marr+2\Mmulti^2 + \eta)}{\varepsilon}$ we obtain
\begin{align*}
\expect{\varphi(t+1) - \varphi(t) \mid \set{F}_t} &\leq \frac{\sum_{n \in \set{N}}\lambda_nQ_n(t) - W_{\bolds{\sigma}(t)}(t)}{\varphi(t)} + \frac{\Marr+\Mmulti^2}{\varphi(t)}\\
&\leq \frac{-\varepsilon\sum_{n \in \set{N}} Q_n(t) + \Mmulti^2 + \eta + \Delta^{(\alpha,\eta)}(t)\|\bolds{Q}(t)\|_{\infty}}{\varphi(t)} + \frac{\Marr+\Mmulti^2}{\varphi(t)} \\
&\leq -\varepsilon + \Delta^{(\alpha,\eta)}(t) + \frac{\Marr+2\Mmulti^2+\eta}{\varphi(t)} \\
&\leq -\varepsilon + \Delta^{(\alpha,\eta)}(t)+\frac{\varepsilon}{4} \leq -\frac{\varepsilon}{4} + \left(\Delta^{(\alpha,\eta)}(t)-\frac{\varepsilon}{2}\right)^+.
\end{align*}
\end{proof}
\begin{proof}[Proof of Lemma~\ref{lem:appr-learning}]
The proof applies Lemma~\ref{lem:bound-from-drift}. Let $Z(t) = \left(\Delta^{(\alpha,\eta)}(t)-\frac{\varepsilon}{2}\right)^+, \kappa = \sqrt{\Marr + \Mmulti^2}$, $B = \frac{4(\Marr+2\Mmulti^2 + \eta)}{\varepsilon}$ and $\delta = \frac{\varepsilon}{4}$. Lemmas~\ref{lem:app-drift-norm2} and \ref{lem:bound-from-drift} show for $t \leq T$,
\begin{align*}
\expect{\varphi(t)} &\leq 4\sqrt{\Marr+\Mmulti^2} + \frac{4(\Marr+2\Mmulti^2 + \eta)}{\varepsilon} + \frac{8(\Marr + \Mmulti^2)}{\varepsilon} + \expect{\sar^{(\alpha,\eta)}(\pi, t)} \\
&\leq \frac{16\Marr + 20\Mmulti^2 + 4\eta}{\varepsilon} + \expect{\sar^{(\alpha,\eta)}(\pi, t)} \\
&\leq \frac{16\Marr + 20\Mmulti^2 + 4\eta}{\varepsilon} + \expect{\sar^{(\alpha,\eta)}(\pi, T)}.
\end{align*}
Summing across $1 \leq t \leq T$ and dividing the sum by $T$ gives the desired result.
\end{proof}

\subsubsection{Queue length bound in the regenerate stage (Lemma~\ref{lem:appr-regenerate})}
The proof is almost identical to that of Lemma~\ref{lem:multi-regenerate} except using Lemma~\ref{lem:app-bound-weight} in lieu of Lemma~\ref{lem:multi-bound-weight}.
\begin{proof}[Proof of Lemma~\ref{lem:appr-regenerate}]
Letting $V(t) = \sum_{n \in \set{N}} Q_n^2(t)$, its drift conditioned on the filtration $\set{F}_t$ is
\begin{align*}
\expect{V(t+1) - V(t) \mid \set{F}_t} &\leq \Marr + \Mmulti^2 + 2\sum_{n \in \set{N}} (\lambda_n - \sum_{k \in \set{B}_n} \sigma_k(t)\mu_k)Q_n(t) \tag{as in \eqref{eq:sqr-drift-simplify}} \\
&\leq \Marr + 3\Mmulti^2 + 2\eta - 2\varepsilon \sum_{n \in \set{N}} Q_n(t) + 2\Delta^{(\alpha,\eta)}(t)\|\bolds{Q}(t)\|_{\infty} \tag{By Lemma~\ref{lem:app-bound-weight}} \\
&\leq \Marr + 3\Mmulti^2 + 2\eta - \varepsilon \sum_{n \in \set{N}} Q_n(t) + 2\left(\Delta^{(\alpha,\eta)}(t)-\frac{\varepsilon}{2}\right)^+\|\bolds{Q}(t)\|_{1}.
\end{align*}
The same proof of Claim~\ref{claim:sample-path-bound} gives
\begin{equation}\label{eq:claim-path-bound}
-\frac{\varepsilon}{2}\expect{\sum_{t=1}^T \|\bolds{Q}(t)\|_1}+2\expect{\sum_{t=1}^T(\Delta^{(\alpha,\eta)}(t)-\frac{\varepsilon}{2})^+\|\bolds{Q}(t)\|_{1}} \leq \frac{8\Marr\sar^{(\alpha,\eta)}(\pi,T)^2}{\varepsilon}.
\end{equation}
For horizon $T$, telescoping the drift $\expect{V(t+1) - V(t)}$ from $t = 1$ to $T$ and applying \eqref{eq:claim-path-bound} implies
\[
0 \leq T(\Marr + 3\Mmulti^2 + 2\eta) - \frac{\varepsilon\expect{\sum_{t=1}^T \|\bolds{Q}(t)\|_1}}{2} + \frac{8\Marr\expect{\sar^{(\alpha,\eta)}(\pi,T)^2}}{\varepsilon},
\]
which gives the desired result by dividing both sides by $\frac{\varepsilon T}{2}$ and rearranging terms.
\end{proof}

\subsubsection{Satisficing regret of $\textsc{App-UCB}$ (Lemma~\ref{lem:sar-app-ucb})}
Recall $\Delta_k(t) = \sqrt{\frac{2\ln t}{C_k(t)}}$ is the confidence bound of server $k$ in period $t$ for $\textsc{App-UCB}$. Let $\set{G}^{(\alpha,\eta)}_t = \{\Delta^{(\alpha,\eta)}(t) \leq 2\sum_{k \in \set{K}} \sigma_k(t)\Delta_k(t)\}$ be the good event that the scaled loss of the chosen schedule under $\textsc{App-UCB}$ is upper bounded by the sum of confidence bounds of chosen servers. We establish the same result of Lemma~\ref{lem:multi-good-event} for $\textsc{App-UCB}$  that this good event happens with high probability.
\begin{lemma}\label{lem:app-good-event}
For every period $t$ and under $\textsc{App-UCB}$, we have $\Pr\left\{\set{G}^{(\alpha,\eta)}_t\right\} \geq 1 - 2Kt^{-3}.$
\end{lemma}
The remaining steps of the proof of Lemma~\ref{lem:sar-mw-ucb} in Appendix~\ref{app:lem-sar-mw-ucb} directly apply to show how Lemma~\ref{lem:app-good-event} implies Lemma~\ref{lem:sar-app-ucb} by replacing the loss $\Delta(t)$ with the scaled loss $\Delta^{(\alpha,\eta)}(t)$, the satisficing regret $\sar^{\multi}$ with $\sar^{(\alpha,\eta)}$ and the good event $\set{G}_t$ with $\set{G}^{(\alpha,\eta)}_t$, and thus we omit the full proof of Lemma~\ref{lem:sar-app-ucb}.

\begin{proof}[Proof of Lemma~\ref{lem:app-good-event}]
Fix a period $t$. Using Hoeffding's Inequality (Fact~\ref{fact:hoeffding}) and the union bound gives that with probability at least $1 - 2Kt^{-3}$, the confidence bound is correct for any server $k$, i.e., $\Pr\{\forall k \in \set{K}, |\mu_k - \hat{\mu}_k(t)| < \Delta_k(t)\} \geq 1 - 2Kt^{-3}.$ Denote this event by $\set{E}.$ We next show event $\set{E}$ implies event $\set{G}^{(\alpha,\eta)}_t$, which prove the desired result as then $\Pr\{\set{G}^{(\alpha,\eta)}_t\} \geq \Pr\{\set{E}\} \geq 1 - 2Kt^{-3}.$ 

Condition on $\set{E}$ and suppose that $\|\bolds{Q}(t)\|_{\infty} > 0$; otherwise $\set{G}^{(\alpha,\eta)}_t$ holds true because $\Delta^{(\alpha,\eta)}(t) = 0$ by its definition. Under event $\set{E}$, the upper confidence bounds satisfy $\bar{\mu}_k(t) - 2\Delta_k(t) \leq \mu_k \leq \bar{\mu}_k(t)$ for any server $k$. Moreover, the weight of the selected schedule $\sigma(t)$ satisfies
\begin{align*}
W_{\sigma(t)}(t) &= \sum_{n \in \set{N}} Q_n(t) \sum_{k \in \set{B}_n} \sigma_k(t)\mu_k \\
&\geq \sum_{n \in \set{N}} Q_n(t)\sum_{k \in \set{B}_n} \sigma_k(t)\left(\bar{\mu}_k(t) - 2\Delta_k(t)\right) \\
&= \sum_{n \in \set{N}} Q_n(t)\sum_{k \in \set{B}_n} \sigma_k(t)\bar{\mu}_k(t) - 2\sum_{n \in \set{N}} Q_n(t)\sum_{k \in \set{B}_n} \sigma_k(t)\Delta_k(t) \\
&\overset{(a)}{\geq} \alpha \sum_{n \in \set{N}} Q_n(t)\sum_{k \in \set{B}_n} \sigma_k^{\MW}(t)\bar{\mu}_k(t) - \eta - 2\sum_{n \in \set{N}} Q_n(t)\sum_{k \in \set{B}_n} \sigma_k(t)\Delta_k(t) \\
&\geq \alpha \sum_{n \in \set{N}} Q_n(t)\sum_{k \in \set{B}_n} \sigma_k^{\MW}(t)\mu_k(t) - \eta - 2\sum_{n \in \set{N}} Q_n(t)\sum_{k \in \set{B}_n} \sigma_k(t)\Delta_k(t) \\
&= \alpha W_{\sigma^{\MW}(t)}(t) - \eta - 2\sum_{n \in \set{N}} Q_n(t)\sum_{k \in \set{B}_n} \sigma_k(t)\Delta_k(t),
\end{align*}
where Inequality (a) is because $\textsc{App-UCB}$ finds an $(\alpha,\eta)-$approximate scheduling that satisfies~\eqref{eq:approx-def}.
Rewriting, we obtain $ \alpha W_{\sigma^{\MW}(t)}(t) - \eta - W_{\sigma(t)}(t)\leq 2\sum_{n \in \set{N}} Q_n(t)\sum_{k \in \set{B}_n} \sigma_k(t)\Delta_k(t)$. Plugging this into \eqref{eq:scaled-loss} we obtain 
\[ \Delta^{(\alpha,\eta)}(t) \leq \frac{2\sum_{n \in \set{N}} Q_n(t)\sum_{k \in \set{B}_n} \sigma_k(t)\Delta_k(t)}{\|\bolds{Q}(t)\|_{\infty}} \leq 2\sum_{n \in \set{N}}\sum_{k \in \set{B}_n} \sigma_k(t)\Delta_k(t) = 2\sum_{k \in \set{K}} \sigma_k(t)\Delta_k(t).
\]
Therefore, event $\set{E}$ implies event $\set{G}_t^{(\alpha,\eta)},$ which completes the proof.
\end{proof}

\section{Additional proofs for Section~\ref{sec:optimal-network}} \label{app:network}
\subsection{Queue length bound in the learning stage (Lemma~\ref{lem:network-learning})}\label{app:network-learning}
For every $t$, define $\tilde{\sigma}(t)$ as the schedule in $\Sigma$ with the largest weight, i.e., $\tilde{\sigma}(t) \in \arg\max_{\sigma \in \bolds{\Sigma}} W^{\net}_{\sigma}(t)$. Given that $\bolds{\Sigma}_t$ may be a strict subset of $\bolds{\Sigma}$, this schedule need not be feasible in period $t$ (it may schedule more jobs from a queue than the queue has).
We first lower bound the weight under schedule $\bolds{\sigma}(t)$ selected by $\textsc{BP-UCB}$ by a coupling with $\tilde{\bolds{\sigma}}(t)$; the proof resembles that of Lemma~\ref{lem:multi-bound-weight}. 
\begin{lemma}\label{lem:network-weight-lowerb}
For a period $t$, $\sum_{n \in \set{N}}\lambda_n Q_n(t)-W^{\net}_{\bolds{\sigma}(t)}(t) \leq -\varepsilon\sum_{n \in \set{N}}Q_n(t)+\Mmulti^2+\Delta^{\net}(t)\|\bolds{Q}(t)\|_{\infty}$.
\end{lemma}
\begin{proof}
Fix a period $t$. The positive traffic slackness, captured through $\bolds{\lambda} + \varepsilon \bm{1} \in \set{S}^{\net}(\bolds{\mu}, \bolds{\Sigma},\bolds{\set{B}}, \bm{P})$, guarantees that there exists a distribution $\phi$ over $\bolds{\Sigma}$ such that for every queue $n$, we have
\[
\lambda_n + \varepsilon \leq \sum_{\bolds{\sigma} \in \bolds{\Sigma}} \phi_{\bolds{\sigma}}\left(\sum_{k \in \set{B}_n} \sigma_k\mu_k - \sum_{k'\in \set{K}} \sigma_{k'}\mu_{k'}p_{k',n}\right).
\]
Multiplying both sides by $Q_n(t)$ and summing over all queues gives $\sum_{n \in \set{N}}(\lambda_n + \varepsilon)Q_n(t)$
\begin{align}
 &\leq \sum_{\bolds{\sigma} \in \bolds{\Sigma}} \phi_{\bolds{\sigma}}\sum_{n \in \set{N}} Q_n(t)\left(\sum_{k \in \set{B}_n} \sigma_k\mu_k - \sum_{k'\in \set{K}} \sigma_{k'}\mu_{k'}p_{k',n}\right) 
&= \sum_{\bolds{\sigma} \in \bolds{\Sigma}} \phi_{\bolds{\sigma}}W^{\net}_{\bolds{\sigma}}(t) \leq W^{\net}_{\tilde{\bolds{\sigma}}(t)}(t) \label{eq:bound-of-weight-network-optimal}
\end{align}
where the last inequality is because $\phi$ is a distribution and $\tilde{\bolds{\sigma}}(t)$ is the schedule in $\bolds{\Sigma}$ with the largest weight. As a result, $\sum_{n \in \set{N}}\lambda_n Q_n(t) - W^{\net}_{\tilde{\bolds{\sigma}}(t)}(t) \leq -\varepsilon\sum_{n \in \set{N}} Q_n(t)$. 

We next lower bound the weight of $\bolds{\sigma}^{\BP}(t) = \arg\max_{\bolds{\sigma} \in \bolds{\Sigma}_t} W^{\net}_{\bolds{\sigma}}(t)$, the schedule with the largest weight over schedules in $\bolds{\Sigma}_t$. We construct a schedule $\bolds{\alpha} \in \bolds{\Sigma}_t$ from $\tilde{\bolds{\sigma}}(t)$. For a queue $n$, if $Q_n(t) \geq \sum_{k \in \set{B}_n} \tilde{\bolds{\sigma}}_k(t)$, we set $\alpha_k = \tilde{\bolds{\sigma}}_k(t)$ for all $k \in \set{B}_n$; otherwise, we set $\alpha_k = 0$ for $k \in \set{B}_n$. By construction we have $\bolds{\alpha} \in \bolds{\Sigma}_t$, and thus $W^{\net}_{\bolds{\sigma}^{\BP}(t)}(t) \geq W^{\net}_{\bolds{\alpha}}(t)$. For a server $k$, let $i_k$ be the queue it belongs to and let $h_n = \indic{Q_n(t) < \sum_{k \in \set{B}_n} \tilde{\sigma}_k(t)}$ for a queue $n$. Then 
\begin{align}
W^{\net}_{\bolds{\sigma}^{\BP}(t)}(t) \geq W^{\net}_{\bolds{\alpha}}(t) &= W^{\net}_{\tilde{\bolds{\sigma}}(t)}(t) - \left(\sum_{n \in \set{N}} h_n Q_n(t)\sum_{k \in \set{B}_n} \tilde{\sigma}_k(t)\mu_k - \sum_{k'\in \set{K}} h_{i_{k'}} \tilde{\sigma}_{k'}(t)\mu_{k'}\sum_{n \in \set{N}}p_{k',n}Q_n(t)\right) \nonumber\\
&\geq W^{\net}_{\tilde{\bolds{\sigma}}(t)}(t) - \sum_{n \in \set{N}} h_n Q_n(t)\sum_{k \in \set{B}_n} \tilde{\sigma}_k(t)\mu_k  \nonumber\\
&\geq W^{\net}_{\tilde{\bolds{\sigma}}(t)}(t) - \sum_{n \in \set{N}} \left(\sum_{k \in \set{B}_n} \tilde{\sigma}_k(t)\mu_k\right)^2 \geq W^{\net}_{\tilde{\bolds{\sigma}}(t)}(t) - \Mmulti^2.\label{eq:bound-of-bp-weight}
\end{align}
We then finish the proof by recalling the definition of $\Delta^{\net}(t)$ in \eqref{eq:def-network-loss} 
\begin{align*}
\sum_{n \in \set{N}} \lambda_n Q_n(t) - W_{\bolds{\sigma}(t)}^{\net}(t) &\leq \sum_{n \in \set{N}} \lambda_n Q_n(t) - W_{\bolds{\sigma}^{\BP}(t)}^{\net}(t) + \Delta^{\net}(t)\|\bolds{Q}(t)\|_{\infty} \tag{Definition of $\Delta^{\net}(t)$} \\
&\leq \sum_{n \in \set{N}} \lambda_n Q_n(t) - W_{\tilde{\bolds{\sigma}}(t)}^{\net}(t) + \Mmulti^2 + \Delta^{\net}(t)\|\bolds{Q}(t)\|_{\infty} \tag{\ref{eq:bound-of-bp-weight}} \\
&\leq -\varepsilon \sum_{n \in \set{N}} Q_n(t) + \Mmulti^2 + \Delta^{\net}(t)\|\bolds{Q}(t)\|_{\infty} \tag{\ref{eq:bound-of-weight-network-optimal}},
\end{align*}
which finishes the proof.
\end{proof}
Let $\set{F}_t$ be the $\sigma-$field generated by the sample path before period $t$ (including arrivals $A_n(\tau)$, services $S_k(\tau)$ and transitions $L_{k,n}(\tau)$ for $\tau < t$). Then $\{\varphi(t)\},\{\Delta^{\net}(t)\}$ and $\{\bolds{\sigma}(t)\}$ are $\{\set{F}_t\}$-adapted process. We bound the maximum increase (decrease) of $\varphi(t)$ and its drift in the following lemma which is similar to Lemma~\ref{lem:multi-drift-norm2}.
\begin{lemma}\label{lem:network-drift-norm2}
For the process $\{\varphi(t)\}$, we have that for every period $t$,
\begin{enumerate}
    \item Bounded difference: $|\varphi(t+1)-\varphi(t)| \leq \sqrt{2\Marr+3\Mmulti^2}$; 
    \item Drift bound: $\expect{\varphi(t+1) - \varphi(t) \mid \set{F}_t} \leq -\frac{\varepsilon}{4} + \left(\Delta^{\net}(t)-\frac{\varepsilon}{2}\right)^+$ if $\varphi(t) \geq \frac{8(\Marr+2\Mmulti^2)}{\varepsilon}$.
\end{enumerate}
\end{lemma}
\begin{proof}
Fix a period $t$. By triangle inequality, $|\varphi(t+1)-\varphi(t)|=|\|\bolds{Q}(t+1)\|_2 - \|\bolds{Q}(t)\|_2| \leq \|\bolds{Q}(t+1)-\bolds{Q}(t)\|_2$ and
\begin{align}
\|\bolds{Q}(t+1)-\bolds{Q}(t)\|_2 &= 
\sqrt{\sum_{n\in\set{N}} \left(-\sum_{k \in \set{B}_n}\sigma_k(t)S_k(t)+A_n(t)+\sum_{k'\in\set{K}}\sigma_{k'}(t)S_{k'}(t)L_{k',n}(t)\right)^2} \tag{by the queueing dynamic \eqref{eq:dynamic-network}} \\
&\leq \sqrt{\sum_{n \in \set{N}} \left(\left(A_n(t)+\sum_{k'\in\set{K}}\sigma_{k'}(t)S_{k'}(t)L_{k',n}(t)\right)^2+\left(\sum_{k\in \set{B}_n} \sigma_k(t)S_k(t)\right)^2\right)} \nonumber\\
&\leq \sqrt{2\sum_{n \in \set{N}} A_n(t) + 2\left(\sum_{n \in \set{N}}\sum_{k' \in \set{K}}\sigma_{k'}(t)S_{k'}(t)L_{k',n}(t)\right)^2 + \Mmulti^2} \tag{$A_n(t) \leq 1$ and $(x+y)^2 \leq 2(x^2+y^2)$ for any $x,y$} \nonumber\\
&\leq \sqrt{2\Marr + 2\left(\sum_{k' \in \set{K}} \sigma_{k'}(t)S_{k'}(t)\right)^2 + \Mmulti^2} \tag{$\sum_{n \in \set{N}} L_{k',n}(t) \leq 1$ for $k' \in \set{K}$} \nonumber\\
&\leq \sqrt{2\Marr + 2\Mmulti^2 + \Mmulti^2} = \sqrt{2\Marr + 3\Mmulti^2}. \label{eq:net-bounded-drift}
\end{align}
To bound the drift condition on $\set{F}_t$, we again use Fact~\ref{fact:diff-norm2} that for $\|\bolds{Q}(t)\|_2 > 0$, we have
\begin{equation} \label{eq:drift-bound-decompose-restated}
\|\bolds{Q}(t+1)\|_2 \leq \|\bolds{Q}(t)\|_2 + \frac{\bolds{Q}(t)\cdot (\bolds{Q}(t+1) - \bolds{Q}(t)) + \|\bolds{Q}(t+1)-\bolds{Q}(t)\|_2^2}{\|\bolds{Q}(t)\|_2}.
\end{equation}
Condition on $\set{F}_t$ and suppose $\varphi(t) \geq \frac{8(\Marr + 2\Mmulti^2)}{\varepsilon}$. We first  bound 
\begin{align}
&\hspace{0.2in}\expect{\bolds{Q}(t)\cdot (\bolds{Q}(t+1) - \bolds{Q}(t)) \mid \set{F}_t} \nonumber\\
&= \expect{\sum_{n \in \set{N}} Q_n(t)\left(A_n(t) - \sum_{k\in\set{B}_n}\sigma_k(t)S_k(t)+\sum_{k'\in\set{K}}\sigma_{k'}(t)S_{k'}(t)L_{k',n}(t)\right) \mid \set{F}_t} \tag{by \eqref{eq:dynamic-network}} \\
&= \sum_{n \in \set{N}} Q_n(t)\left(\lambda_n - \sum_{k\in\set{B}_n}\sigma_k(t)\mu_k+\sum_{k'\in\set{K}}\sigma_{k'}(t)\mu_{k'}p_{k',n}\right) \tag{arrivals, services and transitions in a period are independent of the history} \\
&=\sum_{n \in \set{N}} \lambda_n Q_n(t) - W^{\net}_{\bolds{\sigma}(t)}(t) \nonumber\\
&\leq -\varepsilon\sum_{n \in \set{N}} Q_n(t) + \Mmulti^2 + \Delta^{\net}(t)\|\bolds{Q}(t)\|_{\infty} \label{eq:network-bound-cross-dot}
\end{align}
where the last inequality is by Lemma~\ref{lem:network-weight-lowerb}. We then have
\begin{align*}
\expect{\varphi(t+1) - \varphi(t) \mid \set{F}_t} &\leq  \expect{\frac{\bolds{Q}(t)\cdot (\bolds{Q}(t+1) - \bolds{Q}(t)) + \|\bolds{Q}(t+1)-\bolds{Q}(t)\|_2^2}{\|\bolds{Q}(t)\|_2} \mid \set{F}_t} \tag{by \eqref{eq:drift-bound-decompose-restated}} \\
&\leq \expect{\frac{\bolds{Q}(t)\cdot (\bolds{Q}(t+1) - \bolds{Q}(t))}{\|\bolds{Q}(t)\|_2} \mid \set{F}_t} + \frac{2\Marr+3\Mmulti^2}{\|\bolds{Q}(t)\|_2} \tag{by \eqref{eq:net-bounded-drift}} \\
&\leq \frac{-\varepsilon\sum_{n \in \set{N}} Q_n(t) + \Mmulti^2 + \Delta^{\net}(t)\|\bolds{Q}(t)\|_{\infty}}{\|\bolds{Q}(t)\|_2} +\frac{2\Marr+3\Mmulti^2}{\|\bolds{Q}(t)\|_2}  \tag{by \eqref{eq:network-bound-cross-dot} and that $\bolds{Q}(t) \in \set{F}_t$} \\
&\leq -\varepsilon + \Delta^{\net}(t) + \frac{2\Marr+4\Mmulti^2}{\|\bolds{Q}(t)\|_2} \tag{$\|\bolds{Q}(t)\|_1 \geq \|\bolds{Q}(t)\|_2 \geq \|\bolds{Q}(t)\|_{\infty}$} \\
&\leq -\frac{\varepsilon}{2}+(\Delta^{\net}(t)-\frac{\varepsilon}{2})^+ + \frac{\varepsilon}{4} \tag{assumption that $\varphi(t) \geq \frac{8(\Marr+2\Mmulti^2)}{\varepsilon}$} \\
&= -\frac{\varepsilon}{4}+(\Delta^{\net}(t)-\frac{\varepsilon}{2})^+,
\end{align*}
which finishes the proof.
\end{proof}
We next show an upper bound on $\Delta^{\net}(t)$.
\begin{lemma}\label{lem:net-bound-delta}
For every period $t$, $\Delta^{\net}(t) \leq 2\Mmulti$.
\end{lemma}
\begin{proof}
If $\|\bolds{Q}(t)\|_{\infty} = 0$, then $\Delta^{\net}(t) = 0$; otherwise,
\begin{align*}
\Delta^{\net}(t) &= \frac{W^{\net}_{\bolds{\sigma}^{\BP}(t)}(t)-W^{\net}_{\bolds{\sigma}(t)}(t)}{\|\bolds{Q}(t)\|_{\infty}} \nonumber\\
&= \frac{\sum_{n \in \set{N}} Q_n(t)\left(\sum_{k \in \set{B}_n} \left(\sigma^{\BP}_k(t)-\sigma_k(t)\right)\mu_k - \sum_{k' \in \set{K}} \left(\sigma^{\BP}_{k'}(t)-\sigma_{k'}(t)\right)\mu_{k'}p_{k',n}\right)}{\|\bolds{Q}(t)\|_{\infty}} \nonumber\\
&\leq \frac{\sum_{n \in \set{N}} Q_n(t)\left(\sum_{k \in \set{B}_n} \sigma^{\BP}_k(t)\mu_k + \sum_{k' \in \set{K}} \sigma_{k'}(t)\mu_{k'}p_{k',n}\right)}{\|\bolds{Q}(t)\|_{\infty}} \nonumber \\
&\leq \sum_{n \in \set{N}} \left(\sum_{k \in \set{B}_n} \sigma^{\BP}_k(t)\mu_k + \sum_{k' \in \set{K}} \sigma_{k'}(t)\mu_{k'}p_{k',n}\right) \nonumber \\
&\leq \sum_{k \in \set{K}} \sigma_k^{\BP}(t) + \sum_{k' \in \set{K}}\sigma_{k'}(t)\sum_{n \in \set{N}}p_{k',n} \nonumber\\
&\leq \sum_{k \in \set{K}} \sigma_k^{\BP}(t) + \sum_{k' \in \set{K}}\sigma_{k'}(t) \tag{since $p_{k',\cdot}$ is a probability distribution over $\set{N} \cup \{\perp\}$} \\
&\leq 2\Mmulti. 
\end{align*}
\end{proof}
The proof of Lemma~\ref{lem:network-learning} combines Lemmas~\ref{lem:network-drift-norm2},~\ref{lem:net-bound-delta} and ~\ref{lem:bound-from-drift}.
\begin{proof}[Proof of Lemma~\ref{lem:network-learning}]
To apply Lemma~\ref{lem:bound-from-drift}, wet set $\Psi(t) = \varphi(t), Z(t) = \left(\Delta^{\net}(t)-\frac{\varepsilon}{2}\right)^+$. To verify the condition of Lemma~\ref{lem:bound-from-drift}, take $\kappa = \sqrt{2\Marr+4\Mmulti^2}, \delta = \varepsilon/4, B = \frac{8(\Marr+2\Mmulti^2)}{\varepsilon}$. By Lemma~\ref{lem:network-drift-norm2}, we have $|\Psi(t+1)-\Psi(t)| \leq \kappa$ with probability one and $\expect{\Psi(t+1)-\Psi(t) \mid \set{F}_t} \leq -\delta + Z(t)$ if $\Psi(t) \geq B$. Lemma~\ref{lem:net-bound-delta} shows $Z(t) \leq \Delta^{\net}(t) \leq 2\Mmulti \leq \kappa$. For a horizon $T$, using Lemma~\ref{lem:bound-from-drift} to $\varphi(t)$ for $t \leq T$ and taking average,
\begin{align*}
\frac{\sum_{t=1}^T \expect{\|\bolds{Q}(t)\|_2}}{T} &= \frac{\sum_{t=1}^T \expect{\Psi(t)}}{T} \tag{Definition of $\Psi(t) = \varphi(t) = \|\bolds{Q}(t)\|_2$} \\
&\leq 4\kappa + B + \frac{2\kappa^2}{\delta}+\expect{\sum_{t=1}^{T-1} Z(t)} \tag{Lemma~\ref{lem:bound-from-drift}} \\
&\hspace{-0.5in}= 4\sqrt{2\Marr+4\Mmulti^2}+\frac{8(\Marr+2\Mmulti^2)}{\varepsilon} + +\frac{16\Marr+32\Mmulti^2}{\varepsilon} + \expect{\sum_{t=1}^{T-1} \left(\Delta^{\net}(t)-\varepsilon/2\right)^+} \\
&\leq \frac{32\Marr+64\Mmulti^2}{\varepsilon} + \sar^{\net}(\pi, T),
\end{align*}
which finishes the proof.
\end{proof}
\subsection{Queue length bound in the regenerate stage (Lemma~\ref{lem:network-regenerate})}\label{app:network-regenerate}
We first show that Lemma~\ref{lem:multi-max-total} will still hold for a queueing network.
\begin{lemma}\label{lem:network-max-total}
For any horizon $T$ and every sample path, we have $\sum_{t=1}^T \|\bolds{Q}(t)\|_1 \geq \frac{(\max_{t=1}^T \|\bolds{Q}(t)\|_1)^2}{4\Marr}$.    
\end{lemma}
\begin{proof}
It suffices to prove that for every period $t$, we have $\|\bolds{Q}(t+1)\|_1 \leq \|\bolds{Q}(t)\|_1 + \Marr$, as we can then follow the proof of Lemma~\ref{lem:multi-max-total} to finish the proof. Now to show the bound on the increase, by the queueing dynamic \eqref{eq:dynamic-network}, we bound in period $t$,
\begin{align*}
\|\bolds{Q}(t+1)\|_1 &= \sum_{n \in \set{N}} \left(Q_n(t)-\sum_{k \in \set{B}_n}\sigma_k(t)S_k(t)+A_n(t)+\sum_{k'\in\set{K}}\sigma_{k'}(t)S_{k'}(t)L_{k',n}(t)\right) \tag{by \eqref{eq:dynamic-network} and summing over all queues} \\
&= \sum_{n \in \set{N}} Q_n(t) + \sum_{n \in \set{N}} A_n(t) + \sum_{k \in \set{K}} \sigma_k(t)S_{k}(t)\left(1 - \sum_{n \in \set{N}} L_{k,n}(t)\right) \\
&\leq \sum_{n \in \set{N}} Q_n(t) + \sum_{n \in \set{N}} A_n(t) \tag{$\sum_{n\in \set{N}}L_{k,n} \leq 1$ since a job transits to either a queue or leaves the system} \\
&\leq \|\bolds{Q}(t)\|_1 + \Marr,
\end{align*}
which completes the proof by following arguments in the second paragraph of the proof of Lemma~\ref{lem:multi-max-total}.
\end{proof}
\begin{proof}[Proof of Lemma~\ref{lem:network-regenerate}]
Fix a period $t$. First, we use equation \eqref{eq:drift-bound-decompose-proof} from the proof of Fact~\ref{fact:diff-norm2} \[\|\bolds{Q}(t)\|_2^2 + 2\bolds{Q}(t)\cdot\left(\bolds{Q}(t+1)-\bolds{Q}(t)\right) + \|\bolds{Q}(t+1)-\bolds{Q}(t)\|_2^2 = \|\bolds{Q}(t+1)\|_2^2.\] As a result,
\begin{align*}
\expect{V(t+1)-V(t) \mid \set{F}_t} &= \expect{\|\bolds{Q}(t+1)\|_2^2 - \|\bolds{Q}(t)\|_2^2 \mid \set{F}_t} \\
&= \expect{2\bolds{Q}(t)\cdot\left(\bolds{Q}(t+1)-\bolds{Q}(t)\right) + \|\bolds{Q}(t+1)-\bolds{Q}(t)\|_2^2 \mid \set{F}_t} \\
&\leq \expect{2\bolds{Q}(t)\cdot\left(\bolds{Q}(t+1)-\bolds{Q}(t)\right) \mid \set{F}_t} + 2\Marr+3\Mmulti^2 \tag{by \eqref{eq:net-bounded-drift}} \\
&\leq 2\left(-\varepsilon\sum_{n \in \set{N}} Q_n(t) + \Mmulti^2 + \Delta^{\net}(t)\|\bolds{Q}(t)\|_{\infty}\right) + 2\Marr+3\Mmulti^2 \tag{by \eqref{eq:network-bound-cross-dot}}.
\end{align*}
Take expectation on both sides gives
\begin{equation}\label{eq:net-one-step-drift}
\expect{V(t+1)}-\expect{V(t)} \leq \expect{-2\varepsilon\|\bolds{Q}(t)\|_1 + 2\Delta^{\net}(t)\|\bolds{Q}(t)\|_{\infty}} + 2\Marr + 5\Mmulti^2.
\end{equation}
Then for a fixed horizon $T$, we have
\begin{align}
0 &\leq \expect{V(T+1)}-\expect{V(1)} \tag{since $V(T+1) \geq 0$ and $V(1) = 0$} \\
&= \sum_{t=1}^T \expect{V(t+1)-V(t)} \nonumber\\
&\leq \expect{-2\varepsilon\sum_{t=1}^T \|\bolds{Q}(t)\|_1 + 2\sum_{t=1}^T \Delta^{\net}(t)\|\bolds{Q}(t)\|_{\infty}} + T(2\Marr+5\Mmulti^2) \nonumber\\
&\leq -\varepsilon\expect{\sum_{t=1}^T \|\bolds{Q}(t)\|_1} + 2\expect{\sum_{t=1}^T \left(\Delta^{\net}(t)-\frac{\varepsilon}{2}\right)^+\|\bolds{Q}(t)\|_1} + T(2\Marr+5\Mmulti^2)  \tag{since $\|\bolds{Q}(t)\|_1 \geq \|\bolds{Q}(t)\|_{\infty}$} \nonumber\\
&\hspace{-0.2in}= -\frac{\varepsilon}{2}\expect{\sum_{t=1}^T \|\bolds{Q}(t)\|_1}+\expect{-\frac{\varepsilon}{2}\sum_{t=1}^T \|\bolds{Q}(t)\|_1+2\sum_{t=1}^T \left(\Delta^{\net}(t)-\frac{\varepsilon}{2}\right)^+\|\bolds{Q}(t)\|_1}+T(2\Marr+5\Mmulti^2) \label{eq:network-sqr-norm-telescope}
\end{align}
We show a sample path upper bound of the second term by seeing that
\begin{align}
&\hspace{0.2in}-\frac{\varepsilon}{2}\sum_{t=1}^T \|\bolds{Q}(t)\|_1+2\sum_{t=1}^T \left(\Delta^{\net}(t)-\frac{\varepsilon}{2}\right)^+\|\bolds{Q}(t)\|_1 \nonumber\\
&\leq -\frac{\varepsilon}{2}\sum_{t=1}^T \|\bolds{Q}(t)\|_1+2\left(\max_{t\leq T} \|\bolds{Q}(t)\|_1\right)\sum_{t=1}^T \left(\Delta^{\net}(t)-\frac{\varepsilon}{2}\right)^+   \nonumber\\
&\leq -\frac{\varepsilon\left(\max_{t\leq T} \|\bolds{Q}(t)\|_1\right)^2}{8\Marr} + 2\left(\max_{t\leq T} \|\bolds{Q}(t)\|_1\right)\sar^{\net}(\pi, T) \tag{by Lemma~\ref{lem:network-max-total} and the definition of satisficing regret} \\
&\leq \frac{8\Marr\sar^{\net}(\pi,T)^2}{\varepsilon} \label{eq:network-bound-sar-drift}
\end{align}
where the last inequality uses the fact that $-ax^2+bx\leq \frac{b^2}{4a}$ for any $a<0$ and takes $x = \max_{t\leq T} \|\bolds{Q}(t)\|_1$. Replacing the second term in \eqref{eq:network-sqr-norm-telescope} by the upper bound in \eqref{eq:network-bound-sar-drift} gives
\[
0 \leq -\frac{\varepsilon}{2}\expect{\sum_{t=1}^T \|\bolds{Q}(t)\|_1} + \frac{8\Marr\expect{\sar^{\net}(\pi,T)^2}}{\varepsilon} + T(2\Marr+5\Mmulti^2).
\]
Dividing both sides by $T$ and reorganizing terms show
\[
\frac{\expect{\sum_{t=1}^T \|\bolds{Q}(t)\|_1}}{T} \leq \frac{4\Marr+10\Mmulti^2}{\varepsilon}+\frac{16\Marr}{\varepsilon^2}\frac{\expect{\sar^{\net}(\pi,T)^2}}{T}.
\]
\end{proof}
\subsection{Satisficing regret of \textsc{BP-UCB} (Lemma~\ref{lem:sar-bp-ucb})}\label{app:sar-bp-ucb}
The proof of Lemma~\ref{lem:sar-bp-ucb} follows a similar structure as that of Lemma~\ref{lem:sar-mw-ucb}. 
The additional complexity comes from estimating $r_{k,n} = \mu_k p_{k,n}$, the probability a job transitions to $n$ given that server $k$ is selected. We define $\Delta_k(t) = \sqrt{\frac{2\ln t}{C_k(t)}}$ as the confidence bound of server $k$ in period $t$ as in \textsc{BP-UCB} (Algorithm~\ref{algo:bp-ucb}) and let $\set{G}_t$ be the good event that $\left\{\Delta^{\net}(t) \leq 4\sum_{k \in \set{K}}\sigma_k(t)\Delta_k(t)|\set{D}_k|\right\}$. Extending Lemma~\ref{lem:multi-good-event}, we show the following lower bound on $\Pr\{\set{G}_t\}$.
\begin{lemma}\label{lem:network-good-event}
For every period $t$ and under $\textsc{BP-UCB}$, we have $\Pr\{\set{G}_t\} \geq 1 - 4\sum_{k \in \set{K}} |\set{D}_k|t^{-3}$.    
\end{lemma}
\begin{proof}
Fix a period $t$. Define the event $\set{E}_t$ as the case that all sample averages are within a confidence bound of their true values:
\[
\set{C}_t = \left\{\forall_{n\in \set{N},k\in \set{K}}, |\hat{\mu}_k(t)-\mu_k| \leq \Delta_k(t), |\hat{r}_{k,n}(t) - r_{k,n}(t)| \leq \Delta_k(t)\right\}.
\]
We first lower bound the probability of $\set{C}_t$. For a server $k$, condition on $C_k(t) = c \in \{0,1,\ldots,t-1\}$. By Hoeffding's inequality (Fact~\ref{fact:hoeffding}), we have
\[
\Pr\left\{|\hat{\mu}_k(t)-\mu_k| \leq \sqrt{\frac{2\ln t}{c}}\right\} \leq 2t^{-3},~\Pr\left\{|\hat{r}_{k,n}(t)-r_{k,n}| \leq \sqrt{\frac{2\ln t}{c}}\right\} \leq 2t^{-3}.
\]
for every queue $n$. A union bound over $k \in \set{K},C_k(t) < t, n \in \set{D}_k$ then gives (with $|\set{D}_k|\geq 1\;\forall k$)
\begin{equation}\label{eq:net-prb-bound-goodevent}
\Pr\left\{\set{C}_t\right\} \geq 1 - 2\sum_{k \in \set{K}} (1 + |\set{D}_k|)t^{-3} \geq 1 - 4\sum_{k \in \set{K}} |\set{D}_k|t^{-3}.
\end{equation}
We next show that $\set{C}_t \subseteq \set{G}_t$. Recall the definitions  $\bar{\mu}_k(t) = \hat{\mu}_k(t) + \Delta_k(t), \ubar{r}_{k,n}(t) = \hat{r}_{k,n}(t) - \Delta_k(t)$ for $k \in \set{K}, n \in \set{D}_k$.  
Condition on $\set{C}_t$, we have for every $k \in \set{K}, n \in \set{D}_k$,
\begin{equation}\label{eq:net-dis-ucb-mean}
\mu_k \leq \bar{\mu}_k(t) \leq \mu_k + 2\Delta_k(t),~\ubar{r}_{k,n}(t) \leq r_{k,n} \leq \ubar{r}_{k,n}(t) + 2\Delta_k(t).
\end{equation}
As a result, the weight of a chosen schedule $\bolds{\sigma}(t)$ is lower bounded by
\begin{align*}
W^{\net}_{\bolds{\sigma}(t)}(t) &= \sum_{n \in \set{N}} Q_n(t)\left(\sum_{k \in \set{B}_n} \sigma_k(t)\mu_k - \sum_{k'\in\set{K}} \sigma_{k'}\mu_{k'}p_{k',n}\right) \\
&= \sum_{n \in \set{N}} Q_n(t)\left(\sum_{k \in \set{B}_n} \sigma_k(t)\mu_k - \sum_{k'\in\set{K}\colon n \in \set{D}_{k'}} \sigma_{k'}r_{k',n}\right) \tag{Definition $r_{k',n} = \mu_{k'}p_{k',n}$} \\
&\hspace{-0.2in}\geq \sum_{n \in \set{N}} Q_n(t)\left(\sum_{k \in \set{B}_n} \sigma_k(t)\left(\bar{\mu}_k(t)-2\Delta_k(t)\right) - \sum_{k'\in\set{K}\colon n \in \set{D}_{k'}} \sigma_{k'}(\ubar{r}_{k',n}(t)+2\Delta_{k'}(t))\right) \tag{By \eqref{eq:net-dis-ucb-mean}} \\
&\hspace{-0.2in}\geq  \sum_{n \in \set{N}} Q_n(t)\left(\sum_{k \in \set{B}_n}\sigma_k(t)\bar{\mu}_k(t) - \sum_{k'\in \set{K}} \sigma_{k'}(t)\ubar{r}_{k',n}(t)\right) - 2\|\bolds{Q}(t)\|_{\infty}\sum_{k \in \set{K}}\sigma_{k}(t)\Delta_{k}(t)(1 + |\set{D}_{k}|) \tag{ reorganizing terms and  $\ubar{r}_{k',n} = 0$ for $n \not \in \set{D}_k$} \\
&\hspace{-0.2in}\geq \sum_{n \in \set{N}} Q_n(t)\left(\sum_{k \in \set{B}_n}\sigma^{\BP}_k(t)\bar{\mu}_k(t) - \sum_{k'\in \set{K}} \sigma^{\BP}_{k'}(t)\ubar{r}_{k',n}(t)\right) - 4\|\bolds{Q}(t)\|_{\infty}\sum_{k \in \set{K}}\sigma_{k}(t)\Delta_{k}(t)|\set{D}_{k}| \tag{Choice of $\bolds{\sigma}(t)$ in \textsc{BP-UCB} (Line~\ref{line:bp-ucb-choice})} \\
&\hspace{-0.2in}\geq \sum_{n \in \set{N}} Q_n(t)\left(\sum_{k \in \set{B}_n}\sigma^{\BP}_k(t)\mu_k - \sum_{k'\in \set{K}} \sigma^{\BP}_{k'}(t)r_{k',n}\right) - 4\|\bolds{Q}(t)\|_{\infty}\sum_{k \in \set{K}}\sigma_{k}(t)\Delta_{k}(t)|\set{D}_{k}| \tag{By~\eqref{eq:net-dis-ucb-mean}} \\
&= W^{\net}_{\bolds{\sigma}^{\BP}(t)}(t) - 4\|\bolds{Q}(t)\|_{\infty}\sum_{k \in \set{K}}\sigma_{k}(t)\Delta_{k}(t)|\set{D}_{k}|. 
\end{align*}
Conditioned on $\set{C}_t$, we then have either $\|\bolds{Q}(t)\|_{\infty} = 0$ and $\Delta^{\net}(t) = 0 \leq 4\sum_{k \in \set{K}}\sigma_{k}(t)\Delta_{k}(t)|\set{D}_{k}|$, or 
\[
\Delta^{\net}(t) = \frac{W^{\net}_{\bolds{\sigma}^{\BP}(t)}(t) - W^{\net}_{\bolds{\sigma}(t)}(t)}{\|\bolds{Q}(t)\|_{\infty}} \leq 4\sum_{k \in \set{K}}\sigma_{k}(t)\Delta_{k}(t)|\set{D}_{k}|.
\]
As a result, $\set{C}_t \subseteq \set{G}_t$, which concludes the proof by $\Pr\{\set{G}_t\} \geq \Pr\{\set{C}_t\} \geq 1 - 4\sum_{k \in \set{K}}|\set{D}_k|t^{-3}$ by \eqref{eq:net-prb-bound-goodevent}.
\end{proof}
We rely on the following decomposition of satisficing regret
\begin{equation}\label{eq:net-satisficing-decomp}
\sar^{\net}(\textsc{BP-UCB},T) = \sum_{t=1}^T \left(\Delta^{\net}(t)-\frac{\varepsilon}{2}\right)^+\indic{\set{G}_t} + \sum_{t=1}^T \left(\Delta^{\net}(t)-\frac{\varepsilon}{2}\right)^+\indic{\set{G}_t^c}.
\end{equation}
The second term is small by Lemma~\ref{lem:network-good-event} and the upper bound on $\Delta^{\net}(t)$ in Lemma~\ref{lem:net-bound-delta}. It remains to upper bound the first term in \eqref{eq:net-satisficing-decomp}, for which we provide a sample-path bound.
\begin{lemma}\label{lem:net-path-sar}
For every horizon $T$ and sample path, we have
\[
\sum_{t=1}^T \left(\Delta^{\net}(t)-\frac{\varepsilon}{2}\right)^+\indic{\set{G}_t} \leq \frac{128\Mdep\Mmulti^2(\ln T+0.5)}{\varepsilon}.
\]
\end{lemma}
\begin{proof}
For a period $t$, define event $\set{E}_t = \left\{\exists k \in \set{K}\colon\sigma_k(t)=1, C_k(t) \leq \frac{32|\set{D}_k|^2\Mmulti^2\ln(t)}{(\Delta^{\net}(t))^2}\right\}$. We first show $\set{G}_t \subseteq \set{E}_t$. Condition on $\set{G}_t$, we have $\Delta^{\net}(t) \leq 4\sum_{k \in \set{K}} \sigma_k(t)\Delta_k(t)|\set{D}_k|$ by definition of $\set{G}_t$. Since the maximum of a sequence is at least the average of a sequence, there must exist a server $k$ with $\sigma_k(t) = 1$, such that
$
4\Delta_k(t)|\set{D}_k| \geq \frac{\Delta^{\net}(t)}{\sum_{k'\in \set{K}}\sigma_{k'}(t)} \geq \frac{\Delta^{\net}(t)}{\Mmulti}.$
Recall that $\Delta_k(t) = \sqrt{\frac{2\ln t}{C_k(t)}}$. Therefore, $4\sqrt{\frac{2\ln t}{C_k(t)}}|\set{D}_k| \geq \frac{\Delta^{\net}(t)}{\Mmulti}$ which then gives $C_k(t) \leq \frac{32|\set{D}_k|^2\Mmulti^2\ln(t)}{(\Delta^{\net}(t))^2}$ and $\set{G}_t \subseteq \set{E}_t$. We then have
\begin{align}
\sum_{t=1}^T \left(\Delta^{\net}(t)-\frac{\varepsilon}{2}\right)^+\indic{\set{G}_t} &\leq  \sum_{t=1}^T \left(\Delta^{\net}(t)-\frac{\varepsilon}{2}\right)^+\indic{\set{E}_t} \nonumber \\
&\leq \sum_{t=1}^T \sum_{k \in \set{K}} \left(\Delta^{\net}(t)-\frac{\varepsilon}{2}\right)^+\indic{\sigma_k(t)=1,C_k(t)\leq \frac{32|\set{D}_k|^2\Mmulti^2\ln(t)}{(\Delta^{\net}(t))^2}}. \label{eq:network-sar-transform}
\end{align}
Following the same argument after \eqref{eq:multi-sar-transform} in the proof of Lemma~\ref{lem:multi-sar-path-bound}, the support of $\{\Delta^{\net}(t)\}_{t \leq T}$ is finite and we define $\Omega = \{\omega_1,\ldots,\omega_C\}$ by the support with $\omega_1 > \cdots > \omega_C$ and $\omega_{\tilde{C}}$ is the smallest value that is at least $\frac{\varepsilon}{2}$. We then rewrite \eqref{eq:network-sar-transform} by
\begin{align*}
&\hspace{0.2in}\sum_{t=1}^T \left(\Delta^{\net}(t)-\frac{\varepsilon}{2}\right)^+\indic{\set{G}_t} \\
&\leq \sum_{t=1}^T \sum_{k \in \set{K}}\sum_{i=1}^C \left(\omega_i-\frac{\varepsilon}{2}\right)^+\indic{\Delta^{\net}(t)=\omega_i,\sigma_k(t)=1,C_k(t)\leq \frac{32|\set{D}_k|^2\Mmulti^2\ln(t)}{(\Delta^{\net}(t))^2}} \\
&= \sum_{k \in \set{K}}\sum_{i=1}^C \left(\omega_i-\frac{\varepsilon}{2}\right)^+\sum_{t=1}^T \indic{\Delta^{\net}(t)=\omega_i,\sigma_k(t)=1,C_k(t)\leq \frac{32|\set{D}_k|^2\Mmulti^2\ln(t)}{(\Delta^{\net}(t))^2}} \\
&\leq \sum_{k \in \set{K}}\sum_{i=1}^{\tilde{C}} \omega_i\sum_{t=1}^T \indic{\Delta^{\net}(t)=\omega_i,\sigma_k(t)=1,C_k(t)\leq \frac{32|\set{D}_k|^2\Mmulti^2\ln(t)}{(\Delta^{\net}(t))^2}} \\
&\hspace{-0.2in}\leq K\omega_1+\sum_{k \in \set{K}}32|\set{D}_k|^2\Mmulti^2\ln(T)\left(\frac{1}{\omega_1}+\sum_{i=2}^{\tilde{C}} \omega_i\left(\frac{1}{\omega_i^2}-\frac{1}{\omega_{i-1}^2}\right)\right) \tag{\eqref{eq:multi-sar-reordering} with $x = 32|\set{D}_k|^2\Mmulti^2\ln(T)$} \\
&\hspace{-0.2in}\leq K\omega_1+32\Mdep\Mmulti^2\ln(T)\left(\frac{1}{\omega_1}+\sum_{i=2}^{\tilde{C}} \omega_i\left(\frac{1}{\omega_i^2}-\frac{1}{\omega_{i-1}^2}\right)\right) \tag{Definition $\Mdep = \sum_{k \in \set{K}} |\set{D}_k|^2$} \\
&\leq K\omega_1+\frac{64\Mdep\Mmulti^2\ln(T)}{\omega_{\tilde{C}}} \leq \frac{128\Mdep\Mmulti^2(\ln(T)+0.5)}{\varepsilon},
\end{align*}
where the second to last inequality uses Fact~\ref{fact:bound-reciprocal}; the last inequality uses the fact that $\Delta^{\net}(t) \leq 2\Mmulti$ by Lemma~\ref{lem:net-bound-delta} so $\omega_1 \leq 2\Mmulti$ and the fact that $\omega_{\tilde{C}} \geq \varepsilon / 2$. We thus finish the proof.
\end{proof}
We are ready to prove Lemma~\ref{lem:sar-bp-ucb}.
\begin{proof}[Proof of Lemma~\ref{lem:sar-bp-ucb}]
For every period $t$, The decomposition \eqref{eq:net-satisficing-decomp} gives
\begin{align*}
\expect{\sar^{\net}(\textsc{BP-UCB}, T)} &= \expect{\sum_{t=1}^T \left(\Delta^{\net}(t)-\frac{\varepsilon}{2}\right)^+\indic{\set{G}_t}} + \expect{\sum_{t=1}^T \left(\Delta^{\net}(t)-\frac{\varepsilon}{2}\right)^+\indic{\set{G}_t^c}} \tag{by \eqref{eq:net-satisficing-decomp}} \\
&\leq \frac{128\Mdep\Mmulti^2(\ln T+0.5)}{\varepsilon} + 2\Mmulti\sum_{t=1}^T \Pr\{\set{G}_t^c\} \tag{by Lemma~\ref{lem:net-path-sar} and the upper bound of $\Delta^{\net}(t)$ in Lemma~\ref{lem:net-bound-delta}} \\
&\leq \frac{128\Mdep\Mmulti^2(\ln T+0.5)}{\varepsilon} + 8\Mmulti\sum_{t=1}^T \sum_{k \in \set{K}} |\set{D}_k|t^{-3} \tag{by Lemma~\ref{lem:network-good-event}} \\
&\leq \frac{128\Mdep\Mmulti^2(\ln T+0.5)}{\varepsilon} + 16\Mdep \leq \frac{128\Mdep\Mmulti^2(\ln T+1)}{\varepsilon}.
\end{align*}
In addition,
\begin{align*}
\expect{\sar^{\net}(\textsc{BP-UCB}, T)^2} &= \expect{\left(\sum_{t=1}^T \left(\Delta^{\net}(t)-\frac{\varepsilon}{2}\right)^+\indic{\set{G}_t} + \sum_{t=1}^T \left(\Delta^{\net}(t)-\frac{\varepsilon}{2}\right)^+\indic{\set{G}_t^c}\right)^2} \tag{by \eqref{eq:net-satisficing-decomp}}    \\
&\hspace{-0.5in}\leq 2\expect{\left(\sum_{t=1}^T \left(\Delta^{\net}(t)-\frac{\varepsilon}{2}\right)^+\indic{\set{G}_t}\right)^2} + 2\expect{\left(\sum_{t=1}^T \left(\Delta^{\net}(t)-\frac{\varepsilon}{2}\right)^+\indic{\set{G}_t^c}\right)^2} \tag{$(x+y)^2 \leq 2(x^2+y^2)$ for any $x,y$} \\
&\leq 2\left(\frac{128\Mdep\Mmulti^2(\ln T+0.5)}{\varepsilon}\right)^2 + 8\Mmulti^2\expect{\left(\sum_{t=1}^T \indic{\set{G}_t^c}\right)^2} \tag{by Lemma~\ref{lem:net-path-sar} and Lemma~\ref{lem:net-bound-delta}}
\end{align*}
We have 
\begin{align*}
\expect{\left(\sum_{t=1}^T \indic{\set{G}_t^c}\right)^2} \leq 2\sum_{1\leq t_1\leq t_2 \leq T} \expect{\indic{\set{G}_{t_1}^c}\indic{\set{G}_{t_2}^c}}\leq 2\sum_{t \leq T} t\Pr\{\set{G}_t^c\} \leq 8\sum_{t\leq T} \sum_{k \in \set{K}} |\set{D}_k|t^{-2} \leq 16\Mdep.
\end{align*}
Therefore,
\begin{align*}
\expect{\sar^{\net}(\textsc{BP-UCB}, T)^2} &\leq 2\left(\frac{128\Mdep\Mmulti^2(\ln T+0.5)}{\varepsilon}\right)^2 + 8\Mmulti^2\expect{\left(\sum_{t=1}^T \indic{\set{G}_t^c}\right)^2}  \\
&\hspace{-1.5in}\leq \frac{2^{15}\Mdep^2\Mmulti^4(\ln T+0.5)^2}{\varepsilon^2}+2^7\Mdep\Mmulti^2 \leq \frac{2^{15}\Mdep^2\Mmulti^4(\ln T+1)^2}{\varepsilon^2},
\end{align*}
which completes the proof.
\end{proof}

\section{Some useful known results}\label{app:ineq}
\subsection{Properties of KL divergence}
Recall that for any two distributions $G,P$ over the same probability space, we use $\tv{G}{P}$ to denote the total variation between $G,P$ and $\kl{G}{P}$ is their KL divergence. Pinsker's Inueqality (Lemma 2.5 in \cite{Tsybakov08}) upper bounds $\tv{G}{P}$ by $\kl{G}{P}$. 
\begin{fact}\label{prop:pinsker}
For any $G,P$, we have $\tv{G}{P} \leq \sqrt{\kl{G}{P} / 2}$.
\end{fact}
If $G(x,y),P(x,y)$ are distributions of two random variables $X,Y$, the conditional KL divergence $\kl{G(y|x)}{P(y|x)}$ is defined by $\sum_{x,y} G(x,y)\ln\frac{G(y|x)}{P(y|x)}$. we have the following chain rule of their KL divergence (Theorem 2.5.3 \cite{CoverThomas06}).
\begin{fact}\label{prop:chain}
For any $G(x,y),P(x,y)$, we have \[\kl{G(x,y)}{P(x,y)} = \kl{G(x)}{P(x)}+\kl{G(y|x)}{P(y|x)}\] where $G(x),P(x)$ are the marginal distributions of $X$.
\end{fact}
Finally, we have that kl divergence is upper bounded by the $\chi^2$ divergence and it induces the following useful inequality.
\begin{fact}\label{prop:kl-bound}
The KL divergence between two Bernoulli distributions with mean $g,q$ respectively can be bounded by $\frac{(g-q)^2}{q(1-q)}$.
\end{fact}
\begin{proof}
Let $G,Q$ denote the two Bernoulli distributions and $\chi^2(G \parallel Q)$ be their $\chi^2$ divergence defined by $q\left(\frac{p}{q}-1\right)^2+(1-q)\left(\frac{1-p}{1-q}-1\right)^2 = \frac{(g-q)^2}{q(1-q)}$. Then by Lemma 2.7 of \cite{Tsybakov08}, we have $\kl{G}{Q} \leq \chi^2(G \parallel Q) =  \frac{(g-q)^2}{q(1-q)}$.
\end{proof}
\subsection{Concentration Inequality}
we use the following version of Hoeffding's Inequality from \cite{boucheron2013concentration}.
\begin{fact}[Hoeffding's Inequality]\label{fact:hoeffding}
Given $N$ independent random variables $X_i$ taking values in $[a,b]$ almost surely. Let $X = \sum_{n=1}^N X_n$. Then for any $x > 0$,
    $\Pr\{|X-\expect{X}| > x\} \leq 2e^{-2x^2/N(b-a)^2}.$
\end{fact}
\subsection{Analytical facts}
The following result compares $\ln(t)$ and a polynomial of $t$.
\begin{fact}\label{fact:lnt-sqrt-prop}
(i) For $t \geq 50000$, we have $(\ln t+2)^2 \leq \sqrt{t}$; (ii) for $t \geq 12^{19}$, we have $\ln^6 t \leq \sqrt{t}$; (iii) for $t \geq 100$, we have $\ln t \leq t^{1/3}$.
\end{fact}
\begin{proof}
(i): Let $f(t) = t^{1/4} - \ln t - 2$. We have $f'(t) = \frac{1}{4}t^{-3/4}-\frac{1}{t} \geq 0$ for $t \geq 256$ and thus $f(t)$ is non-decreasing for $t \geq 256$. Then since $f(50000) \geq 0$, we have $f(t) \geq 0$ for all $t \geq 50000$. 

(ii): Similarly, let $g(t) = t^{1/12} - \ln t$. We have $g'(t) = \frac{1}{12}t^{-11/12}-\frac{1}{t} \geq 0$ when $t \geq 12^{12}$. Since $g(12^{19}) \geq 0$, we have $g(t) \geq 0$ for any $t \geq 12^{19}.$

(iii): let $h(t) = t^{1/3} - \ln t$. Then $h'(t) = \frac{1}{3}t^{-2/3} - \frac{1}{t} \geq 0$ when $t \geq 27.$ Since $h(100) \geq 0$, we prove $\ln t \leq t^{1/3}$ for $t \geq 100$ since $h(x)$ is non-decreasing for $x \geq 100.$
\end{proof}
This fact is from Equation 4 in \cite{ShahTZ10} and we include the proof for completeness.
\begin{fact}\label{fact:diff-norm2}
For any non-zero queue-length vector $\bolds{Q}(t)$, the 2-norm of $\bolds{Q}(t+1)$ can be bounded by
\[
\|\bolds{Q}(t+1)\|_2 \leq \|\bolds{Q}(t)\|_2+ \frac{\bolds{Q}(t) \cdot (\bolds{Q}(t+1) - \bolds{Q}(t)) +  \|\bolds{Q}(t+1)-\bolds{Q}(t)\|_2^2}{\|\bolds{Q}(t)\|_2}.
\]
\end{fact}
\begin{proof}
To simplify the notation, we denote the right hand side of the inequality by $\mathrm{RHS}$. To prove the fact, we take the square of $\mathrm{RHS}$ which gives
\begin{equation}\label{eq:drift-bound-decompose-proof}
\mathrm{RHS}^2 \geq \|\bolds{Q}(t)\|_2^2+2\bolds{Q}(t) \cdot (\bolds{Q}(t+1) - \bolds{Q}(t)) +  \|\bolds{Q}(t+1)-\bolds{Q}(t)\|_2^2 = \|\bolds{Q}(t+1)\|_2^2.
\end{equation}
Note that
\[
\mathrm{RHS}\|\bolds{Q}(t)\|_2 =\|\bolds{Q}(t)\|_2^2 + \bolds{Q}(t) \cdot (\bolds{Q}(t+1) - \bolds{Q}(t)) +  \|\bolds{Q}(t+1)-\bolds{Q}(t)\|_2^2 \geq 0,
\]
and thus $\mathrm{RHS} \geq 0$. Taking the square root for \eqref{eq:drift-bound-decompose-proof} shows the desired claim.
\end{proof}

The following fact is restated from \cite[Lemma~3]{KvetonWAEE14}.
\begin{fact}\label{fact:bound-reciprocal}
Let $\omega_1 \geq \cdots \geq \omega_{\tilde{C}}$ be a sequence of $\tilde{C}$ positive numbers. Then 
\[
\frac{1}{\omega_1} + \sum_{i=2}^{\tilde{C}} \omega_i\left(\frac{1}{\omega_i^2}-\frac{1}{\omega_{i-1}^2}\right) \leq \frac{2}{\omega_{\tilde{C}}}.
\]
\end{fact}
\begin{fact}\label{fact:generalized-mean}
For any $a,b > 0$ and positive integer $n$, we have $(a+b)^n \leq 2^{n-1}(a^n+b^n)$.
\end{fact}
\begin{proof}
It is equivalent to show $\left(\frac{a+b}{2}\right)^n \leq \frac{a^n + b^n}{2}$, which is true because $x^n$ is a convex function.
\end{proof}

\end{document}